\documentclass{article}

\usepackage[english]{babel}

\usepackage[letterpaper, top=25mm, bottom=25mm, left=34mm, right=34mm, marginparwidth=17.5mm]{geometry}

\usepackage[utf8]{inputenc}
\usepackage[T1]{fontenc}
\usepackage{microtype}
\usepackage[colorlinks=true, linkcolor=blue, citecolor=blue, urlcolor=blue]{hyperref}
\usepackage{url}           		
\usepackage{booktabs}       	
\usepackage{amsfonts}       	
\usepackage{nicefrac}       	
\usepackage[table]{xcolor}  	
\usepackage{natbib}
\usepackage{titlesec}
\usepackage{etoolbox}

\titlespacing{\paragraph}{0pt}{0.3em}{0.6em}

\usepackage{colors}
\usepackage{custom}
\usepackage{table_shortcuts}

\usepackage{amsmath,amsfonts,bm}

\def\1{\bm{1}}

\DeclareMathAlphabet{\mathsfit}{\encodingdefault}{\sfdefault}{m}{sl}
\SetMathAlphabet{\mathsfit}{bold}{\encodingdefault}{\sfdefault}{bx}{n}

\graphicspath{
	{.}
	{figures}
}

\title{
Partition, Prompt, Aggregate:\\Statistical Self-Consistency in Language Models
}

\usepackage{authblk}
\stepcounter{footnote}
\addtocounter{footnote}{-1}
\author[1,2,3,4]{Patrik Wolf}
\author[2]{Thomas Kleine Buening}
\author[2]{Andreas Krause}
\author[1,2,3,4]{Celestine Mendler-D\"unner}

\affil[1]{Max Planck Institute for Intelligent Systems, T\"ubingen, Germany}
\affil[2]{ETH Z\"urich}
\affil[3]{ELLIS Institute T\"ubingen}
\affil[4]{T\"ubingen AI Center}

\date{}
\begin{document}

\maketitle

\begin{abstract}
\noindent
Large language models (LLMs) are often used to perform a form of conditional inference. In-context learning is one implicit way to do this. The prompt specifies a context and the model's output is treated as an estimate of the corresponding conditional distribution. 
If this interpretation holds, then LLM estimates should satisfy basic probabilistic identities. In particular, the law of total probability asserts that prior-weighted conditional distributions aggregate into population-level marginals over any valid partition of the population. In this work, we investigate to what extent LLM estimates adhere to this self-consistency principle.
We use binary trees as an evaluation scaffold to recursively partition a population into increasingly fine-grained subpopulations. We then prompt LLMs with verbalized subpopulation descriptions in context, aggregate the resulting estimates back into population-level estimates, and compare them across partitions of varying granularity. 
Translating these comparisons into a benchmark score and evaluating frontier models across problem domains, we demonstrate widespread violations of basic consistency properties.
An in-depth study of persona prompting further reveals a pattern we call the \emph{macro fallacy}: estimates reconstructed from more fine-grained subpopulation responses are often better aligned with human reference data than direct population-level estimates. The effect persists across variations in partitioning, verbalization, and estimation task, suggesting that models possess relevant subpopulation knowledge but do not reliably propagate it into aggregate estimates.
This gap establishes \emph{statistical self-consistency} as an unsaturated, reference-free criterion for evaluating LLMs.\looseness=-1
\end{abstract}


\vspace*{0.5mm}
\section{Introduction}\label{sec:introduction}

``Given that a voter lives in California, how likely are they to vote Republican in the next presidential election?'' When we pose such a question to a large language model, we typically interpret its answer as an estimate of a conditional probability: the prompt identifies individuals within a subpopulation, and the model estimates the probability of an outcome within this group.
The same interpretation applies across forecasting, social simulation, and other application domains, where contextual information is used to elicit predictions for particular scenarios. This view also underpins theoretical accounts of in-context learning (ICL), which explain how models can use the information supplied in context to perform conditional inference \citep{xie2022explanationincontextlearningimplicit, müller2023pfns4boincontextlearningbayesian, reuter2025transformerslearnbayesianinference}.

Persona prompting offers one use case of in-context learning where the conditional-inference interpretation is particularly widespread. Here, a prompt is augmented with attributes of a person, such as age, occupation, or location, and the model is asked to respond from this perspective \citep{Argyle_2023}.
Outputs are then expected to approximate a conditional distribution over responses of individuals with a given attribute profile.
As a testament to this perspective, distributional alignment has become a popular evaluation metric~\citep{santurkar2023opinionslanguagemodelsreflect, meister2024benchmarkingdistributionalalignmentlarge}.

\begin{figure}
\centering
\resizebox{1\linewidth}{!}{%
\input{teaser-figure}
}
\vspace{-4mm}
\caption{
\textbf{Reconstructing aggregate estimates from subpopulation estimates.}
A binary tree recursively partitions a base population into increasingly specific subpopulations.
For each subpopulation, we elicit an LLM estimate by verbalizing the split criteria in context, then reconstruct population-level aggregate estimates at each level via the law of total probability.
We compare the reconstructed aggregates against the model's direct marginal estimate to assess \emph{self-consistency}, and against ground-truth population data to assess \emph{alignment}.
}
\label{fig:teaser_figure}
\end{figure}

This conditional-inference interpretation implies that model outputs should obey basic laws of probability theory, yet it remains poorly understood to what extent they do.
As \citet{cruz2024evaluatinglanguagemodelsrisk} noted, if language-model predictions formed a coherent joint distribution, prompting with fewer attributes would yield the same prediction as marginalizing over the omitted attributes. We evaluate and operationalize this consistency criterion using the law of total probability. The marginal over any partition must equal the prior-weighted aggregate of its conditionals. Persona attributes serve to induce such partitions at different levels of granularity. Similarly, any pair of complementary events defines a binary partition, such as \emph{inflation rising} versus \emph{not rising.} Comparing the direct estimate with the prior-weighted aggregate of the estimates for the partition elements yields a natural, reference-free self-consistency check.
While self-consistency does not require any ground-truth reference data, persona prompting serves as an appealing testbed because it additionally supplies human reference statistics for alignment analyses and ablations. We use this framework to stress-test the often implicit assumption that in-context learning performs conditional inference.

\subsection{Our work}

We start with an in-depth study of persona prompting as an instance of in-context learning. We use \emph{binary conditioning trees} to systematically construct personas that define subpopulations at increasing levels of specificity (see \cref{fig:teaser_figure}). Let $\mathcal{X}$ denote a base population.
Each successive level of the tree refines the persona description with an additional attribute. By construction, the nodes at that level define a partition $\mathcal{P}$ of the population, with $\mathcal{X} = \cup_{\mathcal{S}\in\mathcal{P}}\,\mathcal{S}$.
Given such a partition, for any event $E$, the \emph{law of total probability} \citep[cf.][]{bertsekas2008introduction} asserts that
\begin{equation}\label{eq:lotp_decomposition}
\underbrace{\mathbb{P}(E\given X \in\mathcal{X})}_{\text{aggregate}} = \sum_{\mathcal{S}\in\mathcal{P}} \underbrace{\gamma_\mathcal{S}}_{\text{prior}}\;\underbrace{\mathbb{P}(E\given X \in\mathcal{S})}_{\text{conditional}},
\end{equation}
where $\gamma_\mathcal{S}\in[0,1]$ denotes the relative size of partition element $\mathcal{S}$. 
This equivalence is the core observation we build on, and we apply it in two ways:

\begin{enumerate}
\item We substitute LLM estimates into the right-hand side of equation \eqref{eq:lotp_decomposition} and compare the resulting aggregate with reference data.
We use in-context learning to elicit the conditional probabilities and combine them using LLM-estimated prior weights. Comparing the resulting aggregate estimates for different partitions against ground-truth data reveals an interesting pattern, which we call the \emph{macro fallacy}.
Namely, the aggregate $\mathbb{P}(E\given X \in\mathcal{X})$ retrieved directly from the LLM is systematically less aligned with human data than the estimate reconstructed from finer-grained conditionals and the corresponding priors via the law of total probability.
The effect is robust across target quantities, models, and tree structures.
\item We compare LLM estimates of the right-hand side of \eqref{eq:lotp_decomposition} across different partitions.
Because the law of total probability requires these quantities to be equal for every partition, any disagreement reveals an internal inconsistency in the model.
We propose two complementary diagnostics of statistical self-consistency: \emph{split consistency} and \emph{order consistency}. Together, these checks assess statistical self-consistency across a broad class of aggregation procedures, without requiring any ground-truth data.
Instantiating our consistency checks in settings spanning persona prompting and synthetic forecasting across a diverse set of models, we find widespread violations even among state-of-the-art frontier models. Furthermore, consistency scores exhibit no apparent correlation with the abilities assessed by standard benchmarks.
\end{enumerate}


\section{Background and related work}\label{sec:related_work}

\paragraph{In-context learning as conditional inference.}
We use \emph{in-context learning} (ICL) in the broad sense, which encompasses general forms of learning from contextual information. This information may consist of natural language instructions, task demonstrations, or both, and includes common settings such as zero-shot prompting, few-shot prompting, persona prompting, and forecasting \citep{brown2020languagemodelsfewshotlearners, schulhoff2025promptreportsystematicsurvey}. A substantial body of literature interprets ICL as conditional or Bayesian inference \citep{xie2022explanationincontextlearningimplicit, müller2023pfns4boincontextlearningbayesian, reuter2025transformerslearnbayesianinference, decarvalho2025simplifyingbayesianoptimizationincontext, buening2026aligninglanguagemodelsuser}. However, there is mixed evidence on how closely LLMs satisfy this characterization \citep{falck2024incontextlearninglargelanguage, chlon2025llmsbayesianexpectationrealization}.

\paragraph{Persona conditioning and distributional alignment.}
Persona conditioning is a form of ICL in which a model is asked to respond from the perspective of a person described by sociodemographic attributes \citep{Argyle_2023, beck2024sensitivityperformancerobustnessdeconstructing}. Distributional alignment studies compare the resulting response distributions with those of human populations, thereby implicitly treating model outputs as conditional distributions \citep{santurkar2023opinionslanguagemodelsreflect, meister2024benchmarkingdistributionalalignmentlarge, fidelity25ma, hu2026simbench}. This interpretation has been challenged by unreliable simulations, sensitivity to prompt formulation, degenerate output distributions, and distortions of the joint structure of human beliefs \citep{dominguez2024questioning, Bisbee_Clinton_Dorff_Kenkel_Larson_2024, cruz2024evaluatinglanguagemodelsrisk, lutz2025promptmakespersonasystematic, barrie_cerina_2026}. We offer a new perspective on this discussion by asking whether statistically equivalent estimators derived from the same model agree across levels of granularity.

\paragraph{Statistical self-consistency.}
We define statistical self-consistency as agreement between a marginal estimate and the prior-weighted aggregate of its conditionals over any valid partition, a requirement that does not rely on reference data. This differs from the prevalent uses of the term self-consistency in the machine learning literature, where it typically refers to consistency across semantically equivalent contexts \citep{elazar2021cons, wang2022self, li2024benchmarking}. 
More closely related work tests probabilistic coherence under conjunctions, Bayes' rule, or logical forecasting constraints \citep{zhu2024incoherent, paruchuri2024odds, fluri2024superhuman, paleka2025consistency}, as well as agreement between direct and compositionally reconstructed answers \citep{bottman2026operadicconsistencylabelfreesignal}. What is unique to our work is that we compare estimates across levels of specificity to assess coherence
and use ground-truth data to examine the structure of consistency violations.


\section{Expressing aggregates through partitions}\label{sec:census}

To instantiate the law of total probability, we systematically slice a population into recursively defined and mutually exclusive subpopulations, and then assess the consistency of LLM estimates aggregated over the partitions induced at different tree levels.

\subsection{Constructing partitions using binary trees}
We use a binary conditioning tree (BCT) to recursively partition a base population into increasingly fine-grained subpopulations.\footnote{We refer to a population in a statistical sense as an arbitrary distribution over some instance space $\mathcal{X}$.}
The root node corresponds to the base population $\mathcal{X}$, and each level $\ell$ applies the same binary split function $\phi_\ell:\mathcal{X}\rightarrow\{0,1\}$ to every node from the preceding level. Data points that satisfy $\phi_\ell(x) = 0$ are assigned to the left child, while data points with $\phi_\ell(x) = 1$ are assigned to the right child. As a result, the tree induces a different partition of $\mathcal{X}$ at every level. 

\paragraph{Persona trees as a testbed.}
As a case study, we construct a binary tree over the population of individuals living in the United States, using data from the 2024 American Community Survey (ACS)\footnote{\url{https://www.census.gov/programs-surveys/acs}} as the human reference population.
The ACS dataset contains records for 3.8 million individuals and captures demographic, employment, health, and related attributes. Due to its large sample size, rigorous sampling design, and population weighting, it offers one of the best representations of the {U.S.} population. 
For a fixed set of attributes, we select balanced splits that reduce the prior-weighted within-subgroup variance of yearly income, analogous to regression trees \citep{decision_tree}. This yields a split by age $\phi_1(x) = \mathbbold{1}\{31 \leq \operatorname{age}(x) \leq 68\}$, followed by employment status $\phi_2(x) = \mathbbold{1}\{x \text{ is employed or self-employed}\}$.
In \cref{subsec:app:experiment_details:bct}, we provide additional details on split selection and data preprocessing, including the ACS attribute mappings and income binning.

\subsection{Obtaining conditional LLM estimates at varying levels of granularity}
For our ACS case study, we instantiate the event of interest $E$ as the income exceeding a specified threshold, which mimics the popular UCI Adult benchmark task~\citep{adult_2, ding2021retiring}. Specifically, for a threshold $\tau\in\mathbb{R}$, we take $E_\tau = \{Y > \tau\}$, where $Y$ denotes yearly income. For any subpopulation $\mathcal{S}\subseteq\mathcal{X}$, we define the corresponding tail probability as
\begin{equation}
T_\tau(\mathcal{S})\triangleq\mathbb{P}(Y > \tau\given X\in\mathcal{S}).
\end{equation}
We denote the LLM estimate of this quantity by $\hat{T}_\tau(\mathcal{S})$. Varying $\tau$ lets us probe binary target distributions with different degrees of entropy. The entropy is maximized when $T_\tau(\mathcal{S}) = 1/2$ and decreases as the tail probability approaches zero or one.

\paragraph{Verbalizing split criteria for in-context learning.}
To implement this conditioning structure, we pair the estimation query with a natural-language description of the subpopulation specific to each BCT node. This description consists of a base context $C_0$ specifying the reference population, followed by the verbalized branch restrictions along the root-to-node path. This construction follows common practice in sociodemographic prompting~\citep{beck2024sensitivityperformancerobustnessdeconstructing}. When the context consists only of $C_0$, we refer to the resulting method as \emph{direct prompting}. The base context $C_0$ may also be empty, in which case the base population is left implicit. For the persona tree, we use the folktexts package \citep{cruz2024evaluatinglanguagemodelsrisk} to systematically translate attributes into persona descriptions.
The exact prompt template is given in \cref{subsec:app:sec:prompt_template:socio}, with prompt-sensitivity ablations reported in \cref{subsub:appendix:persona_vs_socio,subsub:appendix:verb_split_crit}.

\subsection{Reconstructing aggregate estimates}

Aggregating conditional estimates according to the law of total probability also requires subgroup priors, which we elicit from LLMs using a level-wise protocol.
For each tree level, we query the model for the prior distribution over the corresponding partition, repeat this multiple times, renormalize the raw outputs, and average the resulting estimates. We denote the resulting LLM-estimated prior of subpopulation $\mathcal{S}$ by $\hat{\gamma}_\mathcal{S}$. Further details are provided in \cref{subsec:app:prior_elicitation}.
Combining the elicited conditionals and priors yields an LLM-reconstructed aggregate at each level $\ell$:
\begin{equation}\label{eq:llm_reconstructed_aggregate}
\hat{T}_{\tau, \textnormal{agg}}^{(\ell)}(\mathcal{X})\triangleq{\sum}_{\mathcal{S}\in\mathcal{P}_\ell(\mathcal{X})} \hat{\gamma}_\mathcal{S}\,\hat{T}_\tau(\mathcal{S}).
\end{equation}
Under the conditional-inference interpretation of in-context learning, each node prompt queries the same target outcome for a specific subpopulation. The branch restrictions along the root-to-node path act as conditioning events that define this subpopulation relative to the base population.
Under these assumptions, the elicited quantities $\hat{T}_\tau(\mathcal{S})$ and $\hat{\gamma}_{\mathcal{S}}$ should satisfy the same aggregation identity as their survey-based counterparts.


\section{The macro fallacy}\label{sec:macro_fallacy}

\begin{figure}[t]
\centering
\resizebox{1.0\linewidth}{!}{%
\input{figures/3i_aggregation_relative_bars_depth_6.pgf}
}
\begin{subfigure}[t]{0.48\linewidth}
\vspace*{-3mm}
\caption{
Threshold $\tau = 40\mathrm{k}$ USD for varying models. 
}
\label{subfig:aggregation_refinement:llm}
\end{subfigure}
\hfill
\begin{subfigure}[t]{0.48\linewidth}
\vspace*{-3mm}
\caption{
GPT-5.4 with varying income threshold.
}
\label{subfig:aggregation_relative_models:llm}
\end{subfigure}
\caption{
\textbf{Relative error reduction of reconstructed aggregates over direct prompting.}
We report the relative aggregation gain of the reconstructed estimate, defined as $1 - \mathrm{AErr}(\ell)/\mathrm{AErr}(0)$, as a function of tree depth $\ell$, where $\mathrm{AErr}(0)$ denotes the alignment error of direct prompting.
The black error bars indicate 90\,\% confidence intervals computed by bootstrapping with 1000 samples. The orange curve shows the average aggregation gain. 
Reconstructing aggregates generally improves alignment in the first few tree levels, but the gain diminishes with depth and can become negative at deeper tree levels.
}
\label{fig:aggregation_refinement:llm_prior}
\end{figure}

Each level of the persona tree induces a different partition of the base population and comes with a corresponding LLM-reconstructed aggregate estimate, computed according to \eqref{eq:llm_reconstructed_aggregate}. We start by comparing each of these estimates against survey-based reference data to understand how well they are aligned with ground-truth population statistics. More formally, for a population $\mathcal{X}$ and the partition $\mathcal{P}_\ell(\mathcal{X})$ induced by level $\ell$, we define the aggregate alignment error between the reconstructed population-level estimate $\hat{T}_{\tau,\textnormal{agg}}^{(\ell)}(\mathcal{X})$ and the ACS ground-truth tail probability $T_\tau(\mathcal{X})$ as\looseness=-1
\begin{equation}\label{eq:agg_error}
\text{AErr}(\ell)\triangleq\big\vert\ssp T_\tau(\mathcal{X}) - \hat{T}_{\tau, \textnormal{agg}}^{(\ell)}(\mathcal{X})\ssp\big\vert.
\end{equation}
At level $\ell = 0$, the definition corresponds to a classical notion of alignment between the model and the data \citep{meister2024benchmarkingdistributionalalignmentlarge}. At deeper levels, the LLM estimate is instead obtained by first eliciting subgroup-level conditionals and then aggregating them back to the population level by using prior estimates. Thus, $\text{AErr}(\ell)$ measures how well the same target quantity is recovered when the prompt is posed at the granularity induced by tree level $\ell$. For estimates adhering to probability axioms, the error should not depend on $\ell$. However, we find that invariance fails in practice and the alignment error varies substantially across levels.

\Cref{fig:aggregation_refinement:llm_prior} shows the reduction in alignment error at deeper levels relative to direct prompting across models and problems. We refer to this reduction as the aggregation gain and use it to highlight a pattern we call the \emph{macro fallacy}. Namely, direct root-level estimates often exhibit larger alignment errors relative to the human baseline than reconstructed estimates at lower levels. 
In our experiment, this pattern holds on average up to level four across models (\cref{subfig:aggregation_refinement:llm}) and target distributions of varying entropy (\cref{subfig:aggregation_relative_models:llm}). We further observe that the aggregation gain profile appears concave, with a peak between levels one and four.
Remarkably, this improvement arises solely from decomposing and recombining the population estimates without injecting any ground-truth statistics.
Additional results, ablations, and experimental details are presented in \cref{sec:appendix:disentangling_cond_priors,subsec:app:normalized_alignment_error,subsec:app:exp_details:alignment_error,subsub:appendix:verb_split_crit}.

\subsection{Disentangling error sources in aggregate estimates}

To better understand the findings of the previous section, we examine the quality of the conditionals and the priors separately. 
\Cref{subfig:disentangling:conditionals} shows that conditional estimates tend to improve with depth. The improvement holds for all nodes at the first two levels, and less consistently at level three. However, the two nodes that exhibit large alignment errors have negligible prior mass and contribute little to the prior-weighted average. 
Conversely, priors become harder to estimate at greater depths, as reflected in the increasing total variation distance between estimated and survey-based priors across models (\cref{subfig:disentangling:priors}). Together, these patterns suggest that improved conditional estimates drive aggregation gains, while errors in population weighting increasingly offset the benefits of finer conditioning.

\begin{figure}[t]
\vspace*{2mm}
\centering
\begin{subfigure}[t]{0.32\linewidth}
\resizebox{1.0\linewidth}{!}{%
\begin{tikzpicture}[
  font=\fontsize{6}{7}\selectfont,
  level 1/.style={sibling distance=24mm, level distance=7mm},
  level 2/.style={sibling distance=12mm, level distance=7mm},
  level 3/.style={sibling distance=6mm, level distance=7mm},
  every label/.style={
    font=\scriptsize,
    inner sep=2pt,
    outer sep=1.5pt,
  },
  label distance=1.5pt,
  edge from parent/.style={
    draw=black!65,
    line width=1.2pt,
  },
  treetitle/.style={
    draw=none,
    fill=none,
    align=center,
  },
  treenode/.style={
    circle,
    draw=black!65,
    line width=0.7pt,
    fill=#1,
    align=center,
    inner sep=2pt,
    minimum size=4mm,
  },
  treenode/.default=white,
]
\begin{scope}

\node[treenode=paperyellow!88!paperorange] {}
   child {
      node[treenode=paperdarkgreen!13!papergreen] {}
         child {
            node[treenode=paperdarkgreen!7!papergreen] {}
               child {
                  node[treenode=papergreen!44!paperyellow, label={below:\scriptsize 33\%}] {}
                  edge from parent node {}
               }
               child {
                  node[treenode=paperorange!27!paperred, label={below:\scriptsize 0\%}] {}
                  edge from parent node {}
               }
            edge from parent node {}
         }
         child {
            node[treenode=paperdarkgreen!45!papergreen] {}
               child {
                  node[treenode=papergreen!82!paperyellow, label={below:\scriptsize 10\%}] {}
                  edge from parent node {}
               }
               child {
                  node[treenode=paperdarkgreen!5!papergreen, label={below:\scriptsize 9\%}] {}
                  edge from parent node {}
               }
            edge from parent node {}
         }
      edge from parent node {}
   }
   child {
      node[treenode=paperdarkgreen!59!papergreen] {}
         child {
            node[treenode=paperdarkgreen!51!papergreen] {}
               child {
                  node[treenode=papergreen!17!paperyellow, label={below:\scriptsize 7\%}] {}
                  edge from parent node {}
               }
               child {
                  node[treenode=paperorange!30!paperred, label={below:\scriptsize 0\%}] {}
                  edge from parent node {}
               }
            edge from parent node {}
         }
         child {
            node[treenode=paperdarkgreen!54!papergreen] {}
               child {
                  node[treenode=paperyellow!90!paperorange, label={below:\scriptsize 12\%}] {}
                  edge from parent node {}
               }
               child {
                  node[treenode=paperdarkgreen!47!papergreen, label={below:\scriptsize 29\%}] {}
                  edge from parent node {}
               }
            edge from parent node {}
         }
      edge from parent node {}
   }
;

\end{scope}


\begin{scope}[xshift=-20mm, yshift=4mm]

\def\barwidth{4}
\def\barheight{0.32}
\def\nsteps{200}
\def\maxval{0.50}

\pgfmathtruncatemacro{\nstepsmone}{\nsteps-1}
\foreach \i in {0,...,\nstepsmone} {

    \pgfmathsetmacro{\x}{\i/\nsteps * \barwidth}
    \pgfmathsetmacro{\val}{\i/\nsteps * \maxval}

    \ifdim \val pt < 0.062501pt
        \pgfmathsetmacro{\rawmix}{(1 - \val/0.0625) * 100}
        \pgfmathtruncatemacro{\mix}{\rawmix}
        \colorlet{barcolor}{paperdarkgreen!\mix!papergreen}
    \else
        \ifdim \val pt < 0.125001pt
            \pgfmathsetmacro{\rawmix}{(1 - (\val - 0.0625)/0.0625) * 100}
            \pgfmathtruncatemacro{\mix}{\rawmix}
            \colorlet{barcolor}{papergreen!\mix!paperyellow}
        \else
            \ifdim \val pt < 0.250001pt
                \pgfmathsetmacro{\rawmix}{(1 - (\val - 0.125)/0.125) * 100}
                \pgfmathtruncatemacro{\mix}{\rawmix}
                \colorlet{barcolor}{paperyellow!\mix!paperorange}
            \else
                \pgfmathsetmacro{\rawmix}{(1 - (\val - 0.25)/0.25) * 100}
                \pgfmathtruncatemacro{\mix}{\rawmix}
                \colorlet{barcolor}{paperorange!\mix!paperred}
            \fi
        \fi
    \fi

    \fill[barcolor]
        (\x, 0) rectangle
        (\x + \barwidth/\nsteps + 0.004, \barheight);
}

\draw[black!60, line width=0.9pt] (0,0) rectangle (\barwidth,\barheight);

\pgfmathsetmacro{\tickzero}{0.00/\maxval*\barwidth}
\pgfmathsetmacro{\tickmiddle}{0.25/\maxval*\barwidth}
\pgfmathsetmacro{\tickmax}{0.50/\maxval*\barwidth}


\draw[black!60, line width=0.9pt] (\tickzero, \barheight) -- (\tickzero, \barheight + 0.12);
\draw[black!60, line width=0.9pt] (\tickmiddle, \barheight) -- (\tickmiddle, \barheight + 0.12);
\draw[black!60, line width=0.9pt] (\tickmax, \barheight) -- (\tickmax, \barheight + 0.12);

\node[anchor=south, font=\scriptsize] at (\tickzero, \barheight + 0.1) {$0$};
\node[anchor=south, font=\scriptsize] at (\tickmiddle, \barheight + 0.1) {$0.25$};
\node[anchor=south, font=\scriptsize] at (\tickmax, \barheight + 0.1) {$0.50$};

\end{scope}


\path (current bounding box.south) ++(0,-2mm) coordinate (dummy-bottom);


\end{tikzpicture}
}
\caption{
Node-wise conditional alignment error. In lower levels, nodes with high alignment error tend to have small priors. 
}
\label{subfig:disentangling:conditionals}
\end{subfigure}
\hfill
\begin{subfigure}[t]{0.32\linewidth}
\resizebox{1.0\linewidth}{!}{%
\input{figures/11c_disentangle_priors_new_colors.pgf}
}
\caption{
Distance between LLM-estimated priors and ground-truth values. Priors get less accurate with depth.
}
\label{subfig:disentangling:priors}
\end{subfigure}
\hfill
\begin{subfigure}[t]{0.32\linewidth}
\resizebox{1.0\linewidth}{!}{%
\input{figures/10d_var_combined_new_colors.pgf}
}
\caption{
Average within-subgroup variance decreases with depth (green) while cross-subgroup variance (orange) increases.
}
\label{subfig:explaining_macro_fallacy:variance}
\end{subfigure}
\vspace*{1mm}
\caption{
\textbf{Tracing the origin of the macro fallacy.}
Panels~\ref{subfig:disentangling:conditionals} and~\ref{subfig:disentangling:priors} separate the subgroup conditional estimates and subgroup priors in the reconstructed aggregate estimate. Node-wise conditional errors tend to improve over the root node, yet they vary substantially across the tree. 
Panel~\ref{subfig:explaining_macro_fallacy:variance} illustrates the variance decomposition. As the population is refined, residual heterogeneity within conditioned subgroups decreases, while systematic differences between subgroups become more explicit.
}
\label{fig:macro_fallacy:variance_and_partition_aggregation}
\end{figure}

\pagebreak

Why do conditional estimates improve with depth? One possible explanation is that partitioning reduces the burden of expressing heterogeneity in a single estimate, which is something LLMs struggle with \citep{Bisbee_Clinton_Dorff_Kenkel_Larson_2024}. Partitioning the population mitigates this issue by reducing average residual heterogeneity within subgroups while making systematic differences between them explicit. 
The law of total variance formalizes this decomposition as $\mathbb{V}(Y) = \mathbb{E}[\mathbb{V}(Y\given A)] + \mathbb{V}(\mathbb{E}[Y\given A])$, where $A$ identifies subgroup membership. The within-group variance, $\mathbb{E}[\mathbb{V}(Y\given A)]$, is the average variance within subgroups, while the cross-group variance, $\mathbb{V}(\mathbb{E}[Y\given A])$, captures variation in subgroup means. As the partition is refined, variation shifts from the within-group component to the cross-group component, as illustrated in \cref{subfig:explaining_macro_fallacy:variance}.
Explicit decomposition thus reduces the burden of representing heterogeneous populations while directing the model’s attention to individual subpopulations.


\section{Statistical self-consistency}\label{sec:self_consistency_checks}

A consistent conditional estimator should satisfy the aggregation identities implied by the law of total probability over any valid partition, across trees and problem domains. 
In this section, we turn these criteria into two finite collections of self-consistency checks.
\emph{Split consistency} tests whether estimates elicited at two adjacent levels of a conditioning tree aggregate according to the law of total probability. 
\emph{Order consistency} tests whether direct estimates are invariant under permutations of the order in which the conditioning constraints are verbalized.

\paragraph{General target functionals.}
For a feature vector $X\in\mathcal{X}$, a subpopulation $\mathcal{S}\subseteq\mathcal{X}$, and a measurable function $h$ of the outcome $Y$, we define the target functional as
\begin{equation}
T_h(\mathcal{S})\triangleq\mathbb{E}\bigs[h(Y)\given X\in\mathcal{S}\bigs].
\end{equation}
Tail probabilities are recovered as the special case $h: y\mapsto\mathbbold{1}\{y > \tau\}$.
Crucially, this functional obeys the aggregation identity induced by the law of total probability and is invariant to the ordering of the constraints defining $\mathcal{S}$. We use $\hat{T}_h(\mathcal{S})$ to refer to the LLM estimate elicited under a fixed canonical ordering and we use $\hat{T}_h(\mathcal{S};\pi)$ whenever an alternative ordering $\pi$ is applied to verbalize the constraints.

\subsection{Split consistency}\label{subsec:split_consistency}

Split consistency asks whether an estimate is preserved when reconstructed over a binary attribute split. Given a subpopulation $\mathcal{S}\subseteq \mathcal{X}$ and a binary attribute $A$, we introduce the notation $\hat{T}_{h,\textnormal{agg}}(\mathcal{S};A)$ for the prior-weighted aggregate over the two subpopulations $\mathcal{S}_{0}\triangleq\mathcal{S}\cap\{A = 0\}$ and $\mathcal{S}_{1}\triangleq\mathcal{S}\cap\{A = 1\}$. Split consistency quantifies the discrepancy between this aggregate and the direct estimate at $\mathcal{S}$ using a metric $\dist(\bigcdot,\bigcdot)$ on the space of elicited estimates.

\begin{definition}[Split consistency]\label{def:split_consistency}
Let $\mathcal{S}\subseteq\mathcal{X}$ be a subpopulation and $A$ a binary attribute. The pair $(\mathcal{S},A)$ is $\varepsilon$-split consistent if
\begin{equation}
\dist\!\big(\hat{T}_h(\mathcal{S}),\hat{T}_{h,\textnormal{agg}}(\mathcal{S};A)\big) < \varepsilon.
\end{equation}
\end{definition}

Beyond certifying a single attribute split, we are interested in an aggregate score over a broader set of possible splits occurring in a binary tree induced by a set of multiple attributes $\mathcal{A}$. We thus collect all pairs of conditioning events and admissible next splits induced by $\mathcal{A}$ and summarize them into a consistency score. We use $\mathcal{S}(\mathcal{J},\boldsymbol{a}_{\mathcal{J}})\triangleq{\bigcap}_{j\in\mathcal{J}} \{A_j = a_j\}$ as a shorthand for the conditioning event defined by a given realization of attributes within the index set $\mathcal{J}$. 

\begin{definition}[Split consistency score]\label{def:split_consistency_score}
Given a set of attributes $\mathcal{A} = \{A_1, \dots, A_d\}$ and a threshold $\varepsilon$, the split consistency score is the fraction of satisfied checks
\begin{equation}
\textnormal{SC}_\varepsilon(\mathcal{A})\triangleq\frac{1}{\vert\mathcal{C}_{\textnormal{SC}}(\mathcal{A})\vert} \sum_{(\mathcal{S},A_k)\in\mathcal{C}_{\textnormal{SC}}(\mathcal{A})} \mathbbold{1}\big\{\!\dist\!\big(\hat{T}_h(\mathcal{S}), \hat{T}_{h,\textnormal{agg}}(\mathcal{S};A_k)\big) < \varepsilon\big\},
\end{equation}
with 
$\mathcal{C}_{\textnormal{SC}}(\mathcal{A})\triangleq\big\{\bigs(\mathcal{S}(\mathcal{J}, \boldsymbol{a}_\mathcal{J}), A_k\bigs) : \mathcal{J}\subseteq[d],\;\boldsymbol{a}_\mathcal{J}\in\{0,1\}^{\vert\mathcal{J}\vert},\;k\notin\mathcal{J}\big\}$.
\end{definition}

The split consistency score counts the fraction of split consistency checks passed by the model. $\textnormal{SC}_\varepsilon(\mathcal{A}) = 1$ corresponds to the requirement that every admissible split consistency check is satisfied, while intermediate values quantify partial consistency. This makes the criterion useful as a benchmark score even when no model satisfies all checks.

\subsection{Order consistency}\label{subsec:order_consistency}

Split consistency alone does not make an estimate statistically self-consistent. A complementary requirement is that the estimate assigned to a conditioning event should not depend on the order in which its defining constraints are applied. 
Such invariance cannot be taken for granted, as LLMs are sensitive to reordered but semantically equivalent inputs \citep{chen2024ordering,guan2025ordereffectinvestigatingprompt}. We formalize order invariance as a self-consistency check over a set of binary attributes. 

\begin{definition}[Order consistency]\label{def:order_consistency}
Let $\mathcal{S} = \mathcal{S}(\mathcal{J},\boldsymbol{a}_{\mathcal{J}})$ be a conditioning event with $\vert\mathcal{J}\vert\geq 2$. Let $\Pi(\mathcal{J})$ denote the set of orderings of its defining constraints. We say that $\mathcal{S}$ is $\varepsilon$-order consistent if
\begin{equation}
\dist\!\big(\hat{T}_h(\mathcal{S};\pi), \hat{T}_h(\mathcal{S};\pi')\big) < \varepsilon
\end{equation}
for every pair of distinct orderings $\pi,\pi'\in\Pi(\mathcal{J})$.
\end{definition}

Checking \cref{def:order_consistency} across all $\lvert\mathcal{J}\rvert!$ constraint orderings quickly becomes infeasible with increasing depth of the binary conditioning tree. We therefore restrict the evaluation to conditioning events involving exactly two attributes and construct an order consistency score defined as the fraction of satisfied checks. 

\begin{definition}[Order consistency score]\label{def:order_consistency_score}
Given an attribute set $\mathcal{A} = \{A_1, \dots, A_d\}$, the order consistency score is the fraction of checks satisfied at tolerance $\varepsilon$,
\begin{equation}
\textnormal{OC}_\varepsilon(\mathcal{A})\triangleq\frac{1}{\left\vert\mathcal{C}_{\textnormal{OC}}(\mathcal{A})\right\vert}\sum_{\mathcal{S}\in\mathcal{C}_{\textnormal{OC}}(\mathcal{A})}
\mathbbold{1}\!\left\{\dist\!\big(\hat{T}_h\bigs(\mathcal{S};\pi_1(\mathcal{S})\bigs),\hat{T}_h\bigs(\mathcal{S};\pi_2(\mathcal{S})\bigs)\big) < \varepsilon\right\},
\end{equation}
where $\mathcal{C}_{\textnormal{OC}}(\mathcal{A})\triangleq\big\{\mathcal{S}\bigs(\{i,j\}, (a_i,a_j)\bigs) : 1\leq i < j\leq d,\;(a_i,a_j)\in\{0,1\}^2\big\}$ and $\pi_1(\mathcal{S})$ and $\pi_2(\mathcal{S})$ denote the two possible constraint orderings of the two-attribute event $\mathcal{S}$.
\end{definition}

Restricting the score to pairwise checks makes the evaluation tractable. In the next subsection, we state the condition under which these checks extend to arbitrary conditioning events.

\subsection{Consistency of reconstructed aggregates}

We now lift the local self-consistency checks to a uniform bound on the discrepancy between aggregates computed at different levels of arbitrary BCTs. This result needs two assumptions about the elicitation protocol and the model, which we formalize below.

\begin{assumption}[Shared aggregation weights]\label{ass:shared_priors} 
For each partition, prior weights are elicited once using a fixed canonical ordering of the attributes and reused for every aggregation over that partition, independently of the ordering used to elicit the conditional estimates.
\end{assumption}

\begin{assumption}[Context-independent pairwise order discrepancy]\label{ass:context_independence}
For any pair of attributes $A_i, A_j$, any assignments $a_i, a_j$, and any (possibly empty) prefix, intervening, and suffix conditioning contexts $P$, $Q$, and $R$, the discrepancy
\begin{equation}
\dist\!\big(\hat{T}_h(P, A_i = a_i, Q, A_j = a_j, R), \hat{T}_h(P, A_j = a_j, Q, A_i = a_i, R)\big)
\end{equation}
depends only on the constraints $A_i = a_i$ and $A_j = a_j$, and is therefore independent of the surrounding contexts $P$, $Q$, and $R$.
\end{assumption}

\Cref{prop:main_result:simplified} explains how the two proposed checks complement each other under the assumption of robust prior elicitation and context independence of pairwise order consistency.
More details and the proof are given in \cref{sec:app:proofs}.

\begin{proposition}[Informal uniform consistency bound across all trees and levels]\label{prop:main_result:simplified}
Suppose that \cref{ass:context_independence,ass:shared_priors} hold and that the distance metric is jointly convex under mixtures.\footnote{A metric $\dist(\bigcdot,\bigcdot)$ is jointly convex under mixtures if, for any weights $\lambda_1,\dots,\lambda_n\geq 0$ satisfying $\sum_{i=1}^n\lambda_i=1$, it holds that $\dist\bigs({\sum}_i \lambda_i P_i, {\sum}_i \lambda_i Q_i\bigs)\leq\sum_i \lambda_i \,\dist(P_i, Q_i).$} 
Assume that every split consistency check in $\mathcal{C}_{\textnormal{SC}}(\mathcal{A})$ is satisfied at tolerance $\varepsilon_{\textnormal{SC}}$, every pairwise order consistency check in $\mathcal{C}_{\textnormal{OC}}(\mathcal{A})$ is satisfied at tolerance $\varepsilon_{\textnormal{OC}}$, and every elicited prior at one level deviates by at most $\varepsilon_{\textnormal{PR}}$ from the prior implied by marginalizing the elicited prior at the next level. 
Let $\pi$ and $\pi'$ be any two constraint orderings that define BCTs induced by $\mathcal{A}$, both rooted at the same base population $\mathcal{X}$. If both trees have depth at most $d$, their corresponding level aggregates satisfy
\begin{equation}
\dist\left(\hat{T}^{(\ell)}_{h,\mathrm{agg},\mathcal{T}}(\mathcal{X}; \pi), \hat{T}^{(\ell')}_{h,\mathrm{agg},\mathcal{T}'}(\mathcal{X}; \pi')\right) < 2d\,\varepsilon_{\textnormal{SC}} + 2(d-1)D_h\,\varepsilon_{\textnormal{PR}} + 2(d-1)\ssp\varepsilon_{\textnormal{OC}},
\end{equation}
for every admissible pair of levels $\ell, \ell'$. Here, $D_h$ is the finite diameter of the admissible estimate space, with $D_h=1$ for normalized distance metrics.
\end{proposition}

The bound separates three sources of discrepancy: errors in split consistency, errors in prior elicitation across levels, and errors in pairwise order consistency. Under the stated assumptions, controlling these errors certifies that level aggregates remain close across different constraint orderings and levels.


\begin{figure}[t]
\centering
\resizebox{1.0\linewidth}{!}{%
\input{figures/12_consistency_model_comparison.pgf}
}
\vspace*{-4mm}
\caption{
\textbf{Self-consistency plateaus across model generations.}
We report the split consistency score $\mathrm{SC}_{\varepsilon}(\mathcal{A})$ and order consistency score $\mathrm{OC}_{\varepsilon}(\mathcal{A})$ at tolerance $\varepsilon=0.02$ for the thresholded ACS income prediction task with $\tau=40\mathrm{k}$ USD.
The gray shaded area shows the Artificial Analysis Intelligence (AAI) Index \citep{artificialanalysis2026} as a proxy for general model capability. 
Neither split nor order consistency improves systematically with increasing model capability.
}
\label{fig:self_consistency:model_comp}
\end{figure}

\section{Evaluation}\label{sec:persona_dataset}

Next, we demonstrate the versatility of our self-consistency checks and evaluate them across different tasks. Our self-consistency checks have the appealing property that they are reference-free, which allows us to extend our evaluation beyond surveys to domains where reliable labels are expensive, unavailable, or ill-defined. We instantiate $\dist(\bigcdot, \bigcdot)$ as the absolute difference for binary targets and the normalized Wasserstein-$1$ distance for ordered categorical targets (\cref{sub:appendix:sc:categorical}).

\paragraph{Self-consistency in persona prompting.}

First, we evaluate self-consistency in the ACS setting discussed in prior sections and examine whether self-consistency improves with general model capabilities assessed by standard benchmarks.
We take $\mathcal{A}$ to be the demographic attributes underlying the tree in \cref{sec:census}. For every conditioning event induced by $\mathcal{A}$, we elicit the model's estimated tail probability on the thresholded income task with $\tau = 40\mathrm{k}$ USD. We then apply the self-consistency checks of \cref{sec:self_consistency_checks} to the resulting estimates.
\Cref{fig:self_consistency:model_comp} shows that, despite large increases in the Artificial Analysis Intelligence (AAI) Index \citep{artificialanalysis2026} within model families, neither split nor order consistency improves systematically, with order consistency remaining generally higher. 
Statistical self-consistency therefore does not appear to emerge naturally with stronger overall model performance and remains far from saturated among current models.
Extending this evaluation to decision models in \cref{subsub:appendix:decision_models} reveals that even models trained specifically for decision tasks violate basic probabilistic consistency requirements.
These findings establish statistical self-consistency as a distinct dimension of model quality, complementary to the capabilities measured by contemporary benchmarks.
Implementation details and more results are deferred to \cref{subsec:app:self_consistency:persona,subsec:app:experiment_details:acs_thresholded}.

\paragraph{Self-consistency in forecasting problems.}
Recent work on LLM forecasting, such as Prophet Arena \citep{yang2025llmasaprophetunderstandingpredictiveintelligence}, studies language models as probabilistic forecasters of real-world events and evaluates their predictions using proper scoring rules, calibration, and market-based baselines. Prior work also tests logical relations such as negation and conjunction, using consistency to diagnose failures when ground truth is unavailable \citep{fluri2024superhuman, paleka2025consistency}. We ask whether forecasts compose coherently across explicit partitions of the conditioning space.
We study this requirement in two synthetic settings with binary variables encoding contextual information. In the tennis task, court surface and first-set outcome refine a prediction for a Federer--Nadal match. In the fictional fantasy-combat task, time of day and wind conditions affect the probability of a successful action. The results in \cref{table:self_consistency:tennis,table:self_consistency:fantasy} mirror those for the ACS: self-consistency failures remain widespread across both tasks. Thus, even intuitive forecasting conditions fail to induce a coherent system of conditional forecasts. The fantasy-combat task further highlights the reference-free nature of self-consistency. Although its fictional setting provides no external ground truth against which to assess alignment, its internal consistency can still be evaluated systematically. Further details appear in \cref{subsec:app:experiment_details:tennis,subsec:app:experiment_details:fantasy}.

\begin{table}[t]
\centering
\renewcommand{\arraystretch}{1.4}
\setlength{\tabcolsep}{4pt}
\begin{subtable}[t]{0.48\linewidth}
\centering
\resizebox{1.0\textwidth}{!}{%
\begin{tabular}{m{26mm} >{\centering\arraybackslash}m{20mm} >{\centering\arraybackslash}m{20mm}}
\toprule
Model & Split consistency & Order consistency\\[0.75mm]\hline
\openai~GPT-5.4 & \ccol{EEB0B0}0.33 & \ccol{F5D0D0}0.50 \\
\claude~Sonnet 4.6 & \ccol{F5D0D0}0.50 & \ccol{FCF2F2}0.75 \\
\grok~Grok 4.1 Fast & \ccol{E58484}0.17 & \ccol{FCF2F2}0.75 \\
\qwen~Qwen~3.6 Plus & \ccol{E58484}0.00 & \ccol{EEB0B0}0.25 \\
\specialrule{\heavyrulewidth}{0pt}{0pt}
\end{tabular}
}
\caption{
Tennis forecasting
}
\label{table:self_consistency:tennis}
\end{subtable}
\hspace{0.025\linewidth} 
\begin{subtable}[t]{0.48\linewidth}
\centering
\resizebox{1.0\textwidth}{!}{%
\begin{tabular}{m{26mm} >{\centering\arraybackslash}m{20mm} >{\centering\arraybackslash}m{20mm}}
\toprule
Model & Split consistency & Order consistency\\[0.75mm]\hline
\openai~GPT-5.4 & \ccol{FCF2F2}0.67 & \ccol{FCF2F2}0.75 \\
\claude~Sonnet 4.6 & \ccol{EEB0B0}0.33 & \ccol{FCF2F2}0.75 \\
\grok~Grok 4.1 Fast & \ccol{E58484}0.17 & \ccol{E58484}0.00 \\
\qwen~Qwen~3.6 Plus & \ccol{FFFFFF}1.00 & \ccol{FCF2F2}0.75 \\
\specialrule{\heavyrulewidth}{0pt}{0pt}
\end{tabular}
}
\caption{%
Fantasy combat
}
\label{table:self_consistency:fantasy}
\end{subtable}
\caption{%
\textbf{Self-consistency in forecasting.}
For each task, we report the split consistency score $\mathrm{SC}_\varepsilon(\mathcal{A})$ and the order consistency score $\mathrm{OC}_\varepsilon(\mathcal{A})$ for $\varepsilon = 0.02$. Higher values indicate stronger self-consistency. Self-consistency violations are widespread even in  simple synthetic settings.\looseness=-1
}
\label{table:self_consistency:synthetic_combined}
\end{table}


\section{Conclusion}\label{sec:discussion}

We investigate how closely the probability estimates obtained from LLMs through in-context learning conform to the law of total probability. To do so, we recursively partition populations into increasingly fine-grained subpopulations and compare two estimates of the same quantity derived at different levels of granularity.
We find widespread violations of statistical self-consistency, even among frontier models. Across the models we study, statistical self-consistency shows no apparent correlation with performance on standard capability benchmarks. These findings suggest that statistical self-consistency captures an underexplored dimension of model quality on which even frontier models have room to improve.

One consequence of these violations is that statistically equivalent formulations can yield incompatible answers. This has important implications for the robustness of existing alignment claims and for how we evaluate models’ ability to estimate uncertainty. In particular, conclusions drawn at the aggregate level may not apply at a finer granularity. Thus, an important open question is how to reliably evaluate LLMs in light of these inconsistencies.
A natural next step is to test whether optimizing for statistical self-consistency can narrow the inconsistency gap we observe. A perfectly accurate model necessarily satisfies these consistency criteria, yet satisfying them alone does not guarantee accuracy. Violations could therefore provide a flexible, reference-free training signal that complements conventional measures of accuracy.


\subsection*{Limitations}

We evaluate statistical self-consistency across multiple settings in both persona prompting and synthetic forecasting. The criterion applies more broadly, and it remains to be seen how far our findings generalize beyond these settings. We also want to emphasize that self-consistency is a necessary but not sufficient condition for faithful conditional inference.
The availability of ACS reference data enables a separate analysis of distributional alignment, which reveals the macro fallacy. Whether these alignment gains generalize to other domains is a promising direction for future work.
Alignment measurements also depend on design choices, including verbalized probability elicitation, post-hoc normalization, and demographic splits.
Our synthetic forecasting examples demonstrate the applicability of consistency checks beyond survey-based settings, but do not by themselves establish accuracy in real-world forecasting or decision-making. It remains to evaluate self-consistency across broader domains and richer conditioning structures.


\subsection*{Acknowledgments}

We would like to thank Ana-Andreea Stoica, Vivian Nastl, Florian Dorner, and Léo le Douarec for their valuable feedback on the manuscript. We also thank the anonymous reviewers for their constructive comments on earlier versions of the paper.
PW was supported by the Max Planck ETH Center for Learning Systems. PW also gratefully acknowledges G-Research for providing financial support to attend NeurIPS in Sydney.
TKB was supported by an ETH AI Center Postdoctoral Fellowship.
CMD acknowledges the financial support of the Hector Foundation.


\subsection*{Reproducibility statement}

To support reproducibility, \cref{sec:app:experiment_details} describes the experimental setup, data preprocessing, and evaluation procedures, while \cref{sec:prompt_templates} provides the exact prompt templates. The source code is publicly available at \url{https://github.com/patrikwolf/statistical_self_consistency}. \Cref{sec:app:uniform_consistency} states the assumptions underlying our theoretical results and provides their proofs.


\bibliographystyle{plainnat}
\bibliography{references}


\newpage
\appendix

\crefalias{section}{appendix}
\crefalias{subsection}{appendix}
\crefalias{subsubsection}{appendix}


\section*{\LARGE Appendices}

\section*{Contents}

{
\startcontents
\printcontents{}{0}[2]{}
}


\section{Illustrative reconstruction across tree levels}\label{sec:app:illustrative_example}

\Cref{fig:illustrate_method} provides a schematic illustration of the evaluation procedure introduced in \cref{sec:census}. Starting from a base population, we partition the population into mutually exclusive and exhaustive subgroups, elicit conditional estimates and subgroup priors from the LLM, and reconstruct the population-level estimate using the law of total probability. The figure shows simplified prompts and intermediate quantities involved in the reconstruction.

\begin{figure}[ht]
\centering
\resizebox{1.0\linewidth}{!}{%
\begin{tikzpicture}[
  font=\fontsize{7}{8}\selectfont,
  level 1/.style={sibling distance=62mm, level distance=32mm},
  edge from parent/.style={
    draw=black!55,
    line width=0.9pt,
    edge from parent path={
      (\tikzparentnode.south) -- ++(0,-3mm) -| (\tikzchildnode.north)
    }
  },
  edgelabel/.style={
    fill=white,
    inner sep=2pt,
    align=center,
  },
  treetitle/.style={
    draw=none,
    fill=none,
    align=center,
  },
  treenode/.style={
    draw=black!65,
    rounded corners=3pt,
    line width=0.7pt,
    fill=white,
    align=center,
    inner xsep=3pt,
    inner ysep=3pt,
    text width=30mm,
  },
  promptbox/.style={
    draw=black!50,
    fill=black!4,
    rounded corners=2pt,
    line width=0.6pt,
    align=left,
    inner xsep=4pt,
    inner ysep=3.5pt,
  },
  condbox/.style={
    promptbox,
    draw=CondTeal,
    fill=CondTeal!8,
  },
  priorbox/.style={
    promptbox,
    draw=PriorOrange,
    fill=PriorOrange!8,
    line width=0.8pt,
    inner xsep=6pt,
    inner ysep=5pt,
  },
]
\begin{scope}

\node[treenode] (root) {%
\makebox[\linewidth]{{\bfseries Level 0: root}}\\[2.5pt]
\begingroup
\tikzset{
  every path/.style={},
  every node/.style={},
  every picture/.style={},
}%
\pgfsetcornersarced{\pgfpointorigin}%
\resizebox{\linewidth}{!}{\input{figures/illustrate_method/shared_binned_hist_0.pgf}}%
\endgroup
}
   child {
      node[treenode] (g1) {%
\makebox[\linewidth]{{\textbf{Level 1:} age $\notin[31, 68]$}}\\[2pt]
\begingroup
\tikzset{
  every path/.style={},
  every node/.style={},
  every picture/.style={},
}%
\pgfsetcornersarced{\pgfpointorigin}%
\resizebox{\linewidth}{!}{\input{figures/illustrate_method/shared_binned_hist_1.pgf}}%
\endgroup
      }
      edge from parent node[edgelabel, pos=0.28] {\textbf{young \& old}}
   }
   child {
      node[treenode] (g2) {%
\makebox[\linewidth]{{\textbf{Level 1:} age $\in[31, 68]$}}\\[2pt]
\begingroup
\tikzset{
  every path/.style={},
  every node/.style={},
  every picture/.style={},
}%
\pgfsetcornersarced{\pgfpointorigin}%
\resizebox{\linewidth}{!}{\input{figures/illustrate_method/shared_binned_hist_2.pgf}}%
\endgroup
      }
      edge from parent node[edgelabel, pos=0.28] {\textbf{working age}}
   }
;

\node[condbox, text width=46mm, anchor=south] (rootprompt)
  at ($(root.north)+(0,3.5mm)$)
  {\textbf{Conditional prompt:} \textcolor{CondPart}{Given a random person in the U.S.,} \textcolor{QuestPart}{estimate the probability that this person's annual income falls into each income bin.}};
\draw[CondTeal, densely dashed, line width=0.7pt] (rootprompt.south) -- (root.north);

\node[condbox, text width=46mm, anchor=north] (c1)
  at ($(g1.south)+(0,-3mm)$)
  {\textbf{Conditional prompt:} \textcolor{CondPart}{Given a random person in the U.S.\ aged below 31 or above 68,} \textcolor{QuestPart}{estimate the probability that this person's annual income falls into each income bin.}};
\node[condbox, text width=46mm, anchor=north] (c2)
  at ($(g2.south)+(0,-3mm)$)
  {\textbf{Conditional prompt:} \textcolor{CondPart}{Given a random person in the U.S.\ aged between 31 and 68,} \textcolor{QuestPart}{estimate the probability that this person's annual income falls into each income bin.}};
\draw[CondTeal, densely dashed, line width=0.7pt] (g1.south) -- (c1.north);
\draw[CondTeal, densely dashed, line width=0.7pt] (g2.south) -- (c2.north);

\node[priorbox, text width=106.6mm, anchor=north] (prior)
  at ($($(c1.south)!0.5!(c2.south)$)+(0,-3mm)$)
  {%
\textbf{Prior prompt:} What share of the U.S.\ population belongs to each of the two age groups? Return probabilities summing to one.\\[2pt]
  };
\draw[PriorOrange, densely dashed, line width=0.7pt] (c1.south) -- (c1.south |- prior.north);
\draw[PriorOrange, densely dashed, line width=0.7pt] (c2.south) -- (c2.south |- prior.north);

\path (current bounding box.south east) coordinate (Lse);
\path (current bounding box.north east) coordinate (Lne);

\end{scope}


\begin{scope}[shift={(7.2, 0)}]

\node[treenode] (agg) at (0.2,-3.2) {%
\makebox[\linewidth]{{\bfseries Level-1 aggregate}}\\[2.5pt]
\begingroup
\tikzset{
  every path/.style={},
  every node/.style={},
  every picture/.style={},
}%
\pgfsetcornersarced{\pgfpointorigin}%
\resizebox{\linewidth}{!}{\input{figures/illustrate_method/agg_binned_hist_1.pgf}}%
\endgroup
\\[1pt]
};

\draw[<->, >={Stealth[length=2.4mm]}, MidRed, line width=0.8pt, rounded corners=2mm]
  (root.east) -| (agg.north)
  node[pos=0.35, above=0.5mm, MidRed] {
  \text{Self-consistency discrepancy}
  };

\end{scope}


\draw[PriorOrange, -{Stealth[length=2.4mm]}, densely dashed, line width=0.8pt, rounded corners=2mm]
  (prior.east) -| (agg.south)
  node[pos=0.78, rotate=90, above=0.5mm, text=PriorOrange] {};


\node[anchor=north, text width=33mm, align=center, inner sep=0pt] at (-4.34, 1.475) {%
\begingroup
\tikzset{
  every path/.style={},
  every node/.style={},
  every picture/.style={},
}%
\pgfsetcornersarced{\pgfpointorigin}%
\resizebox{\linewidth}{!}{\input{figures/illustrate_method/vertical_legend.pgf}}%
\endgroup
};


\end{tikzpicture}
}
\vspace{-1mm}
\caption{
\textbf{Illustrative reconstruction of an aggregate estimate from subgroup estimates.}
We show one refinement step of the binary conditioning tree for the ACS income-distribution task.
At level~0, the LLM directly estimates the income distribution of the base population.
At level~1, the population is partitioned by age into two mutually exclusive and exhaustive subgroups. The LLM estimates the conditional income distribution for each subgroup from its verbalized description.
A separate prior-elicitation prompt estimates the relative sizes of the two subgroups. The resulting LLM-estimated priors are used to combine the conditional distributions into the level-1 aggregate according to the law of total probability.
The normalized Wasserstein-$1$ distance between the direct root estimate $\hat{T}_h(\mathcal{X})$ and the reconstructed aggregate $\hat{T}_{h,\textnormal{agg}}(\mathcal{X})$ quantifies the corresponding violation of statistical self-consistency.
All displayed prompts are simplified for readability.
}
\label{fig:illustrate_method}
\end{figure}


\section{Disentangling conditionals and priors}\label{sec:appendix:disentangling_cond_priors}

Deviations in the alignment error across levels in \cref{fig:aggregation_refinement:llm_prior} indicate that model estimates fail to compose consistently under the law of total probability. The concave shape of the aggregation gain suggests that the models' ability to represent the population improves with depth until it starts declining as the partition becomes more fine-grained. In this appendix, we trace where these errors come from. The reconstructed estimate $\hat{T}_{\tau,\textnormal{agg}}^{(\ell)}(\mathcal{X})$ combines two ingredients elicited from the model: the subgroup conditionals $\hat{T}_\tau(\mathcal{S})$ and the subgroup priors $\hat{\gamma}_\mathcal{S}$ used to weight them. Since the aggregate alignment error conflates both, we examine the quality of the two components separately.
The survey data makes this decomposition possible because it provides reference statistics for each subgroup induced by the binary conditioning tree. We can therefore go beyond evaluating whether reconstructed population-level estimates align with the survey-based ground truth, and directly analyze how well the LLM estimates the conditional quantities at different levels of specificity. 
In particular, we first study an oracle-prior regime in which subgroup estimates are aggregated using ground-truth survey priors, thereby isolating the contribution of the LLM-estimated conditional probabilities to the macro fallacy.
We then compare the LLM-estimated priors themselves against the survey-based subgroup proportions to assess how errors in population weighting evolve as the partition becomes finer.

\subsection{Conditional alignment across tree levels}\label{subsec:appendix:gt_prior_aggregation}

The aggregate alignment error in \eqref{eq:agg_error} is computed from both LLM-estimated subgroup conditionals $\hat{T}_\tau(\mathcal{S})$ and LLM-estimated subgroup priors $\hat{\gamma}_\mathcal{S}$, thereby conflating two sources of error. To isolate the error caused by the conditionals, we replace the LLM-estimated priors with ground-truth subgroup priors $\gamma_{\mathcal{S}}$ obtained from the survey data. This yields the conditional alignment error
\begin{equation}\label{eq:cond_agg_error}
\text{CAErr}(\ell)\triangleq\bigg\vert\,T_\tau(\mathcal{X}) - \sum_{\mathcal{S}\in\mathcal{P}_\ell(\mathcal{X})} \gamma_\mathcal{S}\,\hat{T}_\tau(\mathcal{S})\,\bigg\vert.
\end{equation}
This oracle-prior regime removes uncertainty about the population composition and measures how accurately the LLM-estimated conditionals aggregate to the ground-truth population statistic.

\begin{figure}[t]
\centering
\input{figures/3j_aggregation_relative_bars_gt.pgf}
\caption{
\textbf{Conditional aggregation gain using ground-truth priors.}
On average across the considered income thresholds, reconstructing population-level estimates from increasingly fine-grained subgroup estimates using ground-truth subgroup priors improves alignment relative to direct prompting across all considered partitions.
Error bars show 90\,\% confidence intervals computed by bootstrapping with 1000 samples.
}
\label{fig:aggregation_refinement:gt_prior}
\end{figure}

\Cref{fig:aggregation_refinement:gt_prior} reports the corresponding conditional aggregation gain, $1 - \text{CAErr}(\ell)/\text{CAErr}(0)$, across income thresholds. Positive values indicate that reconstructing the aggregate from subgroup estimates improves upon direct prompting. 
Over the first few partition depths, the reconstructed estimates are generally better aligned with the ACS reference statistic than the root-level estimate. At greater depths, however, this advantage can diminish or even reverse. Together with the node-wise results, this pattern shows that the initial aggregation gains are already present in the LLM's conditional estimates, indicating that subgroup-conditioned estimates contain useful population-level information that is not reliably elicited by direct prompting.

\paragraph{Categorical income distributions.}
We next repeat the oracle-prior aggregation experiment for a categorical income target. Instead of estimating the probability of exceeding a fixed threshold, the model is asked to estimate the full distribution over income bins. To quantify alignment for such distribution-valued targets, we generalize the conditional alignment error in \eqref{eq:cond_agg_error} to
\begin{equation}\label{eq:cond_agg_error:wasserstein}
\text{CAErr}(\ell)\triangleq\overline{\mathcal{W}}_1\bigg(T_h(\mathcal{X}),\,{\sum}_{\mathcal{S}\in\mathcal{P}_\ell(\mathcal{X})} \gamma_\mathcal{S}\,\hat{T}_h(\mathcal{S})\bigg),
\end{equation}
where $\overline{\mathcal{W}}_1$ denotes the normalized Wasserstein-$1$ distance. This definition strictly generalizes the thresholded case. A threshold event $\{Y > \tau\}$ induces Bernoulli distributions with atoms at unit distance, for which the Wasserstein-$1$ distance reduces to the absolute difference of tail probabilities, recovering \eqref{eq:cond_agg_error}.

\Cref{fig:app:error_aggregated_dist_binned_setting} shows the reconstructed population-level income distributions at different levels of the BCT, along with the corresponding errors \eqref{eq:cond_agg_error:wasserstein}. The root-level estimate corresponds to direct elicitation of the aggregate income distribution, whereas levels one to three reconstruct this estimate from increasingly refined LLM-estimated subgroup distributions, using ACS ground-truth subgroup priors as weights. The tree structure is detailed in \cref{subsec:app:exp_details:alignment_error}. 
Consistent with the thresholded results, the reconstructed distributions are better aligned with the ACS reference distribution than the direct root-level estimate. This shows that the benefit of subgroup aggregation is not limited to binary threshold events. Further, we observe that the level-3 reconstruction is slightly less accurate than the level-2 reconstruction, suggesting that the gains from refinement are eventually offset by noisier estimates for very fine-grained subgroups.

\begin{figure}[t]
\centering
\resizebox{1.015\textwidth}{!}{%
\begin{tikzpicture}[
  font=\fontsize{7}{8}\selectfont,
  grow=right,
  level distance=40.5mm,
  edge from parent/.style={
    draw=black!75,
    line width=1.6pt,
    -{Stealth[length=2mm, width=2mm]},
  },
  aggTreenode/.style={
    draw=black!65,
    rounded corners=6pt,
    line width=1.6pt,
    fill=white,
    align=center,
    inner xsep=6pt,
    inner ysep=6pt,
    text width=29mm
  },
]

\begin{scope}

\node[aggTreenode, draw={paperyellow!22!paperorange}, label={\bf Level 0 (Root)}] {$\mathrm{CAErr}(0) = 0.14$\vspace*{1mm}
   \centering
\begingroup
\tikzset{
  every path/.style={},
  every node/.style={},
  every picture/.style={},
}%
\pgfsetcornersarced{\pgfpointorigin}%
\resizebox{\linewidth}{!}{%
\input{figures/tree_agg_appendix/agg_binned_hist_0.pgf}
}
\endgroup
}
  child { node[aggTreenode, draw={paperdarkgreen!22!papergreen}, label={\bf Level 1}] {$\mathrm{CAErr}(1) = 0.04$\vspace*{1mm}
         \centering
\begingroup
\tikzset{
  every path/.style={},
  every node/.style={},
  every picture/.style={},
}%
\pgfsetcornersarced{\pgfpointorigin}%
\resizebox{\linewidth}{!}{%
\input{figures/tree_agg_appendix/agg_binned_hist_1.pgf}
}
\endgroup
      }
        child { node[aggTreenode, draw={paperdarkgreen!43!papergreen}, label={\bf Level 2}] {$\mathrm{CAErr}(2) = 0.03$\vspace*{1mm}
               \centering
\begingroup
\tikzset{
  every path/.style={},
  every node/.style={},
  every picture/.style={},
}%
\pgfsetcornersarced{\pgfpointorigin}%
\resizebox{\linewidth}{!}{%
\input{figures/tree_agg_appendix/agg_binned_hist_2.pgf}
}
\endgroup
            }
              child { node[aggTreenode, draw={paperdarkgreen!38!papergreen}, label={\bf Level 3}] {$\mathrm{CAErr}(3) = 0.03$\vspace*{1mm}
                     \centering
\begingroup
\tikzset{
  every path/.style={},
  every node/.style={},
  every picture/.style={},
}%
\pgfsetcornersarced{\pgfpointorigin}%
\resizebox{\linewidth}{!}{%
\input{figures/tree_agg_appendix/agg_binned_hist_3.pgf}
}
\endgroup
                  } } } };

\end{scope}


\definecolor{gtbin0}{RGB}{46,107,142}
\definecolor{gtbin1}{RGB}{30,152,138}
\definecolor{gtbin2}{RGB}{75,194,108}
\definecolor{gtbin3}{RGB}{189,222,38}

\node[anchor=north] at (6.45, 3.4) {
\centering
\begingroup
\tikzset{
  every path/.style={},
  every node/.style={},
  every picture/.style={},
}%
\pgfsetcornersarced{\pgfpointorigin}%
\resizebox{124mm}{!}{%
\input{figures/tree_agg_appendix/horizontal_legend.pgf}
}
\endgroup
};


\begin{scope}[xshift=140mm, yshift=-13.1mm]

\def\barwidth{0.45}
\def\barheight{2.63}
\def\nsteps{200}
\def\maxval{0.20}

\pgfmathtruncatemacro{\nstepsmone}{\nsteps-1}
\foreach \i in {0,...,\nstepsmone} {

    \pgfmathsetmacro{\y}{\i/\nsteps * \barheight}
    \pgfmathsetmacro{\val}{\i/\nsteps * \maxval}

    \ifdim \val pt < 0.050001pt
        \pgfmathsetmacro{\rawmix}{(1 - \val/0.05) * 100}
        \pgfmathtruncatemacro{\mix}{\rawmix}
        \colorlet{barcolor}{paperdarkgreen!\mix!papergreen}
    \else
        \ifdim \val pt < 0.100001pt
            \pgfmathsetmacro{\rawmix}{(1 - (\val - 0.05)/0.05) * 100}
            \pgfmathtruncatemacro{\mix}{\rawmix}
            \colorlet{barcolor}{papergreen!\mix!paperyellow}
        \else
            \ifdim \val pt < 0.150001pt
                \pgfmathsetmacro{\rawmix}{(1 - (\val - 0.10)/0.05) * 100}
                \pgfmathtruncatemacro{\mix}{\rawmix}
                \colorlet{barcolor}{paperyellow!\mix!paperorange}
            \else
                \pgfmathsetmacro{\rawmix}{(1 - (\val - 0.15)/0.05) * 100}
                \pgfmathtruncatemacro{\mix}{\rawmix}
                \colorlet{barcolor}{paperorange!\mix!paperred}
            \fi
        \fi
    \fi

    \fill[barcolor]
        (0, \y) rectangle
        (\barwidth, \y + \barheight/\nsteps + 0.004);
}

\draw[black!60, line width=0.9pt] (0,0) rectangle (\barwidth,\barheight);

\pgfmathsetmacro{\tickzero}{0.00/\maxval*\barheight}
\pgfmathsetmacro{\tickone}{0.05/\maxval*\barheight}
\pgfmathsetmacro{\ticktwo}{0.10/\maxval*\barheight}
\pgfmathsetmacro{\tickthree}{0.15/\maxval*\barheight}
\pgfmathsetmacro{\tickfour}{0.20/\maxval*\barheight}

\draw[black!60, line width=0.9pt] (\barwidth, \tickzero) -- (\barwidth + 0.12, \tickzero);
\draw[black!60, line width=0.9pt] (\barwidth, \tickone) -- (\barwidth + 0.12, \tickone);
\draw[black!60, line width=0.9pt] (\barwidth, \ticktwo) -- (\barwidth + 0.12, \ticktwo);
\draw[black!60, line width=0.9pt] (\barwidth, \tickthree) -- (\barwidth + 0.12, \tickthree);
\draw[black!60, line width=0.9pt] (\barwidth, \tickfour) -- (\barwidth + 0.12, \tickfour);

\node[anchor=west, font=\scriptsize] at (\barwidth + 0.18, \tickzero) {$0$};
\node[anchor=west, font=\scriptsize] at (\barwidth + 0.18, \tickone) {$0.05$};
\node[anchor=west, font=\scriptsize] at (\barwidth + 0.18, \ticktwo) {$0.10$};
\node[anchor=west, font=\scriptsize] at (\barwidth + 0.18, \tickthree) {$0.15$};
\node[anchor=west, font=\scriptsize] at (\barwidth + 0.18, \tickfour) {$0.20$};

\end{scope}


\end{tikzpicture}
}
\vspace*{-4mm}
\caption{
\textbf{Reconstructed aggregate income distributions with ground-truth priors.}
We reconstruct the population-level income distribution at each tree level by aggregating LLM-estimated subgroup distributions using ACS subgroup priors. Solid bars show the survey-based aggregate distribution, while hatched bars show the aggregate reconstructed from LLM-estimated subgroup conditionals. The colored border indicates the conditional alignment error $\mathrm{CAErr}(\ell)$, with red denoting larger errors and green denoting smaller errors.
}
\label{fig:app:error_aggregated_dist_binned_setting}
\end{figure}


\subsection{LLM-estimated priors}\label{subsec:appendix:llm_prior_analysis}

Deeper partitions require estimates for a growing number of increasingly specific subgroup priors. Since aggregation gains persist to greater depths when oracle priors are used, errors in the LLM-estimated priors appear to increasingly offset the benefits of finer-grained conditioning. 
We investigate this possibility by comparing the estimated subgroup priors with their survey-based counterparts.

Following the level-wise elicitation protocol described in the main text and detailed in \cref{subsec:app:prior_elicitation}, let $\gamma_\ell$ denote the survey-based distribution over the subgroups at tree level $\ell$, and let $\hat{\gamma}_\ell$ denote the corresponding LLM-estimated distribution. We quantify their discrepancy using the total variation (TV) distance,
\begin{equation}
\operatorname{TV}(\hat{\gamma}_\ell, \gamma_\ell) = \frac{1}{2}\sum_{\mathcal{S}\in\mathcal{P}_\ell} \big\vert\hat{\gamma}_\ell(\mathcal{S}) - \gamma_\ell(\mathcal{S})\big\vert,
\end{equation}
where $\mathcal{P}_\ell$ denotes the partition induced by level $\ell$. Because TV distance lies in $[0,1]$, it provides a normalized measure of prior misalignment that remains comparable across tree levels, even as the number of subgroups grows with depth.

\begin{figure}[t]
\centering
\resizebox{1.0\linewidth}{!}{%
\input{figures/7_llm_prior_tv_dist.pgf}
}
\begin{subfigure}[t]{0.48\linewidth}
\vspace*{-4mm}
\caption{
BCT with splits: \texttt{AGEP}, \texttt{COW}, \texttt{WKHP}, \texttt{OCCP}
}
\label{subfig:llm_prior_comparison:standard}
\end{subfigure}
\hfill
\begin{subfigure}[t]{0.48\linewidth}
\vspace*{-4mm}
\caption{
BCT with splits: \texttt{COW}, \texttt{MAR}, \texttt{ESR}, \texttt{AGEP}
}
\label{subfig:llm_prior_comparison:different_split}
\end{subfigure}
\caption{
\textbf{LLM-estimated priors become less aligned at deeper tree levels.}
For each tree level, we compare the LLM-estimated distribution over subgroups with the survey-based ground-truth prior distribution from the ACS data using total variation (TV) distance. Across models, prior alignment deteriorates as the tree becomes more fine-grained. Error bars show 90\,\% confidence intervals computed by bootstrapping with 1000 samples.
}
\label{fig:app:llm_prior_comparison}
\end{figure}

\Cref{fig:app:llm_prior_comparison} shows that LLM-estimated priors become less aligned with the survey-based ground truth as the tree becomes more fine-grained. This trend is consistent across the evaluated models. While coarse population splits are estimated relatively accurately, errors increase as the model is asked to assign probability mass across more specific subpopulations. This pattern indicates that estimating the population composition of fine-grained demographic partitions remains challenging for current models and helps explain why empirical aggregation gains diminish at greater depths.
Aggregation therefore entails a trade-off between improved conditional specificity and more difficult population weighting.


\section{Extensions of self-consistency checks}
\label{sec:self-consistency}

We provide additional details for the self-consistency checks introduced in \cref{sec:self_consistency_checks}. 
We first discuss the choice of discrepancy metric for binary and categorical distributions. 
We then interpret self-consistency as operations on a binary conditioning tree.
Finally, we derive the expected variance contraction criterion from the law of total variance.

\subsection{Discrepancy metrics for outcome distributions}\label{sub:appendix:sc:categorical}

The self-consistency checks in \cref{def:split_consistency,def:order_consistency} are stated for a general discrepancy metric. This allows the same aggregation principle to be applied to different types of elicited model outputs. In this subsection, we discuss how the distance metric can be defined in three canonical examples.

\paragraph{Binary distributions.}
In the ACS income setting of \cref{sec:census,sec:macro_fallacy}, the target event is $Y > \tau$, where $Y$ denotes yearly income in USD and $\tau$ is an income threshold. The corresponding tail probability can be written as a special case of our general target functional:
\begin{equation}
T_\tau(\mathcal{S})\triangleq\mathbb{P}(Y > \tau\given X\in\mathcal{S}) = \mathbb{E}\bigs[h(Y)\given X\in\mathcal{S}\bigs]\qquad\text{with}\qquad h: Y\mapsto\mathbbold{1}\{Y > \tau\}.
\end{equation}
Thus, each estimate corresponds to the success probability of a Bernoulli distribution indicating whether an individual in subpopulation $\mathcal{S}$ has income above the threshold. For this binary target, the natural discrepancy between a direct estimate $\hat{T}_\tau(\mathcal{S})$ and the aggregated estimate $\hat{T}_{\tau,\textnormal{agg}}(\mathcal{S})$ from \eqref{eq:llm_reconstructed_aggregate} is the absolute difference
\begin{equation}
\dist\!\big(\hat{T}_\tau(\mathcal{S}), \hat{T}_{\tau,\textnormal{agg}}(\mathcal{S})\big) = \big\vert\hat{T}_\tau(\mathcal{S}) - \hat{T}_{\tau,\textnormal{agg}}(\mathcal{S})\big\vert.
\end{equation}
For Bernoulli distributions, the absolute difference between success probabilities coincides with the (normalized) Wasserstein-$1$ distance on the ordered support $\{0,1\}$.

\paragraph{Ordered categorical outcomes.}
More generally, the outcome of interest may be categorical, with support $\{1,\dots,K\}$. 
In this case, we use $\hat{T}_h(\mathcal{S})$ to denote the elicited conditional distribution over the $K$ categories conditional on membership in $\mathcal{S}$.
Similarly, $\hat{T}_{h,\textnormal{agg}}(\mathcal{S})$ denotes the corresponding prior-weighted aggregate. 
The consistency error can then be defined by using an appropriate distance between probability distributions. For ordinal response options, we use the normalized Wasserstein-$1$ distance,
\begin{equation}
\dist\!\big(\hat{T}_h(\mathcal{S}), \hat{T}_{h,\textnormal{agg}}(\mathcal{S})\big) = \overline{\mathcal{W}}_1\big(\hat{T}_h(\mathcal{S}), \hat{T}_{h,\textnormal{agg}}(\mathcal{S})\big)\triangleq\frac{1}{K-1}\sum_{k=1}^{K-1} \big\vert F(k) - F_\textnormal{agg}(k)\big\vert,
\end{equation}
where $F$ and $F_\textnormal{agg}$ denote the corresponding cumulative distribution functions of $\hat{T}_h(\mathcal{S})$ and $\hat{T}_{h,\textnormal{agg}}(\mathcal{S})$, respectively. 
This metric accounts for how far probability mass must be shifted along the response order. The normalization maps the distance to $[0,1]$, making discrepancies comparable across questions with different numbers of response options.

\paragraph{Nominal categorical outcomes.}
For nominal response sets, the category labels have no meaningful order. We therefore use a permutation-invariant discrepancy, such as total variation distance:
\begin{equation}
\dist\!\big(\hat{T}_h(\mathcal{S}), \hat{T}_{h,\textnormal{agg}}(\mathcal{S})\big) = \operatorname{TV}\big(\hat{T}_h(\mathcal{S}), \hat{T}_{h,\textnormal{agg}}(\mathcal{S})\big).
\end{equation}
If both distributions are represented as probability vectors $p, p_\textnormal{agg}\in\Delta^{K-1}$, this is given by
\begin{equation}
\operatorname{TV}\big(\hat{T}_h(\mathcal{S}), \hat{T}_{h,\textnormal{agg}}(\mathcal{S})\big) = \frac{1}{2}\sum_{k=1}^{K} \big\vert p[k] - p_\textnormal{agg}[k]\big\vert.
\end{equation}
The factor $1/2$ normalizes the distance to lie in $[0,1]$.


\subsection{Wasserstein distance on ordered discrete outcomes}\label{subsec:appendix:wasserstein}

This subsection recalls the closed-form expression of the Wasserstein-$1$ distance on ordered discrete supports and explains the normalization used to compare distributions with different numbers of categories.

\begin{lemma}\label{lem:wasserstein_cdf}
Let $\mu$ and $\nu$ be discrete probability measures on the ordered support $\{1,\dots,K\}$, equipped with the metric $d(i, j) = \vert i - j\vert$, and let $F_\mu$ and $F_\nu$ denote their cumulative distribution functions. Then
\begin{equation}\label{eq:wasserstein_cdf}
\mathcal{W}_1(\mu,\nu) = \sum_{k=1}^{K-1} \big\vert F_\mu(k) - F_\nu(k)\big\vert.
\end{equation}
\end{lemma}

\begin{proof}
For probability measures on the real line, the Wasserstein-$1$ distance admits the classical closed form
\begin{equation}\label{eq:wasserstein:real_line}
\mathcal{W}_1(\mu,\nu) = \int_{\mathbb{R}} \bigs\vert F_\mu(t) - F_\nu(t)\bigs\vert\;\mathrm{d}t,
\end{equation}
see \citep[Proposition 2.17]{opt_transport}. Since $\mu$ and $\nu$ are supported on $\{1,\dots,K\}$, the integrand in equation~\eqref{eq:wasserstein:real_line} vanishes outside $[1, K)$ and is constant on each interval $[k, k+1)$ for $k\in[K - 1]$, where it equals $\vert F_\mu(k) - F_\nu(k)\vert$. Integrating over the intervals $[k, k+1)$ for all $k\in[K-1]$ yields the result in equation~\eqref{eq:wasserstein_cdf}.
\end{proof}

\paragraph{Normalization.}
The Wasserstein-$1$ distance on an ordered support, as introduced in \cref{lem:wasserstein_cdf}, depends on the size of the support. In particular, for probability measures $\mu,\nu$ supported on $\{1,\dots,K\}$ with $K\geq 2$ and metric $d(i,j) = \vert i - j\vert$, the maximum possible distance is $K - 1$, attained when all mass is placed on opposite endpoints. Hence, raw Wasserstein distances are not directly comparable across multiple-choice questions with different numbers of answer options. 
We therefore use the normalized Wasserstein-$1$ distance
\begin{equation}
\overline{\mathcal{W}}_1(\mu, \nu)\triangleq\frac{1}{K-1}\,\mathcal{W}_1(\mu,\nu) = \frac{1}{K-1} \sum_{k=1}^{K-1} \big\vert F_\mu(k) - F_\nu(k)\big\vert.
\end{equation}
This normalization maps the distance to $[0,1]$ independently of the number of ordered categories, with $0$ indicating identical distributions and $1$ indicating maximal separation on the support.

\paragraph{Bernoulli distributions.}
For Bernoulli distributions $\mu = \operatorname{Ber}(p)$ and $\nu = \operatorname{Ber}(q)$ on the ordered support $\{0,1\}$, equation~\eqref{eq:wasserstein_cdf} contains a single cumulative difference and therefore yields
\begin{equation}\label{eq:bernoulli_wasserstein}
\mathcal{W}_1\bigs(\operatorname{Ber}(p), \operatorname{Ber}(q)\bigs) = \big\vert(1 - p) - (1 - q)\big\vert = \vert p - q\vert.
\end{equation}
Thus, the absolute error used in our binary experiments is exactly the Wasserstein-$1$ distance between the corresponding Bernoulli distributions.


\subsection{Tree operations}\label{sec:app:sc:tree_permutation}

The self-consistency checks introduced in \cref{sec:self_consistency_checks} admit a natural interpretation in terms of operations on binary conditioning trees. Split consistency concerns the aggregation relations between parent and child nodes, whereas order consistency compares equivalent conditioning events obtained under different orders of the tree levels. Interchanging two adjacent levels illustrates how these two requirements complement one another.

\paragraph{Split consistency.}
Consider a node representing a conditioning event $\mathcal{S}$ and an admissible splitting attribute $A_k$. Refining $\mathcal{S}$ according to the two values of $A_k$ yields the child events
\begin{equation}
\mathcal{S}_{k,0} = \mathcal{S}\cap\{A_k = 0\}\qquad\text{and}\qquad\mathcal{S}_{k,1} = \mathcal{S}\cap\{A_k = 1\},
\end{equation}
which form a partition of $\mathcal{S}$. Split consistency requires the estimate elicited directly at the parent node to agree, up to the prescribed tolerance, with the prior-weighted aggregate of the estimates at its children. In \cref{fig:tree_interchange}, this corresponds to reconstructing the orange parent node from the two red child nodes. Thus, every parent--child relation in a conditioning tree induces a local split consistency check.

\begin{figure}[t]
\centering
\input{tree_operation}
\caption{
\textbf{Tree interpretation of split and order consistency.}
Within each tree, a parent estimate should agree with the prior-weighted aggregate of its two child estimates. Interchanging two adjacent levels preserves the leaf events but changes both the order in which their constraints are presented and the intermediate conditioning events exposed by the tree.
}
\label{fig:tree_interchange}
\end{figure}

A fixed conditioning tree exposes only those refinements associated with its prescribed order of split attributes. Other attributes may nevertheless define equally valid next splits of the same conditioning event. For example, after interchanging two adjacent levels in \cref{fig:tree_interchange}, the pink node in the tree on the right is obtained by aggregating the two blue nodes shown in the tree on the left. Because these blue nodes belong to different subtrees of the original tree, this parent--child relation is not exposed there. This observation motivates defining the split consistency score in \cref{def:split_consistency_score} over all pairs of conditioning events and admissible next splits, rather than only over the parent–child relations of a single fixed tree.

\paragraph{Order consistency.}
Interchanging tree levels also changes the order in which the corresponding attribute constraints are presented to the model. For example, the tree on the left of \cref{fig:tree_interchange} first splits by sex and then by age, whereas the tree on the right first splits by age and then by sex. Nevertheless, corresponding leaves define the same conditioning events. Order consistency therefore requires the direct estimates elicited for corresponding leaves to agree, up to the prescribed tolerance.


\subsection{Beyond binary partitions}\label{subsec:app:beyond-binary}

Throughout the paper, we construct partitions using binary conditioning trees. This is a design choice rather than a conceptual restriction. The law of total probability in \eqref{eq:lotp_decomposition} holds for any valid partition of the population. In particular, suppose a subpopulation $\mathcal{S}$ is partitioned into $K$ mutually disjoint and jointly exhaustive child subpopulations $\mathcal{S}_1,\dots,\mathcal{S}_K$. The corresponding aggregate estimate is
\begin{equation}
\hat{T}_{h,\textnormal{agg}}(\mathcal{S}) = \sum_{k=1}^{K} \hat{\gamma}_k \ssp \hat{T}_h(\mathcal{S}_k),
\end{equation}
where $\hat{\gamma}_k$ denotes the estimated conditional prevalence of $\mathcal{S}_k$ within $\mathcal{S}$. The split consistency check in \cref{def:split_consistency} extends directly by comparing this $K$-term aggregate with the direct estimate $\hat{T}_h(\mathcal{S})$.

We nevertheless adopt binary trees for two practical reasons. 
First, binary splits do not restrict the granularity of the analysis. Any $K$-way partition can be represented by a sequence of $K-1$ recursive binary splits, so binary trees generate arbitrarily fine partitions of the base population. 
Second, binary splits simplify the construction of valid partitions. Defining one child as the complement of the other within the parent guarantees by construction that every split is mutually exclusive and jointly exhaustive, as required by the law of total probability. This makes it straightforward to build valid conditioning trees at varying levels of granularity.


\subsection{Variance contraction}\label{sub:appendix:sc:variance_contraction}

The self-consistency checks in \cref{sec:self_consistency_checks} constrain elicited estimates through first-moment identities. The law of total probability prescribes exactly how conditional estimates must aggregate across a partition. The same reasoning extends to higher-order moments of the outcome distribution. In this subsection, we illustrate this with an auxiliary check derived from the law of total variance. In particular, we reduce the law of total variance to a one-sided constraint. Refining a subpopulation can only reduce the outcome variance in expectation, since the variance of the conditional means across the children is non-negative. 

\begin{definition}[Expected variance contraction]\label{def:self_consistency_check:variance_contraction}
Given a tree $\mathcal{T}$, we say a model satisfies expected variance contraction if for every internal node $v$ it holds that
\begin{equation}
\hat{\mathbb{V}}(\mathcal{S}_v) \geq \hat{\gamma}_L\hat{\mathbb{V}}(\mathcal{S}_{v,L}) + \hat{\gamma}_R\hat{\mathbb{V}}(\mathcal{S}_{v,R}),
\end{equation}
where $\hat{\mathbb{V}}(\bigcdot)$ denotes the variance implied by the model’s estimated (conditional) distribution.
\end{definition}
 
We now derive the expected variance contraction condition used in \cref{def:self_consistency_check:variance_contraction} from the law of total variance. For two random variables $Y$ and $A$, the law of total variance states that
\begin{equation}
\mathbb{V}(Y) = \mathbb{E}\bigs[\mathbb{V}(Y\given A)\bigs] + \mathbb{V}\bigs(\mathbb{E}[Y\given A]\bigs).
\end{equation}
Since the second term is non-negative, conditioning on $A$ can only reduce variance in expectation:
\begin{equation}\label{eq:appendix:variance_contraction:generic}
\mathbb{V}(Y)\geq \mathbb{E}\bigs[\mathbb{V}(Y\given A)\bigs].
\end{equation}
Importantly, this is an averaged statement over the values of $A$; it does not imply that the conditional variance decreases for every realization of $A$ individually. We now instantiate \eqref{eq:appendix:variance_contraction:generic} at an internal node of a binary conditioning tree. Let $\mathcal{S}$ be a parent subpopulation with children $\mathcal{S}_L$ and $\mathcal{S}_R$, and let $A$ encode child membership within $\mathcal{S}$. That is, $A = L$ corresponds to $X\in\mathcal{S}_L$ and $A = R$ corresponds to $X\in\mathcal{S}_R$, conditional on $X\in\mathcal{S}$. Applying \eqref{eq:appendix:variance_contraction:generic} within the parent subpopulation $\mathcal{S}$ yields
\begin{equation}
\mathbb{V}(Y\given X\in\mathcal{S}) \geq \gamma_L\,\mathbb{V}(Y\given X\in \mathcal{S}_L) + \gamma_R\,\mathbb{V}(Y\given X\in \mathcal{S}_R),
\end{equation}
where $\gamma_L = \mathbb{P}(X\in\mathcal{S}_L\given X\in\mathcal{S})$ and $\gamma_R = \mathbb{P}(X\in\mathcal{S}_R\given X\in\mathcal{S})$ are the child priors within the parent. Thus, if the LLM estimates form a self-consistent system of conditional distributions, the same inequality should hold for the corresponding LLM-estimated variances, denoted by $\hat{\mathbb{V}}$. In our check, we replace the population variances and child priors by the corresponding model-induced quantities, yielding the condition stated in \cref{def:self_consistency_check:variance_contraction}.


\section{Uniform consistency guarantees}\label{sec:app:uniform_consistency}

This appendix develops the technical results underlying our uniform consistency guarantee. We first establish auxiliary bounds and then combine them to control aggregate estimates across tree levels and tree structures.

\subsection{Notation and auxiliary results}

Throughout this appendix, let $\mathcal{T}$ be a binary conditioning tree of depth $L\leq d$, induced by $\mathcal{A} = \{A_1,\dots,A_d\}$ and rooted at the base population $\mathcal{X}$. We index constraints by their attribute, so that constraint $i$ is the constraint imposed on $A_i$, and we index positions by $\{1,\dots,d\}$ likewise. A \emph{constraint ordering} is then a bijection $\pi: \{1, \dots, d\}\to\{1, \dots, d\}$, where $\pi(i)$ is the position assigned to constraint $i$. The canonical ordering is denoted by $\overline{\pi}$.
For a conditioning event $\mathcal{S}$, we write $\pi_{\mathcal{S}}$ and $\overline{\pi}_{\mathcal{S}}$ for the restrictions of these orderings to the constraints defining $\mathcal{S}$. 
For $r\in\{0,\dots,L\}$, let $\mathcal{L}_r(\mathcal{T})$ denote the nodes at level $r$ and let $\hat{\gamma}_r$ be the LLM-estimated prior distribution on $\mathcal{L}_r(\mathcal{T})$.
For $v\in\mathcal{L}_r(\mathcal{T})$, denote its conditioning event by $\mathcal{S}_v$ and write $\pi_v\triangleq\pi_{\mathcal{S}_v}$. We define the aggregate estimate obtained from level $r$ as
\begin{equation}\label{eq:level_aggregate}
\hat{T}^{(r)}_{h,\mathrm{agg}}(\mathcal{X}; \pi) \triangleq \sum_{v\in\mathcal{L}_r(\mathcal{T})} \hat{\gamma}_r(v)\,\hat{T}_h(\mathcal{S}_v;\pi_v),
\end{equation}
which at $r = 0$ reduces to the direct root estimate $\hat{T}_h(\mathcal{X})$, independently of $\pi$.
To compare aggregates at adjacent levels, we marginalize the child-level prior onto the parent-level partition. Specifically, since every level-$(r+1)$ node has a unique parent at level $r$, the prior $\hat{\gamma}_{r+1}$ induces a prior over $\mathcal{L}_r(\mathcal{T})$ by summing the masses of each parent’s children:
\begin{equation}\label{eq:def:lifted_priors}
\hat{\gamma}^{\uparrow}_{r+1}(v) \triangleq \sum_{b\in\{0,1\}} \hat{\gamma}_{r+1}(v,b).
\end{equation}
Because the priors are elicited separately at each level, $\hat{\gamma}_r$ need not agree with the marginal of $\hat{\gamma}_{r+1}$. We quantify their deviation by
\begin{equation}\label{def:adjacent_lvl_prior_err}
\delta_r(\mathcal{T})\triangleq\operatorname{TV}(\hat{\gamma}_r,\hat{\gamma}^{\uparrow}_{r+1}). 
\end{equation}
To define the local relative size weights used in the split consistency check, fix a level $r\in\{0,\dots,L - 1\}$ and let $A_{k(r)}$ be the attribute on which every node at that level is split. For $v\in\mathcal{L}_r(\mathcal{T})$, write $\{(v,b)\}_{b\in\{0,1\}}$ for its two children.
If $\hat{\gamma}_{r+1}^{\uparrow}(v) > 0$, the local weights are obtained by restricting the level-$(r+1)$ prior to the children of $v$ and renormalizing:
\begin{equation}\label{eq:local_relative_size}
\hat{\gamma}_{k(r),b\given\mathcal{S}_v}\triangleq\frac{\hat{\gamma}_{r+1}(v,b)}{\hat{\gamma}_{r+1}^{\uparrow}(v)}, \qquad b\in\{0,1\}.
\end{equation}
For an ordering $\pi$, define the local split aggregate at $v\in\mathcal{L}_r(\mathcal{T})$ by
\begin{equation}\label{eq:ordered_split_aggregate}
\hat{T}_{h,\mathrm{agg}}(\mathcal{S}_v;A_{k(r)}; \pi) \triangleq \sum_{b\in\{0,1\}} \hat{\gamma}_{k(r),b\given\mathcal{S}_v}\,\hat{T}_h\big(\mathcal{S}_{v,b}; \pi_{\mathcal{S}_{v,b}}\big).
\end{equation}
If $\hat{\gamma}_{r+1}^{\uparrow}(v) = 0$, we use the convention $\hat{T}_{h,\mathrm{agg}}(\mathcal{S}_v;A_{k(r)};\pi) = \hat{T}_h(\mathcal{S}_v;\pi_v)$.\footnote{The choice does not affect the level-$(r+1)$ aggregate because $v$ receives zero mass under the induced prior. Setting its split aggregate equal to the direct estimate makes the associated split consistency check trivially satisfied.}
Whenever no ordering is specified, the canonical ordering is used.
For two orderings $\pi$ and $\pi'$ of the same set of $m$ constraints, let $\cay(\pi,\pi')$ denote the transposition distance, defined as the minimum number of pairwise swaps required to transform $\pi$ into $\pi'$. We define $\sigma\triangleq\pi'\circ\pi^{-1}$ as the permutation of positions that carries the arrangement $\pi$ to the arrangement $\pi'$. Letting $c(\sigma)$ be its number of cycles (counting fixed points), the transposition distance satisfies Cayley's identity $\cay(\pi, \pi') = m - c(\sigma)$ \citep[][pp.~112, 118]{diaconis1988group}. Since $c(\sigma)\geq 1$, this yields the diameter bound
\begin{equation}\label{eq:cayley-diameter}
\cay(\pi,\pi')\leq m - 1,
\end{equation}
with equality precisely when $\sigma$ is an $m$-cycle.
Because a BCT splits on one shared attribute per level, every conditioning event $\mathcal{S}_v$ with $v\in\mathcal{L}_r(\mathcal{T})$ is defined by the same set of $r$ constraints. Thus, $\cay(\pi_v, \overline{\pi}_v)$ does not depend on $v$. We denote this common value by $\kappa_r(\mathcal{T}, \pi)$, with $\kappa_0(\mathcal{T}, \pi)\triangleq 0$. Applying \eqref{eq:cayley-diameter} with $m = r$ then gives
\begin{equation}\label{eq:kappa-linear}
\kappa_r(\mathcal{T}, \pi)\leq\max\{r - 1,\,0\}.
\end{equation}
Further, we assume that the discrepancy metric $\dist(\bigcdot, \bigcdot)$ is jointly convex under finite mixtures. That is, for any weights $\lambda_1,\dots,\lambda_n\geq 0$ satisfying $\sum_{i=1}^n\lambda_i=1$, it holds that 
\begin{equation}
\dist\left({\sum}_i \lambda_i P_i, {\sum}_i \lambda_i Q_i\right)\leq\sum_i \lambda_i \,\dist(P_i, Q_i).
\end{equation}
Finally, let $\mathcal{E}_h$ be the convex hull of the admissible values of the target functional $T_h$, and define its diameter under the discrepancy metric as
\begin{equation}
D_h\triangleq\sup_{U,V\in\mathcal{E}_h}\dist(U,V).
\end{equation}
We assume that $D_h$ is finite. For the normalized discrepancy metrics and target spaces used in the paper, we have $D_h = 1$.

\begin{lemma}[Mixture sensitivity to weights]\label{lemma:mixture_weight_sensitivity}
For estimates $T_1, \dots, T_n\in\mathcal{E}_h$ and probability vectors $\alpha, \beta\in\Delta^{n-1}$,
\begin{equation}
\dist\left(\sum_{i = 1}^n \alpha_i\ssp T_i, \sum_{i = 1}^n \beta_i\ssp T_i\right)\leq D_h\operatorname{TV}(\alpha,\beta).
\end{equation}
\end{lemma}

\begin{proof}
Let $\tau = \operatorname{TV}(\alpha, \beta)$ and $c_i = \min\{\alpha_i,\beta_i\}$. The two mixtures share mass $\sum_i c_i = 1 - \tau$. If $0 < \tau < 1$, we decompose both mixtures into their shared part with weight $1 - \tau$ and their respective residual parts with weight $\tau$:
\begin{equation}
T_{\mathrm{sh}}\triangleq\frac{1}{1 - \tau} \sum_{i=1}^n c_i\ssp T_i,\qquad T_{\alpha}^{\mathrm{res}}\triangleq\frac{1}{\tau}\sum_{i = 1}^n (\alpha_i - c_i)\ssp T_i, \qquad T_{\beta}^{\mathrm{res}}\triangleq\frac{1}{\tau} \sum_{i=1}^n (\beta_i - c_i)\ssp T_i.
\end{equation}
Joint convexity then gives
\begin{equation}
\dist\left(\sum_{i = 1}^n \alpha_i\ssp T_i, \sum_{i = 1}^n\beta_i\ssp T_i\right) \leq (1 - \tau)\,{\underbrace{\dist(T_{\mathrm{sh}}, T_{\mathrm{sh}})}_{=\;0}} + \tau\dist(T_{\alpha}^{\mathrm{res}}, T_{\beta}^{\mathrm{res}})\leq \tau D_h.
\end{equation}
The cases $\tau = 0$ and $\tau = 1$ follow directly.
\end{proof}

\begin{lemma}[Split consistency across levels]\label{lemma:multi_step_split_cons}
Fix $0\leq\ell<\ell'\leq L$. If every internal node at levels $\ell,\dots,\ell'-1$, paired with its split attribute, is $\varepsilon_{\textnormal{SC}}$-split consistent under the canonical ordering, then
\begin{equation}\label{eq:multi_step_with_prior_error}
\dist\left(\hat{T}^{(\ell)}_{h,\textnormal{agg}}(\mathcal{X}; \overline{\pi}), \hat{T}^{(\ell')}_{h,\textnormal{agg}}(\mathcal{X}; \overline{\pi})\right) < (\ell' - \ell)\ssp\varepsilon_{\textnormal{SC}} + D_h\sum_{j = \ell}^{\ell' - 1} \delta_j(\mathcal{T}),
\end{equation}
where $\delta_r(\mathcal{T})$ is the adjacent-level prior-refinement discrepancy.
\end{lemma}

\begin{proof}
By the triangle inequality, we have
\begin{equation}\label{eq:proof_main_thm:telescope}
\dist\left(\hat{T}^{(\ell)}_{h,\textnormal{agg}}(\mathcal{X}; \overline{\pi}),\hat{T}^{(\ell')}_{h,\textnormal{agg}}(\mathcal{X}; \overline{\pi})\right)\leq\sum_{j = \ell}^{\ell' - 1} \dist\left(\hat{T}^{(j)}_{h,\textnormal{agg}}(\mathcal{X}; \overline{\pi}), \hat{T}^{(j+1)}_{h,\textnormal{agg}}(\mathcal{X}; \overline{\pi})\right).
\end{equation}
Thus, it suffices to bound the discrepancy between each pair of adjacent levels. For fixed $\ell \leq j < \ell'$, grouping the level-$(j+1)$ nodes by their parents using the lifted relative size estimates \eqref{eq:local_relative_size} yields
\begin{align}
\hat{T}^{(j+1)}_{h,\textnormal{agg}}(\mathcal{X};\overline{\pi}) &= \sum_{v\in\mathcal{L}_j(\mathcal{T})} \hat{\gamma}_{j+1}^{\uparrow}(v) \sum_{b\in\{0,1\}} \hat{\gamma}_{k(j),b\given\mathcal{S}_v}\,\hat{T}_h\big(\mathcal{S}_{v,b}; \overline{\pi}_{\mathcal{S}_{v,b}}\big)\\
&= \sum_{v\in\mathcal{L}_j(\mathcal{T})} \hat{\gamma}_{j+1}^{\uparrow}(v)\,\hat{T}_{h,\mathrm{agg}}(\mathcal{S}_v;A_{k(j)}; \overline{\pi}).
\end{align}
Using the shorthand notation $Q_v\triangleq\hat{T}_{h,\mathrm{agg}}(\mathcal{S}_v; A_{k(j)}; \overline{\pi})$ and inserting the intermediate mixture $\sum_v\hat{\gamma}_j(v)\ssp Q_v$ gives
\begin{equation}\label{eq:regroup_level_prior}
\begin{split}
\dist\left(\hat{T}^{(j)}_{h,\textnormal{agg}}(\mathcal{X}; \overline{\pi}), \hat{T}^{(j+1)}_{h,\textnormal{agg}}(\mathcal{X}; \overline{\pi})\right)&\leq\dist\left(\hat{T}^{(j)}_{h,\textnormal{agg}}(\mathcal{X}; \overline{\pi}), {\sum}_{v\in\mathcal{L}_j(\mathcal{T})} \hat{\gamma}_j(v)\ssp Q_v\right)\\
&+ \dist\left({\sum}_{v\in\mathcal{L}_j(\mathcal{T})} \hat{\gamma}_j(v)\ssp Q_v, {\sum}_{v\in\mathcal{L}_j(\mathcal{T})} \hat{\gamma}_{j+1}^{\uparrow}(v)\ssp Q_v\right).
\end{split}
\end{equation}
The first summand can be bounded by using joint convexity and local split consistency:
\begin{equation}
\dist\left(\hat{T}^{(j)}_{h,\textnormal{agg}}(\mathcal{X}; \overline{\pi}), {\sum}_{v\in\mathcal{L}_j(\mathcal{T})} \hat{\gamma}_j(v)\ssp Q_v\right)\leq \sum_{v\in\mathcal{L}_j(\mathcal{T})} \hat{\gamma}_j(v)\ssp\dist\!\big(\hat{T}_h(\mathcal{S}_v; \overline{\pi}_v), Q_v\big) < \varepsilon_\textnormal{SC}.
\end{equation}
By \cref{lemma:mixture_weight_sensitivity}, the second summand is bounded by
\begin{equation}
\dist\left({\sum}_{v\in\mathcal{L}_j(\mathcal{T})} \hat{\gamma}_j(v)\ssp Q_v, {\sum}_{v\in\mathcal{L}_j(\mathcal{T})} \hat{\gamma}_{j+1}^{\uparrow}(v)\ssp Q_v\right) \leq D_h\ssp\delta_j(\mathcal{T}).
\end{equation}
Combining these bounds with \eqref{eq:regroup_level_prior} yields
\begin{equation}
\dist\left(
\hat{T}^{(j)}_{h,\textnormal{agg}}(\mathcal{X};\overline{\pi}),
\hat{T}^{(j+1)}_{h,\textnormal{agg}}(\mathcal{X};\overline{\pi})
\right)
<\varepsilon_{\textnormal{SC}}+D_h\delta_j(\mathcal{T}).
\end{equation}
Summing over $j = \ell, \dots, \ell' - 1$ proves the claim.
\end{proof}

\begin{lemma}[Sufficiency of pairwise checks]\label{lemma:pairwise_order_suff}
Suppose \cref{ass:context_independence} holds and every (pairwise) check in $\mathcal{C}_{\textnormal{OC}}(\mathcal{A})$ is satisfied at tolerance $\varepsilon$. Let $\mathcal{S} = \mathcal{S}(\mathcal{J},\boldsymbol{a}_{\mathcal{J}})$ be any conditioning event induced by $\mathcal{A}$. Then, for any two distinct orderings $\pi\neq\pi'$ of its defining constraints,
\begin{equation}\label{eq:pairwise_order_suff}
\dist\!\big(\hat{T}_h(\mathcal{S}; \pi), \hat{T}_h(\mathcal{S}; \pi')\big) < \cay(\pi,\pi')\,\varepsilon\leq(\vert\mathcal{J}\vert - 1)\ssp\varepsilon.
\end{equation}
\end{lemma}

\begin{proof}
Since $\pi\neq\pi'$, the quantity $q\triangleq\cay(\pi,\pi')$ satisfies $q\geq 1$.
Let $\sigma\triangleq\pi'\circ\pi^{-1}$. By definition of the transposition distance, $\sigma$ admits a decomposition into $q$ transpositions, $\sigma = \tau_q \cdots \tau_1$. Set $\pi^{(t)}\triangleq\tau_t\cdots\tau_1\ssp\pi$ for $t\in\{0,\dots,q\}$, so that $\pi^{(0)} = \pi$ and $\pi^{(q)} = \pi'$. Consecutive orderings $\pi^{(t - 1)}$ and $\pi^{(t)}$ differ by exchanging the positions of exactly two constraints with all remaining constraints held in place. By \cref{ass:context_independence}, the discrepancy induced by this swap equals the discrepancy of the corresponding two-attribute check in $\mathcal{C}_{\textnormal{OC}}(\mathcal{A})$. Because every such check is satisfied at tolerance $\varepsilon$, we obtain, for every $t\in\{1,\dots,q\}$,
\begin{equation}
\dist\!\left(\hat{T}_h\big(\mathcal{S}; \pi^{(t-1)}\big), \hat{T}_h\big(\mathcal{S}; \pi^{(t)}\big)\right) < \varepsilon.
\end{equation}
The triangle inequality therefore gives
\begin{align}
\dist\!\big(\hat{T}_h(\mathcal{S}; \pi), \hat{T}_h(\mathcal{S}; \pi')\big) &\leq \sum_{t = 1}^{q} \dist\!\left(\hat{T}_h\big(\mathcal{S}; \pi^{(t-1)}\big), \hat{T}_h\big(\mathcal{S}; \pi^{(t)}\big)\right)\\
&< q\,\varepsilon = \cay(\pi,\pi')\,\varepsilon,
\end{align}
which proves the first claim. The second claim uses the fact that the transposition distance between two orderings of $\vert\mathcal{J}\vert$ constraints is at most $\vert\mathcal{J}\vert - 1$.
\end{proof}


\subsection{Within-tree and cross-tree consistency}\label{subsec:app:uniform_consistency}

The preceding lemma lifts local pairwise order consistency checks to arbitrary conditioning events. Together with \cref{lemma:multi_step_split_cons}, it enables uniform comparisons of aggregate estimates across tree levels. Specifically, the following theorem shows that the discrepancy between the aggregate estimates at any two levels is controlled by three quantities: the number of intervening splits, the prior-refinement discrepancy accumulated across those splits, and the deviation of the path orderings at the two endpoint levels from their canonical orderings.

\begin{theorem}[Uniform consistency across tree levels]\label{thm:uniform_consistency}
Let $\mathcal{A} = \{A_1,\dots,A_d\}$, and suppose that
\begin{enumerate}
\item every $(\mathcal{S},A_k)\in\mathcal{C}_{\textnormal{SC}}(\mathcal{A})$ is $\varepsilon_{\textnormal{SC}}$-split consistent under the canonical ordering of the constraints defining $\mathcal{S}$;
\item every pairwise check in $\mathcal{C}_{\textnormal{OC}}(\mathcal{A})$ is satisfied at tolerance $\varepsilon_{\textnormal{OC}}$; and
\item \cref{ass:context_independence,ass:shared_priors} hold.
\end{enumerate}
Let $\mathcal{T}$ be any binary conditioning tree of depth $L\leq d$, induced by $\mathcal{A}$, rooted at $\mathcal{X}$, and equipped with a constraint ordering $\pi$. Then, for every $0\leq\ell<\ell'\leq L$,
\begin{equation}
\label{eq:uniform_tree_level_consistency}
\dist\left(\hat{T}^{(\ell)}_{h,\textnormal{agg}}(\mathcal{X}; \pi),\hat{T}^{(\ell')}_{h,\textnormal{agg}}(\mathcal{X}; \pi)\right)
< (\ell' - \ell)\ssp\varepsilon_{\textnormal{SC}}
+D_h\sum_{j=\ell}^{\ell'-1}\delta_j(\mathcal{T})
+ \big[\kappa_\ell(\mathcal{T},\pi) + \kappa_{\ell'}(\mathcal{T},\pi)\big]\,\varepsilon_{\textnormal{OC}},
\end{equation}
where $\kappa_r(\mathcal{T},\pi)$ denotes the minimum number of pairwise swaps required to transform the attribute ordering at level $r$ induced by $\pi$ into the canonical order.
\end{theorem}

\begin{proof}
We first bound the effect of replacing the path ordering by the canonical ordering at a fixed level $r$. Since all pairwise checks in $\mathcal{C}_{\textnormal{OC}}(\mathcal{A})$ are satisfied, \cref{ass:context_independence,lemma:pairwise_order_suff} imply that, for every $v\in\mathcal{L}_r(\mathcal{T})$,
\begin{equation}\label{eq:node_order_bound}
\dist\!\big(\hat{T}_h(\mathcal{S}_v; \pi_v), \hat{T}_h(\mathcal{S}_v;\overline{\pi}_v)\big) \leq \cay(\pi_v, \overline{\pi}_v)\,\varepsilon_\textnormal{OC}
\end{equation}
with strict inequality whenever $\pi_v\neq\overline{\pi}_v$. When $\pi_v = \overline{\pi}_v$, both sides of \eqref{eq:node_order_bound} are zero. By \cref{ass:shared_priors}, the same level prior $\hat{\gamma}_r$ is used under $\pi$ and $\overline{\pi}$. Joint convexity of $\dist(\bigcdot, \bigcdot)$, followed by \eqref{eq:node_order_bound}, therefore gives
\begin{align}
\dist\left(\hat{T}^{(r)}_{h,\textnormal{agg}}(\mathcal{X}; \pi),\hat{T}^{(r)}_{h,\textnormal{agg}}(\mathcal{X}; \overline{\pi})\right)&\leq\sum_{v\in\mathcal{L}_r(\mathcal{T})}\hat{\gamma}_r(v)\,
\dist\!\big(\hat{T}_h(\mathcal{S}_v; \pi_v),\hat{T}_h(\mathcal{S}_v; \overline{\pi}_v)\big)\\
&\leq\varepsilon_{\textnormal{OC}}\sum_{v\in\mathcal{L}_r(\mathcal{T})}\hat{\gamma}_r(v)\,\cay(\pi_v, \overline{\pi}_v)\\
&= \kappa_r(\mathcal{T},\pi)\,\varepsilon_{\textnormal{OC}}.\label{eq:proof_main_thm:order_cons_bound}
\end{align}
Next, \cref{lemma:multi_step_split_cons} bounds the discrepancy between the aggregates under canonical ordering:
\begin{equation}\label{eq:proof_main_thm:multi_level_bound}
\dist\left(\hat{T}^{(\ell)}_{h,\textnormal{agg}}(\mathcal{X}; \overline{\pi}), \hat{T}^{(\ell')}_{h,\textnormal{agg}}(\mathcal{X}; \overline{\pi})\right)
< (\ell' - \ell)\ssp\varepsilon_{\textnormal{SC}}
+D_h\sum_{j=\ell}^{\ell'-1}\delta_j(\mathcal{T}).
\end{equation}
Finally, insert these two canonical order aggregates between the original aggregates. By the triangle inequality, and the two bounds \eqref{eq:proof_main_thm:order_cons_bound} and \eqref{eq:proof_main_thm:multi_level_bound}, we have
\begin{align}
\dist\left(\hat{T}^{(\ell)}_{h,\textnormal{agg}}(\mathcal{X}; \pi),\hat{T}^{(\ell')}_{h,\textnormal{agg}}(\mathcal{X}; \pi)\right) &\leq \dist\left(\hat{T}^{(\ell)}_{h,\textnormal{agg}}(\mathcal{X}; \pi), \hat{T}^{(\ell)}_{h,\textnormal{agg}}(\mathcal{X}; \overline{\pi})\right)\\
&+ \dist\left(\hat{T}^{(\ell)}_{h,\textnormal{agg}}(\mathcal{X}; \overline{\pi}), \hat{T}^{(\ell')}_{h,\textnormal{agg}}(\mathcal{X}; \overline{\pi})\right)\\
&+ \dist\left(\hat{T}^{(\ell')}_{h,\textnormal{agg}}(\mathcal{X}; \overline{\pi}), \hat{T}^{(\ell')}_{h,\textnormal{agg}}(\mathcal{X}; \pi)\right)\\
&< (\ell' - \ell)\ssp\varepsilon_{\textnormal{SC}}
+D_h\sum_{j=\ell}^{\ell'-1}\delta_j(\mathcal{T})
+ \big[\kappa_\ell(\mathcal{T},\pi) + \kappa_{\ell'}(\mathcal{T},\pi)\big]\,\varepsilon_{\textnormal{OC}},
\end{align}
which establishes the proof.
\end{proof}

\begin{corollary}[Cross-tree consistency]\label{cor:main_result}
Let $\mathcal{T}$ and $\mathcal{T}'$ be any two binary conditioning trees induced by $\mathcal{A} = \{A_1, \dots, A_d\}$, rooted at the same population $\mathcal{X}$ with depths $L$ and $L'$, respectively. 
Assume $\mathcal{T}$ and $\mathcal{T}'$ are equipped with arbitrary constraint orderings $\pi$ and $\pi'$, respectively.
Under the conditions of \cref{thm:uniform_consistency}, for every $0\leq\ell\leq L$ and $0\leq\ell'\leq L'$ with $(\ell,\ell')\neq(0,0)$, we have
\begin{equation}\label{eq:cor:main_res:statement}
\begin{split}
\dist\left(\hat{T}^{(\ell)}_{h,\textnormal{agg},\mathcal{T}}(\mathcal{X}; \pi), \hat{T}^{(\ell')}_{h,\textnormal{agg},\mathcal{T}'}(\mathcal{X}; \pi')\right) &< (\ell + \ell')\ssp\varepsilon_{\textnormal{SC}} + D_h\big[\Delta_\ell(\mathcal{T})+\Delta_{\ell'}(\mathcal{T}')\big]\\
&+ \big[(\ell - 1)_+ + (\ell' - 1)_+\big]\ssp\varepsilon_{\textnormal{OC}},
\end{split}
\end{equation}
where $(x)_+ = \max\{x,0\}$ and $\Delta_r(\mathcal{T})\triangleq\sum_{j = 0}^{r - 1}\delta_j(\mathcal{T})$ defines the cumulative prior-refinement discrepancy through level $r$, with $\Delta_0(\mathcal{T})\triangleq 0$. 
\end{corollary}

\begin{proof}
The level-zero aggregate is the direct estimate at the common root and is therefore independent of the tree:
\begin{equation}\label{eq:common_root_estimate}
\hat{T}^{(0)}_{h,\textnormal{agg},\mathcal{T}}(\mathcal{X}; \pi) = \hat{T}_h(\mathcal{X}) = \hat{T}^{(0)}_{h,\textnormal{agg},\mathcal{T}'}(\mathcal{X}; \pi').
\end{equation}
For every $1\leq r\leq L$, applying \cref{thm:uniform_consistency} within $\mathcal{T}$ to levels $0$ and $r$ and then invoking \eqref{eq:kappa-linear} gives
\begin{align}
\dist\left(\hat{T}^{(r)}_{h,\textnormal{agg},\mathcal{T}}(\mathcal{X}; \pi), \hat{T}_h(\mathcal{X})\right) &< r\ssp\varepsilon_{\textnormal{SC}} + D_h\ssp\Delta_r(\mathcal{T}) + \kappa_r(\mathcal{T},\pi)\,\varepsilon_{\textnormal{OC}}\\
&\leq r\ssp\varepsilon_{\textnormal{SC}} + D_h\ssp\Delta_r(\mathcal{T}) + (r - 1)_+\,\varepsilon_{\textnormal{OC}},\label{eq:root_bound}
\end{align}
The same bound holds within $\mathcal{T}'$ for every $1\leq r\leq L'$.
Now recall that $(\ell,\ell')\neq(0,0)$. Using the direct root estimate as an intermediate point for the triangle inequality and applying \eqref{eq:root_bound} to each non-zero level, we obtain
\begin{align}
\dist\left(\hat{T}^{(\ell)}_{h,\textnormal{agg},\mathcal{T}}(\mathcal{X}; \pi), \hat{T}^{(\ell')}_{h,\textnormal{agg},\mathcal{T}'}(\mathcal{X}; \pi')\right) &\\
&\hspace*{-24mm}\leq\dist\left(\hat{T}^{(\ell)}_{h,\textnormal{agg},\mathcal{T}}(\mathcal{X}; \pi), \hat{T}_h(\mathcal{X})\right) + \dist\left(\hat{T}_h(\mathcal{X}), \hat{T}^{(\ell')}_{h,\textnormal{agg},\mathcal{T}'}(\mathcal{X}; \pi')\right)\\
&\hspace*{-24mm}< (\ell + \ell')\ssp\varepsilon_{\textnormal{SC}} + D_h\ssp\big[\Delta_\ell(\mathcal{T})+\Delta_{\ell'}(\mathcal{T}')\big] + \big[(\ell - 1)_+ + (\ell' - 1)_+\big]\ssp\varepsilon_{\textnormal{OC}},
\end{align}
where a term corresponding to a zero level vanishes by \eqref{eq:common_root_estimate}. This proves the claimed bound.
\end{proof}


\subsection{Additional proof}\label{sec:app:proofs}

\begin{proof}[Proof of \cref{prop:main_result:simplified}]
Consider any two level aggregates, arising at levels $\ell$ and $\ell'$, respectively, from BCTs induced by $\mathcal{A}$ and rooted at $\mathcal{X}$. If $\ell = \ell' = 0$, they coincide by \eqref{eq:common_root_estimate}. Otherwise, \cref{cor:main_result} gives
\begin{align}
\dist\left(\hat{T}^{(\ell)}_{h,\mathrm{agg},\mathcal{T}}(\mathcal{X}; \pi), \hat{T}^{(\ell')}_{h,\mathrm{agg},\mathcal{T}'}(\mathcal{X}; \pi')\right) &< (\ell + \ell')\ssp\varepsilon_{\textnormal{SC}} + D_h\big[\Delta_\ell(\mathcal{T}) + \Delta_{\ell'}(\mathcal{T}')\big]\\
&+ \big[(\ell - 1)_+ + (\ell' - 1)_+\big]\varepsilon_{\textnormal{OC}}\\
&\leq 2d\,\varepsilon_{\textnormal{SC}} + 2(d-1)D_h\,\varepsilon_{\textnormal{PR}} + 2(d - 1)\ssp\varepsilon_{\textnormal{OC}},
\end{align}
where the second inequality uses $\ell,\ell'\leq d$, $\delta_0=0$, and $\Delta_r(\mathcal{T})\leq\max\{r-1,0\}\ssp\varepsilon_{\textnormal{PR}}$.
\end{proof}


\section{Additional experimental results}\label{app:additional_results}

This appendix reports additional analyses that probe the robustness of our experimental findings across prompting and implementation choices.

\subsection{Partition-then-aggregate prompting}\label{subsec:app:partition_aggregation}

The aggregation experiments in \cref{sec:macro_fallacy} rely on an explicitly specified binary conditioning tree, where subgroup-level estimates from a fixed partition are recombined via the law of total probability. Their success raises a natural question: can the same benefit be recovered within a single query? 
We construct a prompt that instructs the model to first identify relevant subpopulations and then combine them into a population-level estimate. We call this strategy \emph{partition-then-aggregate (PTA) prompting}. Unlike explicit tree-based aggregation, PTA prompting does not require an externally specified partition. If successful, this would provide a lightweight and broadly applicable prompting strategy. The exact prompt template is provided in \cref{subsec:app:sec:prompt_template:mtm_prompt}.

\begin{figure}[t]
\centering
\resizebox{0.98\linewidth}{!}{%
\input{figures/8b_micro_macro_ACS.pgf}
}
\begin{subfigure}[t]{0.46\linewidth}
\vspace*{-4mm}
\caption{
\textbf{Average error gain.}
Average relative error gain of PTA prompting over direct prompting across the five income thresholds. Positive values indicate that PTA prompting reduces absolute error relative to direct prompting.
}
\label{subfig:app:partition_aggregation:aggregate:ACS_thresholded}
\end{subfigure}
\hspace*{0.02\linewidth}
\begin{subfigure}[t]{0.46\linewidth}
\vspace*{-3mm}
\caption{
\textbf{Win matrix.}
Each tile corresponds to one model--threshold pair. Green indicates that PTA prompting yields lower absolute error than direct prompting; red indicates the complement.
}
\label{subfig:app:partition_aggregation:aggregate:ACS_win_rate}
\end{subfigure}
\caption{
\textbf{Benefit of PTA prompting on ACS income estimation.}
We compare direct prompting with PTA prompting, which asks the model to first reason over relevant subpopulations before returning an aggregate prediction, while leaving the subpopulations implicit. Unlike explicit tree-based aggregation, the decomposition is not fixed in advance but chosen implicitly by the model.
For each model and income threshold $\tau\in\{100,20\mathrm{k},40\mathrm{k},60\mathrm{k},80\mathrm{k}\}$ USD, we evaluate the absolute error between the predicted population-level probability $\mathbb{P}(Y > \tau\given X\in\mathcal{X})$ and the ACS survey-based ground truth.
}
\label{fig:macro_fallacy:partition_aggregation}
\end{figure}

To evaluate this idea, we compare PTA prompting against a direct prompting baseline, in which the model is asked to estimate the aggregate target quantity in a single step without any reference to subpopulations. 
On the ACS income task, \cref{subfig:app:partition_aggregation:aggregate:ACS_thresholded} shows that PTA prompting often improves over direct prompting when averaged across the five income thresholds. However, the gains are model-dependent and less systematic than those obtained with explicit tree-based aggregation. 
PTA prompting therefore recovers part of the benefit of decomposition through a lightweight implicit strategy, but it is not uniformly beneficial across models. A more detailed analysis of when model-chosen decompositions help is left for future work.


\subsection{Self-consistency across problem domains}\label{subsec:app:self_consistency:persona}

This subsection reports the additional self-consistency results. We first vary the prediction target within the ACS setting and then move to opinion distributions from the World Values Survey.

\paragraph{ACS prediction tasks.}
First, we evaluate self-consistency on the ACS setting discussed in the main body. The ACS income prediction task corresponds to a widely used testbed for evaluating alignment and fairness of predictive models~\citep{ding2021retiring,cruz2024evaluatinglanguagemodelsrisk}. 
We take $\mathcal{A}$ to be the demographic attributes underlying the tree in \cref{sec:census} and, for every conditioning event induced by $\mathcal{A}$, elicit the model's estimated binned income distribution. We then apply the self-consistency checks from \cref{sec:self_consistency_checks}. The two columns in \cref{table:self_consistency:acs_income} report the two complementary benchmark scores.
The first is the split consistency score $\textnormal{SC}_\varepsilon(\mathcal{A})$. For every admissible pair of a conditioning event and a next split attribute in $\mathcal{C}_{\textnormal{SC}}(\mathcal{A})$, we compare the direct estimate to the prior-weighted aggregate of its two refinements and report the fraction of checks satisfied at tolerance $\varepsilon = 0.02$.
The second is the order consistency score $\textnormal{OC}_\varepsilon(\mathcal{A})$. For every two-attribute conditioning event in $\mathcal{C}_{\textnormal{OC}}(\mathcal{A})$, we elicit direct estimates under both orderings of the attribute constraints and report the fraction of pairs whose predictions agree within tolerance $\varepsilon = 0.02$.
Across frontier models, we observe widespread violations of both criteria. These discrepancies point to an important weakness of model-induced estimates and help explain the varying alignment scores of reconstructed estimates across levels.

We next hold the population structure fixed while changing the prediction target. In the binned ACS commute time task reported in \cref{table:self_consistency:acs_commute}, we use the same conditioning attributes and subgroup priors as in the income experiment, but ask models to estimate a different target distribution. The resulting split and order consistency scores differ substantially from those obtained for income, showing that self-consistency is not solely a property of the model or conditioning tree, but can also vary across prediction tasks.

\begin{table}[t]
\centering
\renewcommand{\arraystretch}{1.4}
\setlength{\tabcolsep}{4pt}
\begin{subtable}[t]{0.48\linewidth}
\centering
\resizebox{1.0\linewidth}{!}{%
\begin{tabular}{m{26mm} >{\centering\arraybackslash}m{20mm} >{\centering\arraybackslash}m{20mm}}
\toprule
Model & Split consistency & Order consistency\\[0.75mm]\hline
\openai~GPT-5.4 & \ccol{EEB0B0}0.33 & \ccol{FFFFFF}1.00 \\
\claude~Sonnet 4.6 & \ccol{E58484}0.00 & \ccol{EEB0B0}0.25 \\
\grok~Grok 4.3 & \ccol{E58484}0.00 & \ccol{F5D0D0}0.50 \\
\qwen~Qwen~3.6 Plus & \ccol{E58484}0.17 & \ccol{F5D0D0}0.50 \\
\specialrule{\heavyrulewidth}{0pt}{0pt}
\end{tabular}
}
\caption{%
\textbf{ACS income prediction.}
The target variable $Y$ is the binned income distribution.
}
\label{table:self_consistency:acs_income}
\end{subtable}
\hspace{0.025\linewidth} 
\begin{subtable}[t]{0.48\linewidth}
\centering
\resizebox{1.0\linewidth}{!}{%
\begin{tabular}{m{26mm} >{\centering\arraybackslash}m{20mm} >{\centering\arraybackslash}m{20mm}}
\toprule
Model & Split consistency & Order consistency\\[0.75mm]\hline
\openai~GPT-5.4 & \ccol{E58484}0.17 & \ccol{F5D0D0}0.50 \\
\claude~Sonnet 4.6 & \ccol{E58484}0.17 & \ccol{F5D0D0}0.50 \\
\grok~Grok 4.3 & \ccol{E58484}0.17 & \ccol{F5D0D0}0.50 \\
\qwen~Qwen~3.6 Plus & \ccol{EEB0B0}0.33 & \ccol{EEB0B0}0.25 \\
\specialrule{\heavyrulewidth}{0pt}{0pt}
\end{tabular}
}
\caption{%
\textbf{ACS commute time prediction.}
The target variable $Y$ is the binned commute time distribution. 
}
\label{table:self_consistency:acs_commute}
\end{subtable}
\caption{%
\textbf{Self-consistency in ACS prediction tasks.}
We evaluate self-consistency for a fixed tree and two different prediction tasks. The attribute set $\mathcal{A}$ consists of age and employment status as used in \cref{fig:acs:income_distribution_tree:combined}. For each task, we report the split consistency score $\mathrm{SC}_\varepsilon(\mathcal{A})$ and the order consistency score $\mathrm{OC}_\varepsilon(\mathcal{A})$ at tolerance $\varepsilon = 0.02$. Higher values indicate stronger self-consistency and are shown in lighter red tones. We observe widespread violations of self-consistency across frontier models.
}
\label{table:self_consistency:acs_combined}
\end{table}

\paragraph{Global opinion modeling.}
The split and order consistency scores cover all conditioning trees induced by a fixed attribute set $\mathcal{A}$ (see \cref{prop:main_result:simplified}). A consistent conditional estimator, however, must satisfy these identities not only across trees, but for every attribute set and every problem domain. To probe this problem axis, we construct binary conditioning trees from the World Values Survey (WVS), extending our evaluation to opinion distributions across countries beyond the United States. Building on the GlobalOpinionQA dataset~\citep{durmus2024measuringrepresentationsubjectiveglobal}, we use demographic attributes as conditioning variables and examine, for each country, how subpopulation-level estimates aggregate to the corresponding country-level estimate. Each survey question thereby defines a distinct prediction problem on which the consistency checks are instantiated separately.

\begin{table}[t]
\centering
\renewcommand{\arraystretch}{1.6}
\setlength{\tabcolsep}{4pt}
\definecolor{deltagrey}{HTML}{F0F0F0}
\resizebox{1.0\textwidth}{!}{%
\begin{tabular}{l || ccccc | ccccc || c || ccccc | ccccc || c}
\toprule
\multicolumn{1}{c}{} 
& \multicolumn{10}{c}{\bfseries Split consistency} 
& \multicolumn{1}{c}{}
& \multicolumn{10}{c}{\bfseries Order consistency}
& \multicolumn{1}{c}{}
\\[-0.4mm]
\cmidrule(l{-2.75pt}r{0pt}){2-11} 
\cmidrule(l{0pt}r{0pt}){12-23}
\noalign{\vskip -0.95mm}
&
& \multicolumn{3}{c}{\makebox[0pt][c]{\canada\;\;\canadaText}}
&
&
& \multicolumn{3}{c}{\makebox[0pt][c]{\indonesia\;\;\indonesiaText}}
&
&
&
& \multicolumn{3}{c}{\makebox[0pt][c]{\canada\;\;\canadaText}}
&
&
& \multicolumn{3}{c}{\makebox[0pt][c]{\indonesia\;\;\indonesiaText}}
&
& \\\hline
\large{Model} 
& $Q_a$ & $Q_b$ & $Q_c$ & $Q_d$ & $Q_e$ 
& $Q_a$ & $Q_b$ & $Q_c$ & $Q_d$ & $Q_e$ 
& Avg.
& $Q_a$ & $Q_b$ & $Q_c$ & $Q_d$ & $Q_e$ 
& $Q_a$ & $Q_b$ & $Q_c$ & $Q_d$ & $Q_e$
& Avg. \\\hline
\openaiSmall~GPT-5.6 Sol & \ccol{FFFFFF}1.00 & \ccol{FFFFFF}1.00 & \ccol{FFFFFF}0.83 & \ccol{FFFFFF}1.00 & \ccol{FFFFFF}0.83 & \ccol{FCF2F2}0.67 & \ccol{FFFFFF}0.83 & \ccol{FCF2F2}0.67 & \ccol{FFFFFF}1.00 & \ccol{F5D0D0}0.50 & \ccol{FFFFFF}0.83 &\ccol{FCF2F2}0.75 & \ccol{FCF2F2}0.75 & \ccol{EEB0B0}0.25 & \ccol{FCF2F2}0.75 & \ccol{FFFFFF}1.00 & \ccol{F5D0D0}0.50 & \ccol{FFFFFF}1.00 & \ccol{FFFFFF}1.00 & \ccol{FFFFFF}1.00 & \ccol{FFFFFF}1.00 & \ccol{FFFFFF}0.80\\
\openaiSmall~GPT-5.4 & \ccol{FFFFFF}1.00 & \ccol{FFFFFF}1.00 & \ccol{EEB0B0}0.33 & \ccol{F5D0D0}0.50 & \ccol{EEB0B0}0.33 & \ccol{FCF2F2}0.67 & \ccol{F5D0D0}0.50 & \ccol{F5D0D0}0.50 & \ccol{EEB0B0}0.33 & \ccol{EEB0B0}0.33 & \ccol{F5D0D0}0.55 &\ccol{FFFFFF}1.00 & \ccol{FCF2F2}0.75 & \ccol{EEB0B0}0.25 & \ccol{FCF2F2}0.75 & \ccol{F5D0D0}0.50 & \ccol{FFFFFF}1.00 & \ccol{FFFFFF}1.00 & \ccol{FCF2F2}0.75 & \ccol{EEB0B0}0.25 & \ccol{FCF2F2}0.75 & \ccol{FCF2F2}0.70\\
\openaiSmall~GPT-5.4 nano & \ccol{EEB0B0}0.33 & \ccol{E58484}0.17 & \ccol{E58484}0.17 & \ccol{E58484}0.00 & \ccol{EEB0B0}0.33 & \ccol{E58484}0.17 & \ccol{E58484}0.17 & \ccol{E58484}0.17 & \ccol{E58484}0.17 & \ccol{EEB0B0}0.33 & \ccol{EEB0B0}0.20 &\ccol{EEB0B0}0.25 & \ccol{FCF2F2}0.75 & \ccol{E58484}0.00 & \ccol{E58484}0.00 & \ccol{F5D0D0}0.50 & \ccol{F5D0D0}0.50 & \ccol{EEB0B0}0.25 & \ccol{F5D0D0}0.50 & \ccol{EEB0B0}0.25 & \ccol{EEB0B0}0.25 & \ccol{EEB0B0}0.32\\
\openaiSmall~GPT-5.1 & \ccol{FCF2F2}0.67 & \ccol{FCF2F2}0.67 & \ccol{EEB0B0}0.33 & \ccol{EEB0B0}0.33 & \ccol{F5D0D0}0.50 & \ccol{FFFFFF}0.83 & \ccol{F5D0D0}0.50 & \ccol{EEB0B0}0.33 & \ccol{EEB0B0}0.33 & \ccol{F5D0D0}0.50 & \ccol{F5D0D0}0.50 &\ccol{F5D0D0}0.50 & \ccol{FFFFFF}1.00 & \ccol{FCF2F2}0.75 & \ccol{FFFFFF}1.00 & \ccol{FCF2F2}0.75 & \ccol{EEB0B0}0.25 & \ccol{E58484}0.00 & \ccol{FCF2F2}0.75 & \ccol{FFFFFF}1.00 & \ccol{F5D0D0}0.50 & \ccol{FCF2F2}0.65\\
\openaiSmall~GPT-4o mini & \ccol{FFFFFF}0.83 & \ccol{E58484}0.00 & \ccol{F5D0D0}0.50 & \ccol{E58484}0.00 & \ccol{EEB0B0}0.33 & \ccol{FFFFFF}0.83 & \ccol{EEB0B0}0.33 & \ccol{F5D0D0}0.50 & \ccol{F5D0D0}0.50 & \ccol{EEB0B0}0.33 & \ccol{F5D0D0}0.42 &\ccol{EEB0B0}0.25 & \ccol{F5D0D0}0.50 & \ccol{F5D0D0}0.50 & \ccol{EEB0B0}0.25 & \ccol{FFFFFF}1.00 & \ccol{F5D0D0}0.50 & \ccol{FFFFFF}1.00 & \ccol{EEB0B0}0.25 & \ccol{EEB0B0}0.25 & \ccol{EEB0B0}0.25 & \ccol{F5D0D0}0.47\\
\claudeSmall~Sonnet 4.6 & \ccol{FCF2F2}0.67 & \ccol{FFFFFF}1.00 & \ccol{F5D0D0}0.50 & \ccol{FFFFFF}1.00 & \ccol{FCF2F2}0.67 & \ccol{F5D0D0}0.50 & \ccol{FFFFFF}0.83 & \ccol{FFFFFF}1.00 & \ccol{F5D0D0}0.50 & \ccol{EEB0B0}0.33 & \ccol{FCF2F2}0.70 &\ccol{FFFFFF}1.00 & \ccol{FCF2F2}0.75 & \ccol{FFFFFF}1.00 & \ccol{FFFFFF}1.00 & \ccol{FFFFFF}1.00 & \ccol{F5D0D0}0.50 & \ccol{FFFFFF}1.00 & \ccol{FCF2F2}0.75 & \ccol{FFFFFF}1.00 & \ccol{FFFFFF}1.00 & \ccol{FFFFFF}0.90\\
\deepseekSmall~DeepSeek-V3.2 & \ccol{F5D0D0}0.50 & \ccol{E58484}0.17 & \ccol{E58484}0.17 & \ccol{E58484}0.17 & \ccol{E58484}0.17 & \ccol{E58484}0.17 & \ccol{FFFFFF}0.83 & \ccol{F5D0D0}0.50 & \ccol{E58484}0.17 & \ccol{EEB0B0}0.33 & \ccol{EEB0B0}0.32 &\ccol{FCF2F2}0.75 & \ccol{F5D0D0}0.50 & \ccol{EEB0B0}0.25 & \ccol{E58484}0.00 & \ccol{EEB0B0}0.25 & \ccol{F5D0D0}0.50 & \ccol{F5D0D0}0.50 & \ccol{EEB0B0}0.25 & \ccol{E58484}0.00 & \ccol{F5D0D0}0.50 & \ccol{EEB0B0}0.35\\
\qwenSmall~Qwen~3.6 Plus & \ccol{EEB0B0}0.33 & \ccol{FFFFFF}1.00 & \ccol{EEB0B0}0.33 & \ccol{F5D0D0}0.50 & \ccol{E58484}0.17 & \ccol{EEB0B0}0.33 & \ccol{FCF2F2}0.67 & \ccol{FCF2F2}0.67 & \ccol{FFFFFF}0.83 & \ccol{E58484}0.17 & \ccol{F5D0D0}0.50 &\ccol{FCF2F2}0.75 & \ccol{FCF2F2}0.75 & \ccol{FCF2F2}0.75 & \ccol{F5D0D0}0.50 & \ccol{F5D0D0}0.50 & \ccol{FCF2F2}0.75 & \ccol{F5D0D0}0.50 & \ccol{EEB0B0}0.25 & \ccol{F5D0D0}0.50 & \ccol{F5D0D0}0.50 & \ccol{F5D0D0}0.57\\
\qwenSmall~Qwen~3.5 27B & \ccol{FCF2F2}0.67 & \ccol{E58484}0.17 & \ccol{F5D0D0}0.50 & \ccol{F5D0D0}0.50 & \ccol{E58484}0.17 & \ccol{E58484}0.00 & \ccol{EEB0B0}0.33 & \ccol{F5D0D0}0.50 & \ccol{EEB0B0}0.33 & \ccol{E58484}0.00 & \ccol{EEB0B0}0.32 &\ccol{EEB0B0}0.25 & \ccol{FCF2F2}0.75 & \ccol{EEB0B0}0.25 & \ccol{F5D0D0}0.50 & \ccol{E58484}0.00 & \ccol{F5D0D0}0.50 & \ccol{EEB0B0}0.25 & \ccol{EEB0B0}0.25 & \ccol{F5D0D0}0.50 & \ccol{FCF2F2}0.75 & \ccol{F5D0D0}0.40\\
\qwenSmall~Qwen~3.5 35B & \ccol{F5D0D0}0.50 & \ccol{F5D0D0}0.50 & \ccol{EEB0B0}0.33 & \ccol{F5D0D0}0.50 & \ccol{E58484}0.00 & \ccol{E58484}0.17 & \ccol{FCF2F2}0.67 & \ccol{EEB0B0}0.33 & \ccol{FCF2F2}0.67 & \ccol{E58484}0.17 & \ccol{EEB0B0}0.38 &\ccol{EEB0B0}0.25 & \ccol{FCF2F2}0.75 & \ccol{EEB0B0}0.25 & \ccol{F5D0D0}0.50 & \ccol{F5D0D0}0.50 & \ccol{F5D0D0}0.50 & \ccol{EEB0B0}0.25 & \ccol{EEB0B0}0.25 & \ccol{EEB0B0}0.25 & \ccol{F5D0D0}0.50 & \ccol{F5D0D0}0.40\\
\qwenSmall~Qwen~3.5 122B & \ccol{F5D0D0}0.50 & \ccol{EEB0B0}0.33 & \ccol{F5D0D0}0.50 & \ccol{EEB0B0}0.33 & \ccol{E58484}0.17 & \ccol{EEB0B0}0.33 & \ccol{FCF2F2}0.67 & \ccol{F5D0D0}0.50 & \ccol{EEB0B0}0.33 & \ccol{EEB0B0}0.33 & \ccol{F5D0D0}0.40 &\ccol{E58484}0.00 & \ccol{FFFFFF}1.00 & \ccol{F5D0D0}0.50 & \ccol{FCF2F2}0.75 & \ccol{F5D0D0}0.50 & \ccol{F5D0D0}0.50 & \ccol{F5D0D0}0.50 & \ccol{F5D0D0}0.50 & \ccol{FCF2F2}0.75 & \ccol{F5D0D0}0.50 & \ccol{F5D0D0}0.55\\
\grokSmall~Grok~4.5 & \ccol{FCF2F2}0.67 & \ccol{FFFFFF}1.00 & \ccol{F5D0D0}0.50 & \ccol{FCF2F2}0.67 & \ccol{FCF2F2}0.67 & \ccol{E58484}0.17 & \ccol{FFFFFF}1.00 & \ccol{FFFFFF}0.83 & \ccol{F5D0D0}0.50 & \ccol{F5D0D0}0.50 & \ccol{FCF2F2}0.65 &\ccol{FFFFFF}1.00 & \ccol{FFFFFF}1.00 & \ccol{FFFFFF}1.00 & \ccol{FCF2F2}0.75 & \ccol{F5D0D0}0.50 & \ccol{FCF2F2}0.75 & \ccol{FCF2F2}0.75 & \ccol{FFFFFF}1.00 & \ccol{FFFFFF}1.00 & \ccol{FFFFFF}1.00 & \ccol{FFFFFF}0.88\\
\grokSmall~Grok~4.3 & \ccol{FCF2F2}0.67 & \ccol{FFFFFF}0.83 & \ccol{E58484}0.17 & \ccol{EEB0B0}0.33 & \ccol{E58484}0.17 & \ccol{FFFFFF}0.83 & \ccol{FFFFFF}0.83 & \ccol{E58484}0.17 & \ccol{EEB0B0}0.33 & \ccol{F5D0D0}0.50 & \ccol{F5D0D0}0.48 &\ccol{EEB0B0}0.25 & \ccol{E58484}0.00 & \ccol{FFFFFF}1.00 & \ccol{F5D0D0}0.50 & \ccol{F5D0D0}0.50 & \ccol{F5D0D0}0.50 & \ccol{FFFFFF}1.00 & \ccol{FFFFFF}1.00 & \ccol{EEB0B0}0.25 & \ccol{F5D0D0}0.50 & \ccol{F5D0D0}0.55\\
\grokSmall~Grok~4.1 Fast & \ccol{F5D0D0}0.50 & \ccol{F5D0D0}0.50 & \ccol{E58484}0.00 & \ccol{FFFFFF}0.83 & \ccol{E58484}0.17 & \ccol{FCF2F2}0.67 & \ccol{FFFFFF}1.00 & \ccol{FCF2F2}0.67 & \ccol{E58484}0.17 & \ccol{EEB0B0}0.33 & \ccol{F5D0D0}0.48 &\ccol{F5D0D0}0.50 & \ccol{FCF2F2}0.75 & \ccol{FCF2F2}0.75 & \ccol{FFFFFF}1.00 & \ccol{F5D0D0}0.50 & \ccol{EEB0B0}0.25 & \ccol{FCF2F2}0.75 & \ccol{FFFFFF}1.00 & \ccol{FFFFFF}1.00 & \ccol{EEB0B0}0.25 & \ccol{FCF2F2}0.68\\
\geminiSmall~Gemini~3 Flash & \ccol{EEB0B0}0.33 & \ccol{FFFFFF}1.00 & \ccol{EEB0B0}0.33 & \ccol{FFFFFF}0.83 & \ccol{F5D0D0}0.50 & \ccol{E58484}0.17 & \ccol{FFFFFF}0.83 & \ccol{FFFFFF}0.83 & \ccol{EEB0B0}0.33 & \ccol{FFFFFF}1.00 & \ccol{FCF2F2}0.62 &\ccol{FFFFFF}1.00 & \ccol{FCF2F2}0.75 & \ccol{F5D0D0}0.50 & \ccol{FFFFFF}1.00 & \ccol{FCF2F2}0.75 & \ccol{FFFFFF}1.00 & \ccol{FCF2F2}0.75 & \ccol{EEB0B0}0.25 & \ccol{FCF2F2}0.75 & \ccol{FFFFFF}1.00 & \ccol{FCF2F2}0.78\\
\specialrule{\heavyrulewidth}{0pt}{0pt}
\end{tabular}
}
\vspace*{2.5mm}
\caption{%
\textbf{Self-consistency in global opinion prediction.}
For five survey questions sampled from the WVS $(Q_a, \dots, Q_e)$ and two countries, we report the split consistency score $\mathrm{SC}_\varepsilon(\mathcal{A})$ and the order consistency score $\mathrm{OC}_\varepsilon(\mathcal{A})$ at tolerance $\varepsilon = 0.02$, where $\mathcal{A}$ consists of age (younger than 45 years old vs.\ at least 45 years old) and income splits (low vs. high). 
Each survey question defines a distinct prediction problem. 
Higher values indicate stronger self-consistency and are shown in lighter red tones. Across models, countries, and questions, current models fail to satisfy the self-consistency checks reliably.
}
\label{table:wvs_eval:results}
\end{table}

We run our evaluation on data from Canada and Indonesia, the two countries with the largest sample sizes. For each country, we construct a BCT of depth two using age and income splits. We then elicit categorical response distributions at every node for five selected survey questions. The full question wording and answer scales are reported in \cref{sec:wvs_questions}.
Consistency is assessed using the reference-free checks from \cref{sec:self_consistency_checks}. 
For each question, this includes the split and order consistency checks. Results for tolerance $\varepsilon = 0.02$ are reported in \cref{table:wvs_eval:results}.

The WVS results show that the self-consistency failures observed on ACS are not specific to income prediction or to the {U.S.} census setting. Across both countries, models satisfy only a fraction of the split and order consistency checks. No model is uniformly consistent across questions. Even the strongest models exhibit scores that vary substantially across survey questions and countries. Thus, the WVS provides complementary evidence that current models do not elicit faithful probability estimates.


\subsection{Self-consistency in decision models}\label{subsub:appendix:decision_models}

We extend the evaluation in \cref{fig:self_consistency:model_comp} to decision models, which are trained specifically for decision tasks and calibrated probability estimation. We retain the thresholded ACS income prediction task with $\tau = 40\mathrm{k}$ USD, demographic attribute set $\mathcal{A}$, and consistency checks. Decision models require a predefined structured output format, which we provide in \cref{subsec:app:prompt_template:decision,subsec:app:prompt_template:decision_prior} together with the prompt and context. We evaluate split and order consistency at tolerance $\varepsilon = 0.02$. \Cref{fig:self_consistency:decision_models} shows that, despite their explicit focus on reliable probability estimation, these models violate basic probabilistic consistency requirements. Structured outputs and calibration-oriented training thus leave a fundamental reliability gap.

\begin{figure}[t]
\centering
\input{figures/12c_consistency_with_decision_models.pgf}
\vspace*{-1mm}
\caption{
\textbf{Decision models violate probabilistic self-consistency.}
We report the split consistency score $\mathrm{SC}_{\varepsilon}(\mathcal{A})$ and order consistency score $\mathrm{OC}_{\varepsilon}(\mathcal{A})$ at tolerance $\varepsilon = 0.02$ for the thresholded ACS income prediction task with $\tau = 40\mathrm{k}$ USD, following the experimental setting of \cref{fig:self_consistency:model_comp}.
Despite training specifically for decision tasks and calibrated probability estimation, these models fail basic consistency checks across their own predictions.
}
\vspace*{-4mm}
\label{fig:self_consistency:decision_models}
\end{figure}


\subsection{Level-wise self-consistency}\label{subsub:appendix:level_wise_self_consistency}

\begin{table}[t]
\centering
\renewcommand{\arraystretch}{1.6}
\setlength{\tabcolsep}{4pt}
\begin{tabular}{p{28mm} cccc}
\toprule
\large{Model} & $\ell=1$ & $\ell=2$ & $\ell=3$ & $\ell=4$\\[0.75mm]\hline
\openai~GPT-5.4 & \ccol{C3E3C9}0.084 & \ccol{FFFFFF}0.145 & \ccol{C3E3C9}0.070 & \ccol{8FC79E}0.050 \\
\claude~Sonnet~4.6 & \ccol{8FC79E}0.042 & \ccol{ECF7EE}0.108 & \ccol{8FC79E}0.064 & \ccol{8FC79E}0.053 \\
\grok~Grok~4.3 & \ccol{8FC79E}0.060 & \ccol{ECF7EE}0.116 & \ccol{C3E3C9}0.081 & \ccol{8FC79E}0.063 \\
\qwen~Qwen~3.6 Plus & \ccol{8FC79E}0.045 & \ccol{FFFFFF}0.143 & \ccol{C3E3C9}0.079 & \ccol{ECF7EE}0.121 \\
\specialrule{\heavyrulewidth}{0pt}{0pt}
\end{tabular}
\vspace*{2.5mm}
\caption{
\textbf{Level-wise root consistency.}
For each level $\ell$ of a depth-four tree, we reconstruct the aggregate income distribution from the corresponding level-$\ell$ partition and compare it to the model's direct root estimate.
Entries report the normalized Wasserstein-$1$ distance. Lower values indicate stronger agreement with the root estimate and are highlighted using darker green tones.
}
\label{table:self_consistency:acs_levelwise}
\end{table}

To assess how self-consistency varies with the granularity of the partition, \cref{table:self_consistency:acs_levelwise} reports a level-wise root consistency evaluation. For this experiment, we use the first four levels of the depth-six binary conditioning tree described in \cref{subsec:app:exp_details:alignment_error}. For each level $\ell$, we reconstruct the aggregate income distribution from the partition induced at that level and compare it to the model’s direct root estimate using the normalized Wasserstein distance $\overline{\mathcal{W}}_1$. The table reports these discrepancies directly, so lower values indicate stronger agreement with the root estimate.


\subsection{Prompt sensitivity}\label{subsub:appendix:persona_vs_socio}

\begin{figure}[t]
\centering
\resizebox{1.0\linewidth}{!}{%
\input{figures/5_prompt_sensitivity_arXiv.pgf}
}
\begin{subfigure}[t]{0.42\linewidth}
\vspace*{-2mm}
\caption{\textbf{Sociodemographic prompting.} The target subpopulation is described in the third person using structured attributes.}
\label{subfig:example_reliability_diag:prompting_scheme:sociodemographic}
\end{subfigure}
\hspace*{0.04\linewidth}
\begin{subfigure}[t]{0.42\linewidth}
\vspace*{-4mm}
\caption{\textbf{Persona prompting.} The model is instructed to role-play as a member of the target subpopulation in the first person.}
\label{subfig:example_reliability_diag:prompting_scheme:persona}
\end{subfigure}
\caption{\textbf{Prompt sensitivity of probability elicitation.}
Results are shown for GPT-5.4 without reasoning, using ground-truth priors and 100 repeated LLM samples. For multiple income thresholds, model estimates are plotted against the corresponding survey-derived probabilities. Points closer to the diagonal indicate stronger ground-truth alignment. 
Direct estimates are obtained from population-level prompts, whereas reconstructed estimates are obtained by aggregating subgroup-level estimates via the law of total probability using ground-truth priors. The population is decomposed according to employment status; see the \texttt{ESR} attribute in \cref{sec:acs_attributes}. 
Error bars correspond to 90\,\% confidence intervals computed by bootstrapping with 1000 samples. Error bars are barely visible as the intervals are very narrow.
}
\label{fig:example_reliability_diag:prompting_scheme}
\end{figure}

We examine prompt sensitivity along two complementary dimensions: ground-truth alignment and reference-free self-consistency. In the default \emph{sociodemographic} formulation, the target subpopulation is described in the third person through a structured list of attributes; see the prompt template in \cref{subsec:app:sec:prompt_template:socio}. In the \emph{persona prompting} formulation, the same information is instead presented as a first-person role-playing instruction; see \cref{subsec:app:sec:prompt_template:persona}. Both formulations elicit the same verbalized probability estimate for the same event and differ only in how the conditioning information is expressed.

We first assess ground-truth alignment in the thresholded ACS income task. Across the two prompt formulations, we keep the model, sampling temperature, number of repeated generations, outcome event, income thresholds, and aggregation weights fixed. The resulting comparisons with the survey-derived probabilities are shown in \cref{fig:example_reliability_diag:prompting_scheme}. 
Under both formulations, estimates reconstructed through the law of total probability lie closer to the diagonal than direct population-level estimates for most thresholds. The benefit of subgroup decomposition is therefore not specific to a particular formulation of the conditioning context. However, persona prompting systematically overestimates the ground-truth probabilities, whereas sociodemographic prompting produces more closely aligned estimates. We therefore use sociodemographic prompting as the default throughout our experiments.

We next examine the complementary, reference-free question of whether the elicited estimates are statistically self-consistent. For a subset of models, we repeat the WVS evaluation from \cref{table:wvs_eval:results}, changing only the prompting method. Specifically, we replace the default third-person sociodemographic template with the first-person persona template. \Cref{table:wvs_eval:ablation_persona} reports the fraction of self-consistency checks satisfied by each model, revealing that the effect of persona prompting varies across models. It lowers self-consistency for some models while improving it for others. Averaged over the six models evaluated under both prompt formulations, the split consistency rate decreases from $0.475$ under sociodemographic prompting to $0.403$ under persona prompting. The order consistency rate increases slightly from $0.587$ to $0.592$. Thus, although prompt formulation affects individual consistency rates, substantial violations persist under both interfaces, leaving our central qualitative finding unchanged.

\begin{table}[t]
\centering
\renewcommand{\arraystretch}{1.6}
\setlength{\tabcolsep}{4pt}
\definecolor{deltagrey}{HTML}{F0F0F0}
\resizebox{1.0\textwidth}{!}{%
\begin{tabular}{l || ccccc | ccccc || c || ccccc | ccccc || c}
\toprule
\multicolumn{1}{c}{} 
& \multicolumn{10}{c}{\bfseries Split consistency} 
& \multicolumn{1}{c}{}
& \multicolumn{10}{c}{\bfseries Order consistency}
& \multicolumn{1}{c}{}
\\[-0.4mm]
\cmidrule(l{-2.75pt}r{0pt}){2-11} 
\cmidrule(l{0pt}r{0pt}){12-23}
\noalign{\vskip -0.95mm}
&
& \multicolumn{3}{c}{\makebox[0pt][c]{\canada\;\;\canadaText}}
&
&
& \multicolumn{3}{c}{\makebox[0pt][c]{\indonesia\;\;\indonesiaText}}
&
&
&
& \multicolumn{3}{c}{\makebox[0pt][c]{\canada\;\;\canadaText}}
&
&
& \multicolumn{3}{c}{\makebox[0pt][c]{\indonesia\;\;\indonesiaText}}
&
& \\\hline
\large{Model} 
& $Q_a$ & $Q_b$ & $Q_c$ & $Q_d$ & $Q_e$ 
& $Q_a$ & $Q_b$ & $Q_c$ & $Q_d$ & $Q_e$ 
& Avg.
& $Q_a$ & $Q_b$ & $Q_c$ & $Q_d$ & $Q_e$ 
& $Q_a$ & $Q_b$ & $Q_c$ & $Q_d$ & $Q_e$
& Avg. \\\hline
\openaiSmall~GPT-5.4 & \ccol{E58484}0.17 & \ccol{EEB0B0}0.33 & \ccol{E58484}0.17 & \ccol{F5D0D0}0.50 & \ccol{EEB0B0}0.33 & \ccol{E58484}0.17 & \ccol{E58484}0.17 & \ccol{FCF2F2}0.67 & \ccol{E58484}0.00 & \ccol{E58484}0.17 & \ccol{EEB0B0}0.27 &\ccol{FCF2F2}0.75 & \ccol{FFFFFF}1.00 & \ccol{FCF2F2}0.75 & \ccol{FFFFFF}1.00 & \ccol{FCF2F2}0.75 & \ccol{E58484}0.00 & \ccol{FCF2F2}0.75 & \ccol{FCF2F2}0.75 & \ccol{FFFFFF}1.00 & \ccol{FCF2F2}0.75 & \ccol{FCF2F2}0.75\\
\claudeSmall~Sonnet 4.6 & \ccol{FFFFFF}0.83 & \ccol{FCF2F2}0.67 & \ccol{EEB0B0}0.33 & \ccol{FCF2F2}0.67 & \ccol{EEB0B0}0.33 & \ccol{E58484}0.17 & \ccol{F5D0D0}0.50 & \ccol{F5D0D0}0.50 & \ccol{E58484}0.17 & \ccol{F5D0D0}0.50 & \ccol{F5D0D0}0.47 &\ccol{FFFFFF}1.00 & \ccol{FFFFFF}1.00 & \ccol{FFFFFF}1.00 & \ccol{FFFFFF}1.00 & \ccol{F5D0D0}0.50 & \ccol{FCF2F2}0.75 & \ccol{FCF2F2}0.75 & \ccol{FCF2F2}0.75 & \ccol{FFFFFF}1.00 & \ccol{F5D0D0}0.50 & \ccol{FFFFFF}0.82\\
\qwenSmall~Qwen~3.6 Plus & \ccol{F5D0D0}0.50 & \ccol{FFFFFF}0.83 & \ccol{E58484}0.17 & \ccol{FFFFFF}1.00 & \ccol{EEB0B0}0.33 & \ccol{E58484}0.00 & \ccol{FFFFFF}0.83 & \ccol{EEB0B0}0.33 & \ccol{FCF2F2}0.67 & \ccol{EEB0B0}0.33 & \ccol{F5D0D0}0.50 &\ccol{FCF2F2}0.75 & \ccol{F5D0D0}0.50 & \ccol{FCF2F2}0.75 & \ccol{E58484}0.00 & \ccol{EEB0B0}0.25 & \ccol{FCF2F2}0.75 & \ccol{FCF2F2}0.75 & \ccol{EEB0B0}0.25 & \ccol{F5D0D0}0.50 & \ccol{F5D0D0}0.50 & \ccol{F5D0D0}0.50\\
\qwenSmall~Qwen~3.5 27B & \ccol{FCF2F2}0.67 & \ccol{FFFFFF}1.00 & \ccol{F5D0D0}0.50 & \ccol{FFFFFF}1.00 & \ccol{E58484}0.17 & \ccol{E58484}0.00 & \ccol{EEB0B0}0.33 & \ccol{F5D0D0}0.50 & \ccol{EEB0B0}0.33 & \ccol{EEB0B0}0.33 & \ccol{F5D0D0}0.48 &\ccol{EEB0B0}0.25 & \ccol{F5D0D0}0.50 & \ccol{F5D0D0}0.50 & \ccol{E58484}0.00 & \ccol{EEB0B0}0.25 & \ccol{F5D0D0}0.50 & \ccol{F5D0D0}0.50 & \ccol{F5D0D0}0.50 & \ccol{F5D0D0}0.50 & \ccol{FCF2F2}0.75 & \ccol{F5D0D0}0.43\\
\qwenSmall~Qwen~3.5 35B & \ccol{F5D0D0}0.50 & \ccol{F5D0D0}0.50 & \ccol{EEB0B0}0.33 & \ccol{F5D0D0}0.50 & \ccol{E58484}0.00 & \ccol{EEB0B0}0.33 & \ccol{FCF2F2}0.67 & \ccol{FCF2F2}0.67 & \ccol{EEB0B0}0.33 & \ccol{E58484}0.17 & \ccol{F5D0D0}0.40 &\ccol{EEB0B0}0.25 & \ccol{F5D0D0}0.50 & \ccol{EEB0B0}0.25 & \ccol{FCF2F2}0.75 & \ccol{EEB0B0}0.25 & \ccol{F5D0D0}0.50 & \ccol{FCF2F2}0.75 & \ccol{F5D0D0}0.50 & \ccol{FCF2F2}0.75 & \ccol{F5D0D0}0.50 & \ccol{F5D0D0}0.50\\
\qwenSmall~Qwen~3.5 122B & \ccol{EEB0B0}0.33 & \ccol{EEB0B0}0.33 & \ccol{EEB0B0}0.33 & \ccol{F5D0D0}0.50 & \ccol{E58484}0.00 & \ccol{E58484}0.17 & \ccol{E58484}0.00 & \ccol{EEB0B0}0.33 & \ccol{FFFFFF}0.83 & \ccol{E58484}0.17 & \ccol{EEB0B0}0.30 &\ccol{EEB0B0}0.25 & \ccol{F5D0D0}0.50 & \ccol{FFFFFF}1.00 & \ccol{F5D0D0}0.50 & \ccol{FCF2F2}0.75 & \ccol{F5D0D0}0.50 & \ccol{F5D0D0}0.50 & \ccol{EEB0B0}0.25 & \ccol{FCF2F2}0.75 & \ccol{F5D0D0}0.50 & \ccol{F5D0D0}0.55\\
\specialrule{\heavyrulewidth}{0pt}{0pt}
\end{tabular}
}
\vspace{2.5mm}
\caption{%
\textbf{Robustness to prompt formulation.}
Self-consistency results under first-person persona prompting, using the same experimental setting, notation, and evaluation criteria as \cref{table:wvs_eval:results}. Higher values indicate stronger self-consistency. The effects vary across models, and substantial violations persist.
}
\label{table:wvs_eval:ablation_persona}
\end{table}


\subsection{Sensitivity to split verbalization}\label{subsub:appendix:verb_split_crit}

We examine whether the aggregation gains reported in \cref{sec:macro_fallacy} depend on how the conditioning-attribute splits are expressed in natural language. To isolate this effect, we repeat the thresholded ACS income experiment while varying only the textual descriptions of the first three splits. The underlying partitions remain unchanged. We compare the default verbalization, which explicitly enumerates the attribute values assigned to each branch, with a simplified verbalization that expresses the same sets using compact intervals or semantically meaningful category descriptions.
Specifically, at level one, the \texttt{AGEP} branches are described using age ranges. For example, the explicit enumeration ``31 years old or 32 years old or \ldots'' is replaced by ``31 to 68 years old.'' At level two, the \texttt{COW} values $\{1,\ldots,7\}$ are summarized as ``currently employed or self-employed,'' whereas values $\{8,9,\texttt{n/a}\}$ are summarized as ``not currently in paid employment.'' These groupings preserve the ACS category mapping reported in \cref{tab:acs_attributes}. Finally, at level three, the \texttt{WKHP} split is expressed using compact intervals, such as ``1 to 39 hours per week,'' instead of explicitly listing every constituent value.

\begin{figure}[t]
\centering
\resizebox{1.0\linewidth}{!}{%
\input{figures/14_aggregation_gain_impr_split.pgf}
}
\begin{subfigure}[t]{0.48\linewidth}
\vspace*{-2mm}
\caption{\textbf{Default verbalization.} The attribute values defining each branch are enumerated explicitly.}
\label{subfig:aggregation_refinement:split_verbalization:default}
\end{subfigure}
\hfill
\begin{subfigure}[t]{0.48\linewidth}
\vspace*{-4mm}
\caption{\textbf{Simplified verbalization.} The same partitions are described using compact intervals and natural-language category summaries.}
\label{subfig:aggregation_refinement:split_verbalization:simplified}
\end{subfigure}
\caption{%
\textbf{Sensitivity of aggregation gains to split verbalization.}
We report the aggregation gain across the first three levels of the thresholded ACS income tree under the default and simplified verbalization schemes. All experimental choices other than the subgroup descriptions are held fixed. Positive values indicate that estimates reconstructed from subgroup predictions are better aligned with the survey-derived reference than the direct population-level estimate at level~0. Bars correspond to different income thresholds, and the orange curve shows their average. Error bars denote 90\,\% bootstrap confidence intervals. Aggregation gains persist under both schemes and are more uniformly positive under the simplified descriptions.
}
\label{fig:aggregation_refinement:split_verbalization}
\end{figure}

As shown in \cref{fig:aggregation_refinement:split_verbalization}, changing the split descriptions affects the magnitude of the resulting aggregation gains. Nevertheless, the central qualitative pattern remains unchanged: estimates reconstructed from finer-grained subgroup predictions generally improve upon direct population-level prompting under both verbalization schemes. These results show that the macro fallacy is not an artifact of the split verbalization method.


\subsection{Robustness to reasoning effort}\label{subsub:appendix:reasoning_sc}

We assess robustness to reasoning effort by repeating the self-consistency evaluation from \cref{table:wvs_eval:results} under the same experimental setup, but instructing the models to use a ``low'' reasoning budget. We compare these results with the corresponding runs without reasoning while keeping all other evaluation choices fixed. Because explicit reasoning controls are not available for every model, the analysis is restricted to models for which both settings can be evaluated. \Cref{table:wvs_eval:ablation_reasoning} reports the fraction of self-consistency checks satisfied by each model.

Enabling reasoning affects individual model performance, but the effects vary considerably across models. Averaged across the evaluated models, the split consistency rate increases only modestly, from $0.498$ without reasoning to $0.512$ with a low reasoning budget. The order consistency rate exhibits a similarly small increase, from $0.634$ to $0.652$. Thus, additional reasoning does not reliably eliminate the observed statistical consistency violations. Our central qualitative finding is therefore robust to enabling low-effort reasoning.

\begin{table}[t]
\centering
\renewcommand{\arraystretch}{1.6}
\setlength{\tabcolsep}{4pt}
\definecolor{deltagrey}{HTML}{F0F0F0}
\resizebox{1.0\textwidth}{!}{%
\begin{tabular}{l || ccccc | ccccc || c || ccccc | ccccc || c}
\toprule
\multicolumn{1}{c}{} 
& \multicolumn{10}{c}{\bfseries Split consistency} 
& \multicolumn{1}{c}{}
& \multicolumn{10}{c}{\bfseries Order consistency}
& \multicolumn{1}{c}{}
\\[-0.4mm]
\cmidrule(l{-2.75pt}r{0pt}){2-11} 
\cmidrule(l{0pt}r{0pt}){12-23}
\noalign{\vskip -0.95mm}
&
& \multicolumn{3}{c}{\makebox[0pt][c]{\canada\;\;\canadaText}}
&
&
& \multicolumn{3}{c}{\makebox[0pt][c]{\indonesia\;\;\indonesiaText}}
&
&
&
& \multicolumn{3}{c}{\makebox[0pt][c]{\canada\;\;\canadaText}}
&
&
& \multicolumn{3}{c}{\makebox[0pt][c]{\indonesia\;\;\indonesiaText}}
&
& \\\hline
\large{Model} 
& $Q_a$ & $Q_b$ & $Q_c$ & $Q_d$ & $Q_e$ 
& $Q_a$ & $Q_b$ & $Q_c$ & $Q_d$ & $Q_e$ 
& Avg.
& $Q_a$ & $Q_b$ & $Q_c$ & $Q_d$ & $Q_e$ 
& $Q_a$ & $Q_b$ & $Q_c$ & $Q_d$ & $Q_e$
& Avg. \\\hline
\openaiSmall~GPT-5.4 & \ccol{F5D0D0}0.50 & \ccol{FFFFFF}1.00 & \ccol{F5D0D0}0.50 & \ccol{FFFFFF}1.00 & \ccol{EEB0B0}0.33 & \ccol{F5D0D0}0.50 & \ccol{FFFFFF}0.83 & \ccol{EEB0B0}0.33 & \ccol{F5D0D0}0.50 & \ccol{EEB0B0}0.33 & \ccol{F5D0D0}0.58 &\ccol{FFFFFF}1.00 & \ccol{EEB0B0}0.25 & \ccol{FFFFFF}1.00 & \ccol{FCF2F2}0.75 & \ccol{FCF2F2}0.75 & \ccol{FCF2F2}0.75 & \ccol{F5D0D0}0.50 & \ccol{FCF2F2}0.75 & \ccol{F5D0D0}0.50 & \ccol{FFFFFF}1.00 & \ccol{FCF2F2}0.72\\
\openaiSmall~GPT-5.4 nano & \ccol{EEB0B0}0.33 & \ccol{F5D0D0}0.50 & \ccol{E58484}0.00 & \ccol{FCF2F2}0.67 & \ccol{EEB0B0}0.33 & \ccol{E58484}0.00 & \ccol{E58484}0.17 & \ccol{FFFFFF}0.83 & \ccol{E58484}0.17 & \ccol{EEB0B0}0.33 & \ccol{EEB0B0}0.33 &\ccol{FCF2F2}0.75 & \ccol{F5D0D0}0.50 & \ccol{F5D0D0}0.50 & \ccol{FCF2F2}0.75 & \ccol{E58484}0.00 & \ccol{E58484}0.00 & \ccol{F5D0D0}0.50 & \ccol{F5D0D0}0.50 & \ccol{F5D0D0}0.50 & \ccol{F5D0D0}0.50 & \ccol{F5D0D0}0.45\\
\openaiSmall~GPT-4o mini & \ccol{EEB0B0}0.33 & \ccol{E58484}0.17 & \ccol{F5D0D0}0.50 & \ccol{EEB0B0}0.33 & \ccol{F5D0D0}0.50 & \ccol{FFFFFF}0.83 & \ccol{F5D0D0}0.50 & \ccol{E58484}0.17 & \ccol{FFFFFF}0.83 & \ccol{EEB0B0}0.33 & \ccol{F5D0D0}0.45 &\ccol{FFFFFF}1.00 & \ccol{FCF2F2}0.75 & \ccol{FFFFFF}1.00 & \ccol{F5D0D0}0.50 & \ccol{EEB0B0}0.25 & \ccol{F5D0D0}0.50 & \ccol{FCF2F2}0.75 & \ccol{FCF2F2}0.75 & \ccol{F5D0D0}0.50 & \ccol{F5D0D0}0.50 & \ccol{FCF2F2}0.65\\
\claudeSmall~Sonnet 4.6 & \ccol{EEB0B0}0.33 & \ccol{FFFFFF}1.00 & \ccol{EEB0B0}0.33 & \ccol{EEB0B0}0.33 & \ccol{F5D0D0}0.50 & \ccol{EEB0B0}0.33 & \ccol{FFFFFF}1.00 & \ccol{FFFFFF}0.83 & \ccol{EEB0B0}0.33 & \ccol{EEB0B0}0.33 & \ccol{F5D0D0}0.53 &\ccol{F5D0D0}0.50 & \ccol{FCF2F2}0.75 & \ccol{FCF2F2}0.75 & \ccol{FFFFFF}1.00 & \ccol{F5D0D0}0.50 & \ccol{FFFFFF}1.00 & \ccol{FFFFFF}1.00 & \ccol{FFFFFF}1.00 & \ccol{EEB0B0}0.25 & \ccol{F5D0D0}0.50 & \ccol{FCF2F2}0.72\\
\geminiSmall~Gemini~3 Flash & \ccol{EEB0B0}0.33 & \ccol{FFFFFF}1.00 & \ccol{FCF2F2}0.67 & \ccol{FFFFFF}0.83 & \ccol{EEB0B0}0.33 & \ccol{E58484}0.17 & \ccol{FFFFFF}0.83 & \ccol{FFFFFF}1.00 & \ccol{FFFFFF}0.83 & \ccol{FCF2F2}0.67 & \ccol{FCF2F2}0.67 &\ccol{FFFFFF}1.00 & \ccol{FFFFFF}1.00 & \ccol{F5D0D0}0.50 & \ccol{EEB0B0}0.25 & \ccol{FCF2F2}0.75 & \ccol{FCF2F2}0.75 & \ccol{F5D0D0}0.50 & \ccol{FCF2F2}0.75 & \ccol{FCF2F2}0.75 & \ccol{FFFFFF}1.00 & \ccol{FCF2F2}0.72\\
\specialrule{\heavyrulewidth}{0pt}{0pt}
\end{tabular}
}
\vspace{2mm}
\caption{%
\textbf{Robustness to reasoning effort.}
Self-consistency results obtained with a low reasoning budget, using the same experimental setting, notation, and evaluation criteria as \cref{table:wvs_eval:results}. Higher values indicate stronger self-consistency. Additional reasoning does not reliably eliminate self-consistency violations.
}
\label{table:wvs_eval:ablation_reasoning}
\end{table}


\section{Experiment details}\label{sec:app:experiment_details}

In this appendix, we describe the experimental setup\footnote{The code is available at: \url{https://github.com/patrikwolf/statistical_self_consistency}} used for the results reported throughout this paper.


\subsection{Binary conditioning tree}\label{subsec:app:experiment_details:bct}

\begin{figure}[t]
\centering
\begin{tikzpicture}[
  font=\fontsize{6}{7}\selectfont,
  level 1/.style={sibling distance=62mm, level distance=27mm},
  level 2/.style={sibling distance=31mm, level distance=27mm},
  edge from parent/.style={
    draw=black!55,
    line width=0.9pt,
    edge from parent path={
      (\tikzparentnode.south) -- ++(0,-3mm) -| (\tikzchildnode.north)
    }
  },
  edgelabel/.style={
    fill=white,
    inner sep=2pt,
    align=center,
  },
  bottom left edge/.style={
    edge from parent path={
      (\tikzparentnode.south) -- ++(0,-3.5mm)
      -| ([xshift=-3mm]\tikzchildnode.north)
    }
  },
  bottom right edge/.style={
    edge from parent path={
      (\tikzparentnode.south) -- ++(0,-3.5mm)
      -| ([xshift=3mm]\tikzchildnode.north)
    }
  },
  treetitle/.style={
    draw=none,
    fill=none,
    align=center,
  },
  treenode/.style={
    draw=black!65,
    rounded corners=3pt,
    line width=0.7pt,
    fill=white,
    align=center,
    inner xsep=3pt,
    inner ysep=3pt,
    text width=27mm,
  },
]
\begin{scope}

\node[treenode] {   
\centering
\begingroup
\tikzset{
  every path/.style={},
  every node/.style={},
  every picture/.style={},
}%
\pgfsetcornersarced{\pgfpointorigin}%
\resizebox{\linewidth}{!}{%
\input{figures/tree_green_no_aggr/shared_binned_hist_0.pgf}
}
\endgroup
}
   child {
      node[treenode] {         
\centering
\begingroup
\tikzset{
  every path/.style={},
  every node/.style={},
  every picture/.style={},
}%
\pgfsetcornersarced{\pgfpointorigin}%
\resizebox{\linewidth}{!}{%
\input{figures/tree_green_no_aggr/shared_binned_hist_1.pgf}
}
\endgroup
      }
         child {
            node[treenode] {               
\centering
\begingroup
\tikzset{
  every path/.style={},
  every node/.style={},
  every picture/.style={},
}%
\pgfsetcornersarced{\pgfpointorigin}%
\resizebox{\linewidth}{!}{%
\input{figures/tree_green_no_aggr/shared_binned_hist_3.pgf}
}
\endgroup
            }
            edge from parent[bottom left edge] node[edgelabel, pos=0.248] {\textbf{unemployed}}
         }
         child {
            node[treenode] {               
\centering
\begingroup
\tikzset{
  every path/.style={},
  every node/.style={},
  every picture/.style={},
}%
\pgfsetcornersarced{\pgfpointorigin}%
\resizebox{\linewidth}{!}{%
\input{figures/tree_green_no_aggr/shared_binned_hist_4.pgf}
}
\endgroup
            }
            edge from parent[bottom right edge] node[edgelabel, pos=0.248] {\textbf{employed}}
         }
      edge from parent node[edgelabel, pos=0.3] {\textbf{young \& old}}
   }
   child {
      node[treenode] {         
\centering
\begingroup
\tikzset{
  every path/.style={},
  every node/.style={},
  every picture/.style={},
}%
\pgfsetcornersarced{\pgfpointorigin}%
\resizebox{\linewidth}{!}{%
\input{figures/tree_green_no_aggr/shared_binned_hist_2.pgf}
}
\endgroup
      }
         child {
            node[treenode] {               
\centering
\begingroup
\tikzset{
  every path/.style={},
  every node/.style={},
  every picture/.style={},
}%
\pgfsetcornersarced{\pgfpointorigin}%
\resizebox{\linewidth}{!}{%
\input{figures/tree_green_no_aggr/shared_binned_hist_5.pgf}
}
\endgroup
            }
            edge from parent[bottom left edge] node[edgelabel, pos=0.248] {\textbf{unemployed}}
         }
         child {
            node[treenode] {               
\centering
\begingroup
\tikzset{
  every path/.style={},
  every node/.style={},
  every picture/.style={},
}%
\pgfsetcornersarced{\pgfpointorigin}%
\resizebox{\linewidth}{!}{%
\input{figures/tree_green_no_aggr/shared_binned_hist_6.pgf}
}
\endgroup
            }
            edge from parent[bottom right edge] node[edgelabel, pos=0.248] {\textbf{employed}}
         }
      edge from parent node[edgelabel, pos=0.3] {\textbf{working age}}
   }
;

\path (current bounding box.south east) coordinate (Lse);
\path (current bounding box.north east) coordinate (Lne);

\end{scope}


\definecolor{gtbin0}{RGB}{46,107,142}
\definecolor{gtbin1}{RGB}{30,152,138}
\definecolor{gtbin2}{RGB}{75,194,108}
\definecolor{gtbin3}{RGB}{189,222,38}

\node[anchor=north] at (-4.34, 1.24) {
\centering
\begingroup
\tikzset{
  every path/.style={},
  every node/.style={},
  every picture/.style={},
}%
\pgfsetcornersarced{\pgfpointorigin}%
\resizebox{38mm}{!}{%
\input{figures/tree_green_no_aggr/vertical_legend.pgf}
}
\endgroup
};


\end{tikzpicture}
\caption{
\textbf{Ground-truth and LLM-elicited income distributions across partitions of the population.}
We construct a binary conditioning tree from the 2024 ACS population by recursively slicing the population with sociodemographic attributes. The first split criterion corresponds to age (\texttt{AGEP}) and the second to class of worker (\texttt{COW}).
Each node reports the distribution of yearly income over four bins.
For each bin, the solid bar shows the survey-based ground-truth distribution, while the hatched bar shows the LLM-estimated conditional distribution.
}
\vspace*{-1mm}
\label{fig:acs:income_distribution_tree:combined}
\end{figure}

\begin{figure}[t]
\centering
\resizebox{1.0\linewidth}{!}{%
\input{figures/6_sorted_mean_income.pgf}
}
\begin{subfigure}[t]{0.48\linewidth}
\vspace*{-4mm}
\caption{Age attribute \texttt{AGEP}.}
\label{fig:bct:mean_incomes:age}
\end{subfigure}
\hfill
\begin{subfigure}[t]{0.48\linewidth}
\vspace*{-4mm}
\caption{Class-of-worker attribute \texttt{COW}.}
\label{fig:bct:mean_incomes:cow}
\end{subfigure}
\caption{
\textbf{Mean income across attribute-induced subgroups.}
We partition the ACS population by all attainable values of a single attribute and report the survey-weighted mean yearly income in USD for each resulting subgroup.
For the age attribute in \cref{fig:bct:mean_incomes:age}, subgroup values are displayed in their natural order, and the binary split is defined by thresholding the corresponding subgroup mean incomes (dashed line).
For the categorical class-of-worker attribute in \cref{fig:bct:mean_incomes:cow}, subgroup values are first sorted by mean income before defining the split.
Both rules yield splits that separate lower-income from higher-income attribute values while keeping the two branches approximately balanced in survey-weighted probability mass.
}
\label{fig:bct:mean_incomes}
\end{figure}

For our running example in \cref{sec:census}, we construct the binary conditioning tree shown in \cref{fig:acs:income_distribution_tree:combined} using demographic attributes from the 2024 American Community Survey (ACS) 1-Year Public Use Microdata Sample \citep{uscensus_acs_pums_2024}. The root node corresponds to the {U.S.} population, which we recursively refine using two binary splits. To select these splits, we consider the following candidate attributes: \texttt{AGEP} (age), \texttt{CIT} (citizenship), \texttt{COW} (class of worker), \texttt{ESR} (employment status), \texttt{MAR} (marital status), \texttt{OCCP} (occupation), \texttt{SCHL} (educational attainment), \texttt{SEX} (sex), and \texttt{WKHP} (usual hours worked per week).
For each candidate attribute, we first partition the population by all attainable attribute values and compute the mean yearly income within each resulting subgroup, as exemplified in \cref{fig:bct:mean_incomes}. We then group the attribute values into a lower-income and a higher-income branch, choosing the grouping so that the two branches have approximately equal survey-weighted probability mass. Finally, we qualitatively select the two attributes whose induced binary splits yield particularly distinctive income distributions across the two branches. 
This procedure yields the two splits used in our main ACS experiments. The first split is induced by the ACS age attribute \texttt{AGEP}. Respondents younger than $31$ or older than $68$ are grouped into the ``young \& old'' subgroup, while respondents aged $31$ to $68$ are grouped into the ``working age'' subgroup. The second split uses the class-of-worker attribute \texttt{COW}. Values $1$--$7$ are labeled as ``employed,'' whereas values $8$, $9$, and missing values are mapped to ``unemployed.'' The corresponding ACS attribute descriptions are reported in \cref{tab:acs_attributes}. 

At each node in \cref{fig:acs:income_distribution_tree:combined}, we report the survey-weighted distribution of yearly income in USD. For the four-bin income setting, the bin edges are chosen so that the root distribution is approximately uniform, and the same binning is then used for all subpopulations in the tree. Since yearly income is unavailable for respondents below age $15$, we set their income to $0$ USD in our experiments, reflecting the assumption that individuals in this age group have no reported labor income. 


\subsection{Prior elicitation}\label{subsec:app:prior_elicitation}

To obtain fully model-induced aggregate quantities, both the subgroup-conditional estimates and the subgroup priors must be elicited from the LLM. In this subsection, we describe the level-wise procedure used to estimate these priors. For each tree level $\ell$ with induced partition $\mathcal{P}_\ell$, we prompt the model $n = 50$ times to estimate the probability mass that a randomly sampled individual from the base population belongs to each subgroup $\mathcal{S}\in\mathcal{P}_\ell$. The prompts use the template in \cref{subsec:app:sec:prompt_template:priors}, which lists the subgroups in the partition and asks the model to return a normalized probability distribution over them. Although the prompt instructs the model to output probabilities that sum to one, the raw responses may contain small normalization errors. We therefore post-process each response before averaging. The LLM-estimated prior at level $\ell$ is then obtained by averaging the renormalized estimates across repetitions. 


\subsection{Alignment error}\label{subsec:app:exp_details:alignment_error}

The alignment analysis in \cref{fig:aggregation_refinement:llm_prior} is based on a binary conditioning tree over ACS attributes. The tree is defined by the following sequence of splits:

\begin{itemize}[leftmargin=15.6mm]
\item[\bfseries Level 1:] \texttt{AGEP} is split into \texttt{\{0, \dots, 30, 69, \dots, 96, n/a\}} and \texttt{\{31, \dots, 68\}}.
\item[\bfseries Level 2:] \texttt{COW} is split into \texttt{\{1, \dots, 7\}} and \texttt{\{8, 9, n/a\}}.
\item[\bfseries Level 3:] \texttt{WKHP} is split into \texttt{\{1, \dots, 39, n/a\}} and \texttt{\{40, \dots, 99\}}.
\item[\bfseries Level 4:] \texttt{OCCP} is split into \texttt{\{2060, 2300, 2740, 3401, 3421, 3422, 3424, 3603, 3605, 3620, 3630, 3640, 3645, 3647, 3648, 3649, 3940, 3960, 2350, 2440, 2545, 2633, 2723, 3430, 3601, 3602, 3646, 3946, 4010, 4020, 4030, 4040, 4055, 4110, 4120, 4130, 4140, 4150, 4160, 4220, 4230, 4251, 4255, 4350, 4420, 4435, 4461, 4500, 4510, 4521, 4522, 4525, 4530, 4540, 4600, 4621, 4622, 4640, 4655, 4720, 4740, 4760, 4900, 4940, 5020, 5160, 5260, 5300, 5310, 5320, 5400, 5510, 5610, 5810, 5820, 5850, 5860, 5900, 6040, 6050, 6120, 6410, 6600, 7260, 7510, 7610, 7800, 7810, 7830, 7840, 7850, 8256, 8300, 8310, 8320, 8335, 8350, 8365, 8465, 8530, 8540, 8710, 8800, 8850, 8910, 8950, 9110, 9142, 9350, 9365, 9415, 9600, 9610, 9620, 9630, 9640, 9645, 9720, 9920, n/a\}} and \texttt{\{10, 20, 40, 51, 52, 60, 101, 102, 110, 120, 135, 136, 137, 140, 150, 160, 205, 220, 230, 300, 310, 335, 340, 350, 360, 410, 420, 425, 440, 500, 510, 520, 530, 540, 565, 600, 630, 640, 650, 700, 705, 710, 725, 726, 735, 750, 800, 810, 820, 830, 845, 850, 860, 900, 910, 930, 940, 960, 1005, 1006, 1007, 1010, 1021, 1022, 1031, 1032, 1050, 1065, 1105, 1106, 1108, 1200, 1220, 1240, 1305, 1306, 1310, 1320, 1340, 1350, 1360, 1400, 1410, 1420, 1430, 1440, 1450, 1460, 1520, 1530, 1541, 1545, 1551, 1555, 1560, 1600, 1610, 1640, 1650, 1700, 1710, 1720, 1745, 1750, 1760, 1800, 1821, 1822, 1825, 1840, 1860, 1900, 1910, 1920, 1935, 1970, 1980, 2001, 2002, 2003, 2004, 2005, 2006, 2011, 2012, 2013, 2014, 2015, 2016, 2025, 2040, 2050, 2100, 2105, 2145, 2170, 2180, 2205, 2310, 2320, 2330, 2360, 2400, 2435, 2555, 2600, 2631, 2632, 2634, 2635, 2636, 2640, 2700, 2710, 2721, 2722, 2751, 2752, 2755, 2770, 2805, 2810, 2825, 2830, 2840, 2850, 2861, 2862, 2865, 2905, 2910, 2920, 3000, 3010, 3030, 3040, 3050, 3090, 3100, 3110, 3120, 3140, 3150, 3160, 3200, 3210, 3220, 3230, 3245, 3250, 3255, 3256, 3258, 3261, 3270, 3300, 3310, 3321, 3322, 3323, 3324, 3330, 3402, 3423, 3500, 3515, 3520, 3545, 3550, 3610, 3655, 3700, 3710, 3720, 3725, 3740, 3750, 3801, 3802, 3820, 3840, 3870, 3900, 3910, 3930, 3945, 4000, 4200, 4210, 4240, 4252, 4330, 4340, 4400, 4465, 4700, 4710, 4750, 4800, 4810, 4820, 4830, 4840, 4850, 4920, 4930, 4950, 4965, 5000, 5010, 5040, 5100, 5110, 5120, 5140, 5150, 5165, 5220, 5230, 5240, 5250, 5330, 5340, 5350, 5360, 5410, 5420, 5500, 5521, 5522, 5530, 5540, 5550, 5560, 5600, 5630, 5710, 5720, 5730, 5740, 5840, 5910, 5920, 5940, 6005, 6010, 6115, 6130, 6200, 6210, 6220, 6230, 6240, 6250, 6260, 6305, 6330, 6355, 6360, 6400, 6441, 6442, 6460, 6515, 6520, 6530, 6540, 6660, 6700, 6710, 6720, 6730, 6740, 6765, 6800, 6825, 6835, 6850, 6950, 7000, 7010, 7020, 7030, 7040, 7100, 7120, 7130, 7140, 7150, 7160, 7200, 7210, 7220, 7240, 7300, 7315, 7320, 7330, 7340, 7350, 7360, 7410, 7420, 7430, 7540, 7560, 7640, 7700, 7720, 7730, 7740, 7750, 7855, 7905, 7925, 7950, 8000, 8025, 8030, 8040, 8100, 8130, 8140, 8225, 8250, 8255, 8450, 8500, 8510, 8555, 8600, 8610, 8620, 8630, 8640, 8650, 8720, 8730, 8740, 8750, 8760, 8810, 8830, 8920, 8930, 8940, 8990, 9005, 9030, 9040, 9050, 9121, 9122, 9130, 9141, 9150, 9210, 9240, 9265, 9300, 9310, 9410, 9430, 9510, 9570, 9650, 9760, 9800, 9810, 9825, 9830\}}.
\item[\bfseries Level 5:] \texttt{ESR} is split into \texttt{\{1, 2, 4, 5\}} and \texttt{\{3, 6, n/a\}}.
\item[\bfseries Level 6:] \texttt{MAR} is split into \texttt{\{1, 3\}} and \texttt{\{2, 4, 5, n/a\}}.
\end{itemize}

The ACS category mapping is reported in \cref{tab:acs_attributes}. 
For each level $\ell\in\{0,1,\dots,6\}$, we compute the alignment error $\mathrm{AErr}(\ell)$ of the aggregate estimate obtained from the nodes at that level. 
To compare errors across levels, we normalize by the root-level error and report $1 - \text{AErr}(\ell)/\text{AErr}(0)$.
Here, $\ell = 0$ corresponds to direct prompting, while larger values of $\ell$ correspond to estimates reconstructed from increasingly fine-grained subgroup predictions. Error bars indicate 90\,\% confidence intervals computed using 1000 bootstrap samples. In \cref{subfig:aggregation_refinement:llm}, we compare a selection of models at a fixed income threshold of $\tau = 40\mathrm{k}$ USD. In \cref{subfig:aggregation_relative_models:llm}, we show the aggregation gain for GPT-5.4 across income thresholds. Across both settings, the results show that aggregate estimates reconstructed from subgroup estimates improve over direct root-level estimates up to moderate tree depths.


\subsection{Disentangling the sources of aggregation error}

The two panels in \cref{subfig:disentangling:conditionals,subfig:disentangling:priors} are based on the depth-six binary conditioning tree described in \cref{subsec:app:exp_details:alignment_error}. For visual clarity, we display only the first three levels in the conditional-estimates panel and the first four levels in the subgroup-priors panel.

\paragraph{Node-wise conditional alignment.}
For \cref{subfig:disentangling:conditionals}, we evaluate the quality of the LLM-estimated conditional income distributions at each node of the tree. The outcome is the four-bin income distribution used in \cref{fig:acs:income_distribution_tree:combined}. For every node $\mathcal{S}$, we compare the LLM-estimated conditional distribution $\hat{T}_h(\mathcal{S})$ to the corresponding survey-based conditional distribution $T_h(\mathcal{S})$ using the normalized Wasserstein distance, $\overline{\mathcal{W}}_1\bigs(\hat{T}_h(\mathcal{S}), T_h(\mathcal{S})\bigs)$.
The node color encodes this error. The resulting plot illustrates that conditioning on more specific subgroups improves conditional alignment relative to the root estimate, but that the improvement is heterogeneous across the tree. Some fine-grained subgroups with negligible prior mass remain poorly aligned.

\paragraph{Prior alignment.}
For \cref{subfig:disentangling:priors}, we evaluate the quality of the subgroup priors. For each tree level $\ell$, let $\gamma_\ell$ denote the survey-based distribution over the partition $\mathcal{P}_\ell(\mathcal{X})$, and let $\hat{\gamma}_\ell$ denote the corresponding LLM-estimated prior distribution obtained with the elicitation procedure from \cref{subsec:app:prior_elicitation}. We measure prior alignment using total variation distance, $\operatorname{TV}(\hat{\gamma}_\ell,\gamma_\ell)$.
This yields one prior-alignment error per model and tree level. The results show that prior estimation becomes increasingly difficult as the tree is refined. Coarse partitions are estimated comparatively accurately, whereas deeper levels require the model to distribute probability mass across a larger number of more specific subpopulations. Thus, aggregation involves a trade-off. More specific conditioning can improve the subgroup outcome estimates, but the corresponding population weights become harder to estimate accurately.

\subsection{Variance decomposition across tree levels}

The variance decomposition in \cref{subfig:explaining_macro_fallacy:variance} is computed on the first four levels of the binary conditioning tree described in \cref{subsec:app:exp_details:alignment_error}. 
For each node, we compute the survey-weighted subgroup prior $\gamma_\mathcal{S}$, the conditional mean $\mu_\mathcal{S}\triangleq\mathbb{E}[Y\given X \in\mathcal{S}]$, and the conditional variance $\mathbb{V}(Y\given X \in\mathcal{S})$ from the ACS data. We then decompose the aggregate income variance into a within-subgroup and a cross-subgroup component. The within-subgroup term is the prior-weighted average conditional variance,
\begin{equation}
V_{\mathrm{within}}(\ell)\triangleq\sum_{\mathcal{S}\in\mathcal{P}_\ell(\mathcal{X})} \gamma_\mathcal{S}\ssp\mathbb{V}(Y\given X \in\mathcal{S}),
\end{equation}
which measures the residual heterogeneity that remains after conditioning on the attributes defining the level-$\ell$ subgroups. The cross-subgroup term is the prior-weighted variance of the subgroup means,
\begin{equation}
V_{\mathrm{cross}}(\ell)\triangleq\sum_{\mathcal{S}\in\mathcal{P}_\ell(\mathcal{X})} \gamma_\mathcal{S}\,(\mu_\mathcal{S} - \mu)^2\qquad\text{with}\qquad\mu\triangleq\sum_{\mathcal{S}\in\mathcal{P}_\ell(\mathcal{X})} \gamma_\mathcal{S}\,\mu_\mathcal{S}.
\end{equation}
By the law of total variance, these quantities satisfy $\mathbb{V}(Y) = V_{\mathrm{within}}(\ell) + V_{\mathrm{cross}}(\ell)$ for every level $\ell$, up to finite-sample estimation error.
Thus, the two terms quantify how refinement shifts variation from residual within-subgroup heterogeneity to explicit cross-subgroup heterogeneity.


\subsection{Self-consistency evaluation on the binned ACS income setting}
\label{subsec:app:experiment_details:acs}

For the self-consistency analysis in \cref{table:self_consistency:acs_income}, we use the binned ACS income prediction task from the running example in \cref{fig:acs:income_distribution_tree:combined}. The model is asked to estimate the income distribution over the following four bins:

\begin{itemize}[leftmargin=20mm]
\item[\bfseries Bin 1:] $[-11.5\mathrm{k}, 1)$ USD
\item[\bfseries Bin 2:] $[1, 25\mathrm{k})$ USD
\item[\bfseries Bin 3:] $[25\mathrm{k}, 60\mathrm{k})$ USD
\item[\bfseries Bin 4:] $[60\mathrm{k}, 1849\mathrm{k}]$ USD
\end{itemize}

These bins cover the full income range observed in the ACS sample: no individual reports an income below $-11.5\mathrm{k}$ USD or above $1849\mathrm{k}$ USD.
The results in \cref{table:self_consistency:acs_income} report the two benchmark scores introduced in \cref{sec:self_consistency_checks}. The first is the split consistency score $\mathrm{SC}_\varepsilon(\mathcal{A})$ from \cref{def:split_consistency_score}. For every admissible pair of a conditioning event and a next split attribute in $\mathcal{C}_{\mathrm{SC}}(\mathcal{A})$, we compare the model's direct estimate at the conditioning event to the prior-weighted aggregate of its estimates at the two refinements induced by the split attribute. For the attribute set $\mathcal{A}$ consisting of age and employment status, this yields six node-level checks.
The second is the order consistency score $\mathrm{OC}_\varepsilon(\mathcal{A})$ from \cref{def:order_consistency_score}. For each of the four two-attribute conditioning events in $\mathcal{C}_{\mathrm{OC}}(\mathcal{A})$, we elicit direct estimates under both orderings of the attribute constraints and check whether the two predictions agree. 
Since the target is an ordered categorical distribution, discrepancies are measured using the normalized Wasserstein distance $\overline{\mathcal{W}}_1$. Each score reports the fraction of checks whose discrepancy is less than the tolerance $\varepsilon = 0.02$.


\subsection{Self-consistency evaluation on the thresholded ACS income setting}
\label{subsec:app:experiment_details:acs_thresholded}

For the model comparison in \cref{fig:self_consistency:model_comp}, we use the thresholded ACS income prediction task with threshold $\tau = 40\mathrm{k}$ USD. The setting mirrors the binned income experiment in \cref{subsec:app:experiment_details:acs}. We use the same conditioning tree over the attribute set $\mathcal{A}$ consisting of age and employment status. The only change is the target variable, which is now the binary indicator of whether an individual's income exceeds $\tau$.

For each model, we compute the split consistency score $\mathrm{SC}_\varepsilon(\mathcal{A})$ and the order consistency score $\mathrm{OC}_\varepsilon(\mathcal{A})$ at tolerance $\varepsilon = 0.02$. Since the target is binary, the normalized Wasserstein distance $\overline{\mathcal{W}}_1$ reduces to the absolute difference between the two estimated probabilities of exceeding the threshold.
As a proxy for general model capability, we report the Artificial Analysis Intelligence (AAI) Index~\citep{artificialanalysis2026}. Models are grouped by family and ordered within each family by their AAI Index.


\subsection{Self-consistency evaluation on the ACS commute time setting}
\label{subsec:app:experiment_details:acs_commute}

For the self-consistency analysis in \cref{table:self_consistency:acs_commute}, we use the binned ACS commute time prediction task. The setting mirrors the binned income experiment in \cref{subsec:app:experiment_details:acs}. We use the same conditioning tree over the attribute set $\mathcal{A}$ consisting of age and employment status, and the same subgroup priors. The only change is the target variable, which is now the commute time in minutes. The model is asked to estimate the commute time distribution over the following four bins:

\begin{itemize}[leftmargin=20mm]
\item[\bfseries Bin 1:] $[0, 1)$ minute
\item[\bfseries Bin 2:] $[1, 15)$ minutes
\item[\bfseries Bin 3:] $[15, 30)$ minutes
\item[\bfseries Bin 4:] $[30, 195]$ minutes
\end{itemize}

These bins cover the full range of commute times observed in the ACS sample. No individual reports a commute time longer than $195$ minutes.
We compute the split consistency score $\mathrm{SC}_\varepsilon(\mathcal{A})$ and the order consistency score $\mathrm{OC}_\varepsilon(\mathcal{A})$ at tolerance $\varepsilon = 0.02$, exactly as in \cref{subsec:app:experiment_details:acs}. Since the target is an ordered categorical distribution, discrepancies are measured using the normalized Wasserstein distance $\overline{\mathcal{W}}_1$.


\subsection{Self-consistency evaluation on WVS opinion distributions}\label{subsec:app:experiment_details:wvs} 

We evaluate self-consistency on opinion distributions derived from the World Values Survey (WVS). The experiment reported in \cref{table:wvs_eval:results} uses data from Canada and Indonesia. For both countries, the evaluation is based on the attribute set $\mathcal{A}$ consisting of the following two splits:

\begin{itemize}[leftmargin=20mm]
\item[\bfseries Level 1:] Age is split into ``16\,--\,44 years old'' and ``45\,--\,103 years old''.
\item[\bfseries Level 2:] Income is split into ``low'' and ``high''.
\end{itemize}

The interpretation of ``low'' and ``high'' in the income split is left to the model. This is sufficient for our self-consistency checks, provided that the two options are interpreted as complementary categories that form a valid partition of the relevant population.
For each country and each of the five selected WVS questions listed in \cref{tab:wvs_questions}, we elicit categorical response distributions at every conditioning event induced by $\mathcal{A}$. The corresponding prompt template is given in \cref{subsec:app:sec:prompt_template:wvs_distribution}. 
Each question then defines a distinct prediction problem, on which we compute the split consistency score $\mathrm{SC}_\varepsilon(\mathcal{A})$ and the order consistency score $\mathrm{OC}_\varepsilon(\mathcal{A})$ at tolerance $\varepsilon = 0.02$. 
For the split consistency score, we evaluate all admissible pairs of a conditioning event and a next split attribute in $\mathcal{C}_{\mathrm{SC}}(\mathcal{A})$, comparing the direct estimate at the conditioning event to the prior-weighted aggregate of its two refinements. 
For the order consistency score, we elicit direct estimates under both orderings of the attribute constraints for each of the four two-attribute conditioning events in $\mathcal{C}_{\mathrm{OC}}(\mathcal{A})$ and check whether the two predictions agree. 
Discrepancies are measured using the normalized Wasserstein distance $\overline{\mathcal{W}}_1$, and each score reports the fraction of checks whose discrepancy is less than $\varepsilon = 0.02$.


\subsection{Self-consistency evaluation of the tennis forecasting example}\label{subsec:app:experiment_details:tennis}

Consider a forecasting task in which a model is asked to estimate the probability that Roger Federer wins a match against Rafael Nadal, conditional on contextual information available before or during the match. All estimates are elicited using the tennis forecasting prompt template in \cref{subsec:app:sec:prompt_template:tennis}. At the root node, the prompt provides only the matchup itself, corresponding to a blind or unconditional forecast without additional information. We then gradually reveal pieces of contextual information that may become available before or during the match, such as the court surface. These conditions induce a binary conditioning tree analogous to the BCTs used in the income prediction setting, but the conditioning variables now describe a sporting event rather than a demographic subpopulation.
Concretely, the attribute set $\mathcal{A}$ consists of two binary conditions: the court surface, distinguishing clay from non-clay courts, and the in-match event of whether Federer wins the first set. The former captures a pre-match contextual factor that is highly relevant to the Federer--Nadal rivalry, while the latter represents partial information that becomes available during the match. The same self-consistency checks from \cref{sec:self_consistency_checks} can then be applied, with the model's estimated prior probabilities serving as aggregation weights.

To generate the results in \cref{table:self_consistency:tennis}, we compute the split consistency score $\mathrm{SC}_\varepsilon(\mathcal{A})$ and the order consistency score $\mathrm{OC}_\varepsilon(\mathcal{A})$ at tolerance $\varepsilon = 0.02$. Since the target outcome is binary, discrepancies are measured by the absolute difference between the corresponding win probabilities. Each score reports the fraction of checks whose discrepancy is less than $\varepsilon = 0.02$.


\subsection{Self-consistency evaluation of the fantasy example}\label{subsec:app:experiment_details:fantasy}

As a second synthetic example, we consider a reasoning task in a fictional fantasy role-playing game. The model is asked to estimate the probability that the main character wins an encounter against an opponent using the prompt template in \cref{subsec:app:sec:prompt_template:fantasy}.
Unlike the survey setting, and unlike the tennis example, this task need not correspond to any real-world data-generating process. The game world is imaginary, and there is no external ground-truth distribution against which the model's estimates can be evaluated. This makes the example useful for highlighting a key feature of our self-consistency checks. They only require the model's own direct and conditional estimates, and therefore remain meaningful even when no reference distribution exists.
At the root node, the prompt specifies only the basic combat situation.
We then refine the scenario by revealing additional contextual information about the game world. The attribute set $\mathcal{A}$ consists of two binary conditions: the time of day, distinguishing day and night, and the wind conditions, distinguishing calm from windy combat environments.

To generate the results in \cref{table:self_consistency:fantasy}, we proceed exactly as in the tennis example (\cref{subsec:app:experiment_details:tennis}). We compute the split consistency score $\mathrm{SC}_\varepsilon(\mathcal{A})$ over the admissible checks in $\mathcal{C}_{\mathrm{SC}}(\mathcal{A})$ and the order consistency score $\mathrm{OC}_\varepsilon(\mathcal{A})$ over the four two-attribute conditioning events in $\mathcal{C}_{\mathrm{OC}}(\mathcal{A})$. Consistency requires the compared estimates to agree, even though the underlying event is fictional and no ground-truth win probability is available. 
Since the target outcome is binary, discrepancies are measured by the absolute difference between the corresponding win probabilities, and each score reports the fraction of checks whose discrepancy is less than $\varepsilon = 0.02$.


\newpage
\section{Additional plots}\label{sec:app:additional_plots}

This appendix collects larger plots with more detailed breakdowns.

\subsection{Aggregation gain}\label{subsec:app:normalized_alignment_error}

We provide additional results on the aggregation gain to complement the main findings in \cref{sec:macro_fallacy}. While \cref{fig:aggregation_refinement:llm_prior} highlights representative settings, \cref{fig:app:normalized_alignment:models_thresholds} reports the aggregation gain across a broader range of fixed income thresholds. These results show that the qualitative pattern from the main text is robust. Reconstructed aggregate estimates from subgroup conditionals are often better aligned with the ACS reference statistic than direct estimates, although the gains vary across thresholds, models, and tree levels.

\begin{figure}[ht]
\centering
\resizebox{1.0\linewidth}{!}{%
\input{figures/3g_aggregation_multi_bars_appendix_top.pgf}
}
\begin{subfigure}[t]{0.48\linewidth}
\vspace*{-4mm}
\caption{
Fixed income threshold $\tau = 1\mathrm{k}$ USD.
}
\label{subfig:aggregation_refinement:llm:1k}
\end{subfigure}
\hfill
\begin{subfigure}[t]{0.48\linewidth}
\vspace*{-4mm}
\caption{
Fixed income threshold $\tau = 10\mathrm{k}$ USD.
}
\label{subfig:aggregation_relative_models:llm:10k}
\end{subfigure}
\vspace*{2mm}

\resizebox{1.0\linewidth}{!}{%
\input{figures/3g_aggregation_multi_bars_appendix_bottom.pgf}
}
\begin{subfigure}[t]{0.48\linewidth}
\vspace*{-4mm}
\caption{
Fixed income threshold $\tau = 50\mathrm{k}$ USD.
}
\label{subfig:aggregation_refinement:llm:50k}
\end{subfigure}
\hfill
\begin{subfigure}[t]{0.48\linewidth}
\vspace*{-4mm}
\caption{
Fixed income threshold $\tau = 80\mathrm{k}$ USD.
}
\label{subfig:aggregation_relative_models:llm:80k}
\end{subfigure}
\caption{
\textbf{Aggregation gain across models and tree levels.}
We report the aggregation gain $1 - \text{AErr}(\ell)/\text{AErr}(0)$ with $\mathrm{AErr}(\ell)$ defined in \eqref{eq:agg_error}. Positive values indicate that reconstructing the population-level estimate from subgroup estimates improves alignment relative to direct prompting, while negative values indicate worse alignment. Each bar corresponds to a different model. Error bars show 90\,\% confidence intervals computed by bootstrapping with 1000 samples. The orange curve shows the average aggregation gain.
}
\label{fig:app:normalized_alignment:models_thresholds}
\end{figure}

\newpage
\subsection{Reconstructed aggregate distributions across tree levels}\label{subsec:app:reconstr_aggr}

\Cref{fig:acs:income_distribution_tree:combined:aggr_only} visualizes the population-level ACS income distributions reconstructed from each tree level using LLM-estimated subgroup priors.

\begin{figure}[ht]
\centering
\begin{tikzpicture}[
  font=\fontsize{7}{8}\selectfont,
  grow=right,
  level distance=40.5mm,
  edge from parent/.style={
    draw=black!75,
    line width=1.6pt,
    -{Stealth[length=2mm, width=2mm]},
  },
  aggTreenode/.style={
    draw=black!65,
    rounded corners=6pt,
    line width=1.6pt,
    fill=white,
    align=center,
    inner xsep=6pt,
    inner ysep=6pt,
    text width=29mm
  },
]
\begin{scope}

\node[aggTreenode, label={\bf Level 0}] {   \centering
\begingroup
\tikzset{
  every path/.style={},
  every node/.style={},
  every picture/.style={},
}%
\pgfsetcornersarced{\pgfpointorigin}%
\resizebox{\linewidth}{!}{%
\input{figures/tree_aggr_only/agg_binned_hist_0.pgf}
}
\endgroup
}
  child { node[aggTreenode, label={\bf Level 1}] {         \centering
\begingroup
\tikzset{
  every path/.style={},
  every node/.style={},
  every picture/.style={},
}%
\pgfsetcornersarced{\pgfpointorigin}%
\resizebox{\linewidth}{!}{%
\input{figures/tree_aggr_only/agg_binned_hist_1.pgf}
}
\endgroup
      }
        child { node[aggTreenode, label={\bf Level 2}] {               \centering
\begingroup
\tikzset{
  every path/.style={},
  every node/.style={},
  every picture/.style={},
}%
\pgfsetcornersarced{\pgfpointorigin}%
\resizebox{\linewidth}{!}{%
\input{figures/tree_aggr_only/agg_binned_hist_2.pgf}
}
\endgroup
            } } };

\end{scope}


\definecolor{gtbin0}{RGB}{46,107,142}
\definecolor{gtbin1}{RGB}{30,152,138}
\definecolor{gtbin2}{RGB}{75,194,108}
\definecolor{gtbin3}{RGB}{189,222,38}

\node[anchor=north] at (4.1, -1.1) {
\centering
\begingroup
\tikzset{
  every path/.style={},
  every node/.style={},
  every picture/.style={},
}%
\pgfsetcornersarced{\pgfpointorigin}%
\resizebox{124mm}{!}{%
\input{figures/tree_aggr_only/horizontal_legend.pgf}
}
\endgroup
};


\end{tikzpicture}
\vspace*{-3mm}
\caption{
\textbf{Reconstructed aggregate income distributions across tree levels.}
For each level of the binary conditioning tree in \cref{fig:acs:income_distribution_tree:combined}, we reconstruct the population-level ACS income distribution by aggregating the subgroup distributions with subgroup priors.
Solid bars show the same survey-based ground-truth distribution of the base population in all panels, while hatched bars show the corresponding LLM-induced reconstructed aggregate.
The left panel corresponds to the direct estimate at the root, whereas the middle and right panels show aggregates reconstructed from increasingly fine-grained conditional estimates.
The comparison illustrates that the reconstructed aggregate becomes better aligned with the ground-truth distribution as we move down the tree.
}
\label{fig:acs:income_distribution_tree:combined:aggr_only}
\end{figure}


\clearpage
\section{Prompt templates}\label{sec:prompt_templates}

This appendix reports the exact prompt templates used to elicit probability estimates from language models.


\subsection{Prior elicitation}\label{subsec:app:sec:prompt_template:priors}

\begin{promptbox}{Subgroup priors}
You are given a mutually exclusive and exhaustive partition of the U.S. population into groups defined by sociodemographic attributes.

For each group below, estimate the probability that a randomly sampled person from the U.S. population belongs to that group.

{groups}

Important:
- Treat the listed groups as the full population partition.
- Assign probability mass across all groups.
- The returned probabilities must sum to 1.
- Use your best demographic knowledge of the U.S. population.

Return only a valid JSON object of the following form:
{
  "Group 1": 0.35,
  "Group 2": 0.65,
  ...
}

Strict output rules:
- Do not include any text outside the JSON object.
- Do not include explanations or uncertainty intervals.
- Each probability must be a single number in [0, 1].
- The JSON object must contain exactly one entry per group.
- Use the group names exactly as given: Group 1, Group 2.
\end{promptbox}

The \texttt{groups} placeholder is replaced by a list of groups, each specified by a structured list of sociodemographic attributes that define the corresponding subgroup in the partition. For example:
\begin{codebox}
Group 1
- sex: female

Group 2
- sex: male
\end{codebox}


\subsection{Sociodemographic prompting}\label{subsec:app:sec:prompt_template:socio}

\begin{promptbox}{Thresholded income task}
You are given a brief sociodemographic profile of an individual. Based on the information provided and your general knowledge of income distributions in the United States, estimate the probability of the specified outcome. The data provided is sufficient to make an approximate probabilistic estimate.

Individual attributes:
{attributes}

Question:
What is the estimated probability that this person's annual income exceeds {income_threshold} USD?

Output requirements (strict):
- Return a single real number between 0 and 1 (inclusive).
- Format the answer exactly as: [[p]] where p is your probability estimate (for example: [[0.5]]).
\end{promptbox}

In the sociodemographic formulation, the \texttt{attributes} placeholder is replaced by a structured list of demographic attributes describing the target individual. For example:
\begin{codebox}
Individual attributes:
- Sex: Female
- Employment status: Civilian employed, at work
- Age: 24 years old
\end{codebox}
This formulation describes the conditioning information in the third person and avoids asking the model to adopt a role-playing perspective.


\subsection{Persona prompting}\label{subsec:app:sec:prompt_template:persona}

\begin{promptbox}{Thresholded income task}
Please answer the following question from the perspective of the described person. Stay consistent with the profile, but do not mention that you are role-playing unless asked.

Persona:
{persona}

Question:
What is the estimated probability that this person's annual income exceeds {income_threshold} USD?

Output requirements (strict):
- Return a single real number between 0 and 1 (inclusive).
- Format the answer exactly as: [[p]] where p is your probability estimate (for example: [[0.5]]).
\end{promptbox}

In the persona formulation, the \texttt{persona} placeholder is replaced by a free-text description that conveys the same demographic information in a first-person role-playing style. For example, the profile may read: ``You are a person who lives in the United States. Your sex is Female.''


\subsection{Partition-then-aggregate prompting}\label{subsec:app:sec:prompt_template:mtm_prompt}

\begin{promptbox}{Thresholded income task}
Based on your general knowledge of income distributions in the United States, estimate the probability of the specified outcome. Please make an approximate probabilistic estimate.

Question:
What is the estimated probability that a randomly selected person's annual income exceeds {income_threshold} USD?

Before giving your final answer, reason from micro to macro:
1. Identify several relevant subpopulations within the U.S. population that may differ in income.
2. For each subpopulation, consider both its approximate prevalence in the U.S. population and its likely probability of exceeding the income threshold.
3. Combine these subgroup-level considerations into a single aggregate estimate for the U.S. population.

Output requirements (strict):
- Return a single real number between 0 and 1 (inclusive).
- Format the answer exactly as: [[p]] where p is your probability estimate (for example: [[0.5]]).
\end{promptbox}

This prompt asks the model to estimate a population-level income probability by explicitly reasoning through relevant {U.S.} subpopulations, their prevalence, and their subgroup-specific probabilities. Unlike our tree-based aggregation prompts, the relevant subpopulations are not specified in advance but are instead chosen by the model.


\subsection{Opinion distribution elicitation}\label{subsec:app:sec:prompt_template:wvs_distribution}

\begin{promptbox}{Survey-response distribution prompt}
You are estimating the distribution of survey responses for a demographic subpopulation. The subpopulation is characterized by the following attributes:

{attribute_list}

Survey question:
{question}

Answer options:
{answer_option_list}

Task:
Estimate the probability distribution over the answer options for a randomly selected individual from this subpopulation.

Strict output rules:
- Return only valid JSON.
- Use exactly the following keys: {answer_option_keys}.
- Each value must be a single number in [0, 1].
- The values must sum to 1.
- Use decimal probabilities, not percentages.
- Do not include explanations, comments, uncertainty intervals, or any text outside the JSON object.

Return only a valid JSON object of the following form. The numbers below are illustrative only and must be replaced by your estimated probabilities:
{example_output}
\end{promptbox}

For the WVS opinion-distribution experiments, the \texttt{attribute\_list} placeholder is replaced by the demographic attributes defining the target subpopulation, \texttt{question} is replaced by the survey question text, and \texttt{answer\_option\_list} lists the available response options. The model is instructed to return a complete categorical probability distribution over the answer options as a JSON object. The placeholders \texttt{answer\_option\_keys} and \texttt{example\_output} ensure that the returned object has a fixed schema across models and prompts.


\subsection{Tennis forecasting task}\label{subsec:app:sec:prompt_template:tennis}

\begin{promptbox}{Tennis match forecasting prompt}
You are estimating the outcome of a tennis match.

Match:
Roger Federer plays against Rafael Nadal.

The match context is characterized by the following attributes:
{attribute_list}

Forecasting question:
What is the probability that Roger Federer wins the match?

Task:
Estimate the probability that Roger Federer wins the match under the specified context.

Strict output rules:
- Return a single real number between 0 and 1, inclusive.
- Use decimal probabilities, not percentages.
- Format the answer exactly as: [[p]], where p is your probability estimate.
- Do not include explanations, comments, uncertainty intervals, or any text outside the required format.

Return only an answer of the following form. The number below is illustrative only and must be replaced by your estimated probability: [[0.5]]
\end{promptbox}


\subsection{Fictional fantasy combat task}\label{subsec:app:sec:prompt_template:fantasy}

\begin{promptbox}{Fantasy combat forecasting prompt}
You are estimating the outcome of a fantasy combat encounter.

A main character is engaged in combat against an enemy in an imaginary fantasy setting. The world has a day-night cycle and varying weather conditions.

The current encounter context is characterized by the following attributes:
{attribute_list}

Forecasting question:
What is the probability that the main character wins the encounter?

Task:
Estimate the probability that the main character wins the encounter under the specified context.

Strict output rules:
- Return a single real number between 0 and 1, inclusive.
- Use decimal probabilities, not percentages.
- Format the answer exactly as: [[p]], where p is your probability estimate.
- Do not include explanations, comments, uncertainty intervals, or any text outside the required format.

Return only an answer of the following form. The number below is illustrative only and must be replaced by your estimated probability: [[0.5]]
\end{promptbox}

\subsection{Conditional estimates from decision models}\label{subsec:app:prompt_template:decision}

To elicit conditional estimates from the evaluated decision models, we use prompts comprising a question, a context, and a set of answer options, as illustrated below. The context attributes and their values are illustrative and may vary across the prompts presented to the models.

\begin{promptbox}{Question}
Is this person's annual income greater than 40000 USD? Estimate the probability under the specified context.
\end{promptbox}

\begin{promptbox}{Context}
The person is an individual in the United States with the following attributes:
- age: 40--60 years old
- class of worker: Currently employed or self-employed
- usual number of hours worked per week: 40--46 hours per week
\end{promptbox}

\begin{promptbox}{Answer Options}
{
  "A": "Yes, their annual income is greater than 40000 USD.", 
  "B": "No, their annual income is at most 40000 USD."
}
\end{promptbox}


\subsection{Prior estimates from decision models}\label{subsec:app:prompt_template:decision_prior}

Prior estimates are elicited jointly for all nodes in a given tree level, with each node represented by an answer option describing its defining attributes.

\begin{promptbox}{Question}
Which group does a randomly sampled person belong to? Estimate the probability of each option under the specified context.
\end{promptbox}

\begin{promptbox}{Context}
A person is sampled uniformly at random from the U.S. population. Use your background knowledge about the demographic composition of the United States.
\end{promptbox}

\begin{promptbox}{Answer Options}
{
  "Group 1": "The person has the following attributes: ...", 
  "Group 2": "The person has the following attributes: ..."
}
\end{promptbox}


\newpage
\section{ACS attributes}\label{sec:acs_attributes}

\Cref{tab:acs_attributes} summarizes the ACS attributes used to define the conditioning variables in our experiments.

\begin{table}[ht]
\centering
\renewcommand{\arraystretch}{1.4}
\setlength{\tabcolsep}{4pt}
\begin{tabular}{l p{0.21\linewidth} p{0.62\textwidth}}
\toprule
\textbf{Attribute} & \textbf{Description} & \textbf{Values}\\
\midrule
\texttt{AGEP} & age & 0 -- 96 years old\\
\texttt{COW} & class of worker &
1: ``Working for a for-profit private company or organization''\newline
2: ``Working for a non-profit organization''\newline
3: ``Working for the local government''\newline
4: ``Working for the state government''\newline
5: ``Working for the federal government''\newline
6: ``Owner of non-incorporated business, professional practice, or farm''\newline
7: ``Owner of incorporated business, professional practice, or farm''\newline
8: ``Working without pay in a for-profit family business or farm''\newline
9: ``Unemployed and last worked 5 years ago or earlier or never worked''\\
\texttt{ESR} & employment status & 
1: ``Civilian employed, at work''\newline
2: ``Civilian employed, with a job but not at work''\newline
3: ``Unemployed''\newline
4: ``Armed forces, at work''\newline
5: ``Armed forces, with a job but not at work''\newline
6: ``Not in labor force''\\
\texttt{MAR} & marital status & 
1: ``Married''\newline
2: ``Widowed''\newline
3: ``Divorced''\newline
4: ``Separated''\newline
5: ``Never married''\\
\texttt{OCCP} & occupation & list of 530 occupations\\
\texttt{WKHP} & usual number of hours worked per week & 1 -- 99 hours\\
\bottomrule
\end{tabular}
\vspace*{2mm}
\caption{
\textbf{ACS attributes used for demographic conditioning.} 
We list the ACS variables used to construct the conditioning trees in our experiments, together with their descriptions and possible values. Integer-valued variables are reported by their observed range, while finite categorical variables are reported using their ACS response codes. For the demographic splits, we augment each value list with an \texttt{n/a} class to represent missing values in the ACS survey data.}
\label{tab:acs_attributes}
\end{table}


\newpage
\section{WVS questions}\label{sec:wvs_questions}

For completeness, \cref{tab:wvs_questions} reports the World Values Survey questions used in our experiments, including their original WVS identifiers, question wording, and ordinal answer scales.

\begin{table}[ht]
\centering
\renewcommand{\arraystretch}{1.4}
\setlength{\tabcolsep}{4pt}
\begin{tabular}{c c p{0.4\textwidth} p{0.4\textwidth}}
\toprule
\textbf{ID} & \textbf{WVS-ID} & \textbf{Question} & \textbf{Answer options}\\
\midrule
$Q_a$ &
Q162 &
Now, I would like to read some statements and ask how much you agree or disagree with each of these statements. For these questions, a 1 means that you ``completely disagree'' and a 10 means that you ``completely agree'': It is not important for me to know about science in my daily life &
[1: Completely disagree, 2, 3, 4, 5, 6, 7, 8, 9, 10: Completely agree]\\
$Q_b$ &
Q177 &
Please tell me for each of the following statements whether you think it can always be justified, never be justified, or something in between, using this card. - Claiming government benefits to which you are not entitled &
[1: Never justifiable, 2, 3, 4, 5, 6, 7, 8, 9, 10: Always justifiable]\\
$Q_c$ &
Q186 &
Please tell me for each of the following statements whether you think it can always be justified, never be justified, or something in between, using this card. Sex before marriage &
[1: Never justifiable, 2, 3, 4, 5, 6, 7, 8, 9, 10: Always justifiable]\\
$Q_d$ &
Q245 &
Please tell me for each of the following things how essential you think it is as a characteristic of democracy. Use this scale where 1 means ``not at all an essential characteristic of democracy'' and 10 means it definitely is ``an essential characteristic of democracy.'' - The army takes over when government is incompetent &
[1: Not an essential characteristic of democracy, 2, 3, 4, 5, 6, 7, 8, 9, 10: An essential characteristic of democracy]\\
$Q_e$ &
Q252 &
How satisfied are you with how the political system is functioning in your country these days? On the scale from 1 to 10 below, 1 means not satisfied at all and 10 means completely satisfied. &
[1: Not satisfied at all, 2, 3, 4, 5, 6, 7, 8, 9, 10: Completely satisfied]\\
\bottomrule
\end{tabular}
\vspace*{2mm}
\caption{
\textbf{Selected World Values Survey questions.}
We report the five WVS questions used in our experiments, together with their original WVS identifiers, question text, and ordinal answer scales.
}
\label{tab:wvs_questions}
\end{table}


\end{document}